\documentclass{article}

    \PassOptionsToPackage{numbers, sort&compress}{natbib}

\usepackage[main,final]{neurips_2026}

\usepackage[utf8]{inputenc}
\usepackage[T1]{fontenc}
\usepackage{xcolor}
\definecolor{paperblue}{RGB}{82,145,165}
\definecolor{paperred}{RGB}{176,65,55}
\definecolor{graphred}{HTML}{C25759}
\usepackage[colorlinks=true,
            linkcolor=paperred,
            citecolor=paperblue,
            filecolor=paperblue,
            urlcolor=paperblue]{hyperref}
\usepackage{url}
\usepackage{booktabs}
\usepackage{amsfonts}
\usepackage{nicefrac}
\usepackage{microtype}
\usepackage{amssymb}
\usepackage{amsmath,nccmath}
\usepackage{multirow}
\usepackage{makecell}
\usepackage{amsthm}
\usepackage{bigints}
\usepackage{graphicx}
\usepackage{float}
\usepackage{enumitem}
\usepackage{algorithm}
\usepackage{algorithmic}
\usepackage{tikz}
\usepackage[most]{tcolorbox}

\newtheorem{theorem}{Theorem}[section]
\newtheorem{corollary}[theorem]{Corollary}
\newtheorem{lemma}[theorem]{Lemma}

\theoremstyle{definition}

\theoremstyle{remark}
\newtheorem{remark}[theorem]{Remark}

\newtcolorbox{assumptionbox}[1]{
  enhanced,
  breakable,
  colback=paperblue!4,
  colframe=paperblue!45,
  coltitle=black,
  fonttitle=\bfseries,
  fontupper=\normalfont,
  title={#1},
  boxrule=0.35pt,
  arc=1pt,
  left=5pt,
  right=5pt,
  top=4pt,
  bottom=4pt,
  before skip=0.8em,
  after skip=0.8em
}

\newcommand{\contrib}[1]{%
  \tikz[baseline=(char.base)]{%
    \node[
      circle,
      draw=graphred!45,
      fill=graphred!8,
      line width=0.20pt,
      inner sep=0.30pt,
      outer sep=0pt,
      minimum size=0.88em
    ] (char) {\textcolor{graphred!85!black}{\bfseries\scriptsize #1}};%
  }%
}

\newcommand{\R}{\mathbb{R}}
\newcommand{\E}{\mathbb{E}}
\newcommand{\Var}{\operatorname{Var}}
\newcommand{\diag}{\operatorname{diag}}

\title{Disentangling Continuous-Time Latent Dynamics: Identifiability of Latent SDEs via Diffusion Shifts}

\author{%
  Yuanyuan Wang$^1$, \, Wenjie Wang$^2$, \, Haoxuan Li$^3$, \,
  Mingming Gong$^{1,2}$\thanks{Co-corresponding authors.}, \,
  Kun Zhang$^{1,4}$\footnotemark[1] \\
  $^1$Mohamed bin Zayed University of Artificial Intelligence \\
  $^2$The University of Melbourne \quad $^3$Peking University \quad
  $^4$Carnegie Mellon University \\
  \texttt{yuanyuan.wang@mbzuai.ac.ae, \, wenjiew3@student.unimelb.edu.au} \\
  \texttt{hxli@stu.pku.edu.cn, \, mingming.gong@unimelb.edu.au, \, kunz1@cmu.edu}
}

\begin{document}

\maketitle

\begin{abstract}
Causal representation learning for time series has developed strong identifiability results in discrete-time latent causal models, but identifiability in continuous-time latent stochastic differential equation (SDE) models remains largely open. We address this gap using environment-induced shifts in diffusion covariance. We study additive-noise latent SDEs observed through an unknown nonlinear diffeomorphism, with shared drift but environment-specific diffusion covariance. We show that two diagonal diffusion regimes with pairwise distinct coordinate-wise variance ratios identify the latent coordinates up to permutation, coordinate-wise scaling, and a possible constant shift, without any sparsity assumption on the drift. We first prove this result for linear Ornstein--Uhlenbeck systems and then extend it to general additive-noise latent SDEs. Under mild smoothness, the instantaneous drift-Jacobian causal graph is identifiable up to the same permutation. We propose a two-stage estimator for latent disentanglement and optional graph recovery; experiments on synthetic systems confirm the predicted identifiability boundary, and an application to Hardanger Bridge monitoring data illustrates the approach on real sensor trajectories.
\end{abstract}

\section{Introduction}

Causal representation learning (CRL) seeks to recover latent causal variables
and their structural relations from entangled observations
\cite{scholkopf2021toward}. This problem is especially compelling for time
series, where causal structure underlies explanation, forecasting, and
decision-making in domains such as climate science, biology, healthcare, and
economics \cite{runge2023causal,shojaie2022granger}. In many temporal settings,
however, the measured variables are not the causal variables of interest:
video frames, sensor streams, and multivariate readouts are often nonlinear mixtures
of latent dynamical factors. Without additional structure, those factors are
not identifiable from observations alone
\cite{hyvarinen1999nonlinear,locatello2019challenging}. Temporal CRL therefore
asks how temporal dependence, environmental variation, or interventions can
break nonlinear mixing ambiguities and make latent causal variables
identifiable.

\paragraph{Motivation.}
Many scientific and engineered systems evolve in continuous time, are sampled
irregularly, and are intrinsically stochastic. Stochastic differential equations
(SDEs) provide a standard language for such dynamics
\cite{oksendal2013stochastic,sarkka2019applied}. Continuous-time latent and
controlled dynamics are also widely used in machine learning through neural
ODEs, latent ODEs, neural controlled differential equations, and neural SDEs
\cite{chen2018neural,rubanova2019latent,kidger2020ncde,li2020scalable}.
Continuous time also exposes objects that are blurred by a discrete-time view,
including drift, diffusion, quadratic variation, and instantaneous dependence.
This distinction is important for representation learning: the drift describes
the deterministic law of motion, whereas the diffusion describes the scale and
orientation of unresolved stochastic fluctuations.

In many latent dynamical systems, it is natural for the deterministic law of
motion to remain stable across environments while the stochastic forcing changes. A
bridge may have the same structural dynamics while wind and traffic loads
change the amplitude and direction of stochastic excitation; a cardiovascular
system may retain the same regulatory feedback loops while exercise or stress
changes physiological variability; and a power grid may retain the same network
dynamics while stochastic load fluctuations vary across demand or weather
regimes. Such regime variation is often easier to observe than direct
interventions on latent causal mechanisms. Such variation also provides a form
of side information specific to continuous time: diffusion leaves a local quadratic-variation
signature in the observed path, even though this signature is distorted by the
unknown nonlinear mixing. This leads to the central identifiability question of
this paper: when can diffusion shifts alone recover the latent coordinate system
of a continuous-time SDE model from nonlinear observations?

\paragraph{Challenges.}
Existing identifiable nonlinear independent component analysis (ICA) and temporal CRL results show that side
information can break nonlinear mixing ambiguities, but their identifiability
theory is largely built for discrete-time settings, using finite-lag
transitions, autoregressive mechanisms, or temporal context on a grid
\cite{hyvarinen2019nonlinear,khemakhem2020variational,yaolearning,yao2022temporally,song2024causal}.
Conversely, work on continuous-time SDE identification and causal discovery
typically assumes that the stochastic state itself is observed, rather than
seen only through an unknown nonlinear observation map
\cite{wang2023generator,wangneural,guan2024identifying,mantensignature,zweig2025towards}.
The latent continuous-time setting is harder: the state is observed only after
an unknown nonlinear mixing, and the target is not merely a predictive state but
the latent coordinate system and its instantaneous drift-Jacobian graph. What is
missing is a theorem-level CRL account for continuous-time latent SDE state
coordinates observed only through such a mixing. The central difficulty is that
diffusion covariance is not invariant under nonlinear coordinate changes: in
observed coordinates, local quadratic variation is the latent covariance
transformed by the Jacobian of the unknown mixing. A single diffusion regime
therefore cannot determine which latent axes are the causal coordinates. Our key
observation is that two diagonal diffusion regimes with distinct
coordinate-wise variance ratios impose simultaneous constraints on the same
unknown Jacobian. These constraints reduce the remaining
ambiguity to a monomial affine transformation (permutation, coordinate-wise scaling,
and a possible constant shift), which in turn makes instantaneous
drift-graph recovery well defined.

\paragraph{Contributions.}
We study additive-noise latent SDEs observed through an unknown nonlinear
diffeomorphism across two environments that share the latent drift but differ in
diffusion covariance. Our contributions are: \contrib{1} We formulate this
continuous-time CRL problem and identify diffusion shifts as the source of side
information: two diagonal regimes with coordinate-wise distinct variance ratios
are sufficient to anchor the latent coordinate system. \contrib{2} We prove
identifiability first for linear Ornstein--Uhlenbeck systems and then for
general additive-noise latent SDEs: the latent coordinate system is recovered up
to a monomial affine transformation without any sparsity assumption on the drift, and
under mild smoothness the instantaneous drift-Jacobian causal graph is
identifiable up to the same permutation. \contrib{3} We turn the theory into a practical two-stage
estimator: diffusion-based disentanglement followed, when graph recovery is
desired, by sparse drift regression. Synthetic experiments validate the
predicted identifiability boundary, and a Hardanger Bridge monitoring study
illustrates the estimator on real sensor trajectories.

\section{Problem setup and target}
\label{sec:setup}

We consider two environments, indexed by $e\in\{1,2\}$, over a time horizon $T>0$. In environment $e$, the latent state
$z_t^{(e)}\in\R^d$ evolves according to the additive-noise SDE
\begin{equation}
\label{eq:generic_latent_sde_setup}
    dz_t^{(e)} = f(z_t^{(e)})\,dt + G_e\,dW_t^{(e)},
    \qquad
    x_t^{(e)} = g(z_t^{(e)}),
\end{equation}
where $x_t^{(e)}\in\R^d$ is the observed process, $W_t^{(e)}$ is a standard
$d$-dimensional Brownian motion, $f:\R^d\to\R^d$ is the latent drift,
$G_e\in\R^{d\times d}$ is the environment-specific diffusion matrix, and
$g:\R^d\to\R^d$ is an unknown nonlinear observation map. The two environments share
the same latent drift $f$ and differ only through their diffusion covariances. In the linear
Ornstein--Uhlenbeck special case studied first in Section~\ref{sec:theory_ou}, the drift takes
the form $f(z)=Az$ for some matrix $A\in\R^{d\times d}$. Throughout, we write
$\Sigma_e := G_e G_e^\top$
for the latent diffusion covariance in environment $e$.

The latent dynamics are autonomous: $f$ has no explicit time dependence, so the
causal mechanism is not explicitly time varying. In
continuous-time causal models for SDEs, a parent relation is typically induced
when one variable enters another variable's drift or diffusion coefficient
\cite{mantensignature,wangneural,manten2025asymmetric}. In our
additive-noise setting, however, the diffusion matrices $G_e$ are constant in
$z$, so no variable enters another variable's diffusion coefficient. We
therefore represent the latent causal structure by the drift alone. In the
linear OU case, this reduces to the support pattern of $A$, which is constant
over time. In the general nonlinear case, we work with the instantaneous
drift-Jacobian graph at state $z$, namely the support pattern of $Df(z)$; this graph
may depend on the current state, but not explicitly on time.

Our goal is to characterize conditions under which the pair of observed path
laws from the two environments is sufficient to recover the latent coordinate system
up to permutation, coordinate-wise scaling, and a possible constant shift, and to identify the corresponding
latent causal structure up to the same permutation. The precise identifiability
statements are given in Section~\ref{sec:theory}.

\section{Identifiability theory}
\label{sec:theory}

\subsection{Linear OU systems}
\label{sec:theory_ou}

We begin with the linear-drift case, where the identification mechanism can be
seen most transparently. This case is a conceptual starting point rather than a
restrictive assumption. In the OU setting, the diffusion comparison already contains the
core disentanglement mechanism, and the state-space coverage step is automatic
from the Gaussian transition law. This lets us isolate the role of the two
regimes with minimal additional technical overhead.

\paragraph{Assumptions.}
We state the shared assumptions for two observationally equivalent latent
representations: the reference parameterization and an alternative tilded
parameterization. Both the OU theorem below and the general nonlinear theorem in
Section~\ref{sec:theory_nonlinear} rely on these four assumptions.
\begin{assumptionbox}{Shared two-regime assumptions}
\begin{enumerate}[leftmargin=1.8em]
    \item[(S1)] (Regularity) $g,\tilde g:\R^d\to\R^d$ are $C^2$
    diffeomorphisms, and $\varphi := \tilde g^{-1}\circ g$.
    \item[(S2)] (Trajectory-law equivalence in each regime) For each
    $e\in\{1,2\}$,
    \[
      \{x_t^{(e)}\}_{t\in[0,T]} \overset{d}{=} \{\tilde x_t^{(e)}\}_{t\in[0,T]}
      \quad \text{in } C([0,T];\R^d).
    \]
    \item[(S3)] (Diagonal positive-definite covariances) For each
    $e\in\{1,2\}$, the diffusion covariances
    \[
      \Sigma_e := G_e G_e^\top,
      \qquad
      \tilde\Sigma_e := \tilde G_e \tilde G_e^\top
    \]
    are diagonal and positive definite.
    \item[(S4)] (Two-regime separation; simple spectrum) The diagonal entries
    of $\Sigma_1^{-1}\Sigma_2$ are pairwise distinct; equivalently, if
    $\Sigma_e=\diag(\sigma_{e,1}^2,\dots,\sigma_{e,d}^2)$, then
    \[
      \frac{\sigma_{2,i}^2}{\sigma_{1,i}^2}
      \neq
      \frac{\sigma_{2,j}^2}{\sigma_{1,j}^2}
      \qquad \forall i\neq j.
    \]
\end{enumerate}
\end{assumptionbox}
Assumptions (S1) and (S2) are standard in identifiable nonlinear ICA and causal
representation learning: one posits an invertible nonlinear mixing and asks
whether two latent parameterizations inducing the same observed law must be
equivalent, as in
\cite{hyvarinen2019nonlinear,khemakhem2020variational,yaolearning,yao2022temporally,liidentification,yaomulti}.
Assumption~(S3) is a nondegeneracy condition on the additive noise. It fixes
coordinate-wise diffusion geometry and rules out singular directions. Related
diagonal or block-diagonal noise restrictions are also
common in continuous-time structure learning and path-space causal discovery
\cite{wangneural,mantensignature,manten2025asymmetric}.

Assumption~(S4) is the separating condition: each latent coordinate carries a
distinct diffusion-ratio fingerprint across environments. This is generically
mild: failure occurs only when two ratios coincide exactly, a lower-dimensional
subset of the diagonal covariance parameter space. It is also the
continuous-time analogue of the variability conditions that drive
identifiability in temporal and multi-environment CRL
\cite{yao2022temporally,song2023temporally,von2023nonparametric,ng2025causal,liidentification}.
In the OU case, no additional support assumption is needed, because every
nondegenerate Gaussian transition has full support on $\R^d$.

Throughout this section, $\varphi = \tilde g^{-1}\circ g$ denotes the change of
coordinates between two observationally equivalent latent parameterizations:
$\tilde z=\varphi(z)$ describes the same underlying state in the alternative
coordinate system. We write $D\varphi(z)$ for its Jacobian. The main
identifiability question is whether $\varphi$ must reduce to a monomial affine
transformation.

\begin{theorem}[OU identifiability from two diagonal diffusion regimes]
\label{thm:ou_two_regimes_monomial}
For each $e\in\{1,2\}$, consider reference and tilded latent OU
parameterizations:
\begin{subequations}
\renewcommand{\theHequation}{ou.\arabic{equation}}
\begin{align}
  d z_t^{(e)} &= A\,z_t^{(e)}\,dt + G_e\,dW_t^{(e)},
  & x_t^{(e)} &= g(z_t^{(e)}), \label{eq:ou_true}\\
  d \tilde z_t^{(e)} &= \tilde A\,\tilde z_t^{(e)}\,dt
  + \tilde G_e\,d\tilde W_t^{(e)},
  & \tilde x_t^{(e)} &= \tilde g(\tilde z_t^{(e)}). \label{eq:ou_alt}
\end{align}
\end{subequations}
Assume \textup{(S1)--(S4)}. Then there exist a constant monomial matrix
$L=\Pi\Lambda$ and a constant vector $c\in\R^d$ such that
\[
\begin{gathered}
D\varphi(z)\equiv L,\qquad \varphi(z)=Lz+c \quad (z\in\R^d),\\
\tilde\Sigma_e = L\Sigma_eL^\top,\qquad
\tilde A L = L A,\qquad
\tilde A c = 0 \quad (e\in\{1,2\}).
\end{gathered}
\]
The relation $\tilde A L=LA$ is equivalently $\tilde A=LAL^{-1}$; the
additional condition $\tilde A c=0$ constrains the possible translation. Hence the observed path laws identify the
latent coordinate system up to permutation and scaling, together with a possible
constant shift if $\varphi(0)\neq 0$. For
$E_A:=\{(j\to i):A_{ij}\neq 0\}$,
$E_{\tilde A}:=\{(j\to i):\tilde A_{ij}\neq 0\}$, and
$\Pi e_j=e_{\pi(j)}$,
\[
(j\to i)\in E_A
\quad \Longleftrightarrow \quad
(\pi(j)\to \pi(i))\in E_{\tilde A}.
\]
Thus the latent causal structure in the OU case is identifiable up to the same
permutation.
\end{theorem}

\paragraph{Proof sketch.}
The full proof is deferred to Appendix~\ref{app:deferred_proofs}. Equality of
observed path laws allows a coupling with
$\tilde z_t^{(e)}=\varphi(z_t^{(e)})$ almost surely in each regime. It\^o's
formula and quadratic variation give
$D\varphi(z_t^{(e)})\Sigma_e D\varphi(z_t^{(e)})^\top=\tilde\Sigma_e$ along
trajectories. For OU dynamics, the latent state at any positive time has a
strictly positive Gaussian density on $\R^d$, so this identity extends to every
state $z\in\R^d$. Whitening with respect to the first regime reduces the problem
to an orthogonal matrix field. The second regime and the simple-spectrum
condition in (S4) force that field to be a signed permutation at every $z$.
Continuity then makes the signed permutation constant, so $D\varphi$ is a
constant monomial matrix and $\varphi$ is affine; substituting into the drift
relation gives $\tilde A L=LA$, yielding the coordinate-system and graph
statements.

\subsection{General latent SDEs}
\label{sec:theory_nonlinear}

We now turn to nonlinear drifts. As in the OU case,
$\varphi := \tilde g^{-1}\circ g$ denotes the coordinate change between two
observationally equivalent latent parameterizations, with Jacobian
$D\varphi(z)$. The diffusion comparison still provides the core algebraic step.
Only the support argument changes: an identity first obtained along sample
paths must be extended to $\R^d$ without explicit Gaussian transition
densities.

The shared assumptions \textup{(S1)--(S4)} remain in force, and \textup{(S5)}
ensures well-defined continuous semimartingales and the
open-set positivity step. This is a standard additive-noise SDE regularity
condition: it requires continuous drifts and weak well-posedness with continuous
sample paths, but not strong solutions or closed-form transition densities.
Appendix~\ref{app:weak_vs_strong} gives further discussion.

\begin{theorem}[Latent SDE identifiability from two diffusion regimes]
\label{thm:nonlinear_sde_two_regimes}
For each $e\in\{1,2\}$, consider reference and tilded latent SDE
parameterizations:
\begin{subequations}
\renewcommand{\theHequation}{nonlinear.\arabic{equation}}
\begin{align}
    d z_t^{(e)} &= f(z_t^{(e)})\,dt + G_e\,dW_t^{(e)},
    & x_t^{(e)} &= g(z_t^{(e)}), \label{eq:true_sde_model}\\
    d \tilde z_t^{(e)} &= \tilde f(\tilde z_t^{(e)})\,dt
    + \tilde G_e\,d\tilde W_t^{(e)},
    & \tilde x_t^{(e)} &= \tilde g(\tilde z_t^{(e)}). \label{eq:alt_sde_model}
\end{align}
\end{subequations}
Assume the shared two-regime conditions \textup{(S1)--(S4)}, and in addition:
\begin{assumptionbox}{Additional regularity assumption for nonlinear SDEs}
\begin{enumerate}[leftmargin=1.8em]
    \item[(S5)] (Drift regularity and well-posedness) The drifts
    $f,\tilde f$ are continuous on $\R^d$. For each $e$, both SDEs in
    \eqref{eq:true_sde_model} and \eqref{eq:alt_sde_model} are weakly
    well-posed on $[0,T]$ with a.s. continuous sample paths.
\end{enumerate}
\end{assumptionbox}
Then there exist a constant monomial matrix $L=\Pi\Lambda$ and a constant
vector $c\in\R^d$ such that
\[
\begin{gathered}
D\varphi(z)\equiv L,\qquad \varphi(z)=Lz+c \quad (z\in\R^d),\\
\tilde\Sigma_e = L\Sigma_eL^\top \quad (e\in\{1,2\}),\qquad
\tilde f(Lz+c)=L f(z) \quad (z\in\R^d).
\end{gathered}
\]
Consequently, the observed path laws identify the latent coordinate system up
to permutation and scaling, together with a possible constant shift $c$.
\end{theorem}

\begin{remark}[Support coverage is automatic]
\label{rem:nonlinear_state_coverage}
The proof only needs that, for each regime $e$, the law of
$z_{t_\star}^{(e)}$ has full support on $\R^d$ at some $t_\star>0$. Under \textup{(S3)} and
\textup{(S5)}, this follows from the open-set positivity lemma in
Appendix~\ref{app:open_set_positivity}, which covers arbitrary initial laws.
\end{remark}

\paragraph{Proof sketch.}
The proof, deferred to Appendix~\ref{app:deferred_proofs}, reuses the OU
argument. The coupling, quadratic-variation comparison, and
whitening/simple-spectrum steps are unchanged. The new ingredient is the
open-set positivity lemma, which upgrades the along-trajectory covariance
identity to a pointwise identity on $\R^d$ without Gaussian transition
densities. The rest follows as in the OU case: $D\varphi$ is a constant
monomial matrix, hence $\varphi(z)=Lz+c$, and the same semimartingale comparison
gives $\tilde f(Lz+c)=Lf(z)$.

Under \textup{(S1)--(S5)}, Theorem~\ref{thm:nonlinear_sde_two_regimes} shows
that any observationally equivalent latent coordinate change must be a monomial
affine map. When $f$ and $\tilde f$ are $C^1$, differentiating
$\tilde f(Lz+c)=Lf(z)$ yields the drift-Jacobian graph statement.

\begin{corollary}[Drift-Jacobian graph identifiability up to permutation]
\label{cor:jacobian_graph_identifiability}
Under the assumptions of Theorem~\ref{thm:nonlinear_sde_two_regimes}, assume
additionally that $f,\tilde f$ are $C^1$. Let $J_f(z):=Df(z)$,
$E_f(z):=\{(j\to i):(J_f(z))_{ij}\neq 0\}$, and define
$J_{\tilde f},E_{\tilde f}$ analogously. Then
\[
D\tilde f(Lz+c)=L\,Df(z)\,L^{-1}
\qquad \forall z\in\R^d,
\]
so the instantaneous drift-Jacobian causal graphs are identified up to the
permutation induced by $\Pi$. Writing $\Pi e_j=e_{\pi(j)}$, equivalently,
\[
(j\to i)\in E_f(z)
\quad \Longleftrightarrow \quad
(\pi(j)\to \pi(i))\in E_{\tilde f}(Lz+c).
\]
\end{corollary}

\section{Two-stage estimation}
\label{sec:estimation}

We translate the identifiability results into a two-stage estimator. Stage~1
learns the latent coordinate system from short-time transition statistics across
the two regimes. Stage~2 is optional and estimates a sparse drift graph when
explicit graph recovery is desired.

\subsection{Model parameterization}
\label{sec:model_param}

We observe trajectories
$\{x_{t_0}^{(e)},x_{t_1}^{(e)},\dots,x_{t_N}^{(e)}\}$ from two regimes
$e\in\{1,2\}$ at uniform intervals $\Delta t$. The encoder $h_\theta$ maps
observations into a learned latent coordinate system, playing the estimator-side
role of $\tilde g^{-1}$ in
Sections~\ref{sec:theory_ou}--\ref{sec:theory_nonlinear}. We keep the tilde
notation $(\tilde z,\tilde f,\tilde\Sigma_e)$ for learned coordinates, drift,
and diffusion parameters, with $\theta,\psi$ denoting parametric components.
Our learnable model consists of three components:
\begin{itemize}[leftmargin=1.5em]
    \item an encoder $h_\theta:\R^d\to\R^d$, implemented as an MLP in our
    experiments, with $\tilde z_t=h_\theta(x_t)$;
    \item a differentiable drift model $\tilde f_\psi:\R^d\to\R^d$, shared
    across regimes and implemented as either a linear map or an MLP;
    \item diagonal per-regime diffusion parameters
    $\tilde\Sigma_e=\diag(\tilde\sigma_{e,1}^2,\ldots,\tilde\sigma_{e,d}^2)$,
    parameterized by $\log \tilde\sigma_{e,i}^2$ to ensure positivity.
\end{itemize}
We do not enforce $h_\theta$ to be globally invertible during training, so the
Stage~1 likelihood below is interpreted as a pseudo-likelihood in general.

\subsection{Stage~1: short-time transition fitting}
\label{sec:stage1}

Stage~1 uses the short-time behavior of additive-noise SDEs. For small
$\Delta t$, the Euler--Maruyama discretization gives a Gaussian approximation to
one-step transitions \cite{oksendal2013stochastic,sarkka2019applied}. In regime
$e$, the encoded-space working model is
\begin{equation}
\label{eq:EM_transition}
\tilde z_{t+\Delta t}\mid \tilde z_t,e
\sim
\mathcal{N}\!\Big(\tilde z_t + \tilde f_\psi(\tilde z_t)\,\Delta t,
\tilde\Sigma_e\,\Delta t\Big).
\end{equation}
Equivalently,
\[
\tilde z_{t+\Delta t}
=
\tilde z_t + \tilde f_\psi(\tilde z_t)\,\Delta t + \varepsilon_t^{(e)},
\qquad
\varepsilon_t^{(e)}\sim \mathcal{N}(0,\tilde\Sigma_e\,\Delta t).
\]
The diagonal $\tilde\Sigma_e$ makes the one-step residual coordinates
conditionally independent under this model.

If the encoded states $\tilde z_t$ were observed, the per-transition negative
log-likelihood, up to the additive constant
$\tfrac{d}{2}\log(2\pi\Delta t)$, is
\begin{equation}
\label{eq:nll_ztilde}
\ell_e(\tilde z_t,\tilde z_{t+\Delta t})
=
\frac{1}{2}\sum_{i=1}^d \log \tilde\sigma_{e,i}^2
+
\frac{1}{2}\sum_{i=1}^d
\frac{\big(\tilde z_{t+\Delta t,i}-\tilde z_{t,i}-\tilde f_{\psi,i}(\tilde z_t)\,\Delta t\big)^2}
{\tilde\sigma_{e,i}^2\,\Delta t}.
\end{equation}
This is the Gaussian transition score for \eqref{eq:EM_transition}.

\paragraph{Observation-space objective.}
We observe only $x_t$ and set $\tilde z_t=h_\theta(x_t)$. If $h_\theta$ is a
$C^1$ diffeomorphism, the change-of-variables formula gives
\[
p_\theta(x_{t+\Delta t}\mid x_t,e)
=
p_{\tilde z,e}\big(h_\theta(x_{t+\Delta t})\mid h_\theta(x_t)\big)
\,\big|\det Dh_\theta(x_{t+\Delta t})\big|,
\]
where $p_{\tilde z,e}$ is the Gaussian density in \eqref{eq:EM_transition} and
$Dh_\theta$ is the encoder Jacobian. The observed-data per-transition negative
log-likelihood, up to the same additive constant, is
\begin{equation}
\label{eq:nll_obs}
s_e(x_t,x_{t+\Delta t})
=
\ell_e\big(h_\theta(x_t),h_\theta(x_{t+\Delta t})\big)
-
\log\big|\det Dh_\theta(x_{t+\Delta t})\big|.
\end{equation}

In implementation, $h_\theta$ is differentiable but need not be globally
invertible, so \eqref{eq:nll_obs} is generally a
\emph{transition pseudo-likelihood}. The first term fits encoded increments to
the diagonal Gaussian transition model, while the log-determinant term provides
a local volume correction. If $h_\theta$ is a $C^1$ diffeomorphism, the same
score is the exact Euler--Maruyama one-step negative log-likelihood.

\paragraph{Stage~1 objective.}
We average the per-transition score over regimes and sampled consecutive pairs:
\begin{equation}
\label{eq:stage1_loss}
\mathcal{L}_{\mathrm{S1}}(\theta,\psi,\{\tilde\sigma_{e,i}^2\})
=
\frac{1}{2}\sum_{e=1}^{2}
\E_{(x_t,x_{t+\Delta t})\sim\mathcal{P}_e}
\Big[
 s_e(x_t,x_{t+\Delta t})
\Big],
\end{equation}
where $\mathcal{P}_e$ is the empirical distribution of consecutive pairs in
regime $e$. Stage~1 jointly learns
$(\theta,\psi,\{\tilde\sigma_{e,i}^2\})$.

\paragraph{Why this objective matches the theory.}
Over a short interval $\Delta t$, the drift contributes a mean shift of order
$O(\Delta t)$, whereas the stochastic increment has standard deviation
$O(\sqrt{\Delta t})$ and covariance $\tilde\Sigma_e\,\Delta t$. Diffusion
therefore dominates the short-time second-order behavior. Our theory shows that
this regime-dependent diagonal diffusion structure anchors the latent coordinate
system. Thus \eqref{eq:nll_ztilde}--\eqref{eq:stage1_loss} trains the encoder to find
coordinates in which short-time residual increments
\[
\tilde z_{t+\Delta t}-\tilde z_t-\tilde f_\psi(\tilde z_t)\,\Delta t
\]
fit a shared-drift, regime-specific diagonal Gaussian model. This is the
finite-step analogue of the two-regime identifiability mechanism in
Section~\ref{sec:theory}.

\subsection{Stage~2: sparse drift regression}
\label{sec:stage2}

Stage~2 is optional and used only to estimate a sparse causal graph. For latent
disentanglement alone, Stage~1 is the relevant part of the pipeline.

\paragraph{Velocity targets.}
With the encoder $h_\theta$ frozen, we form velocity targets from encoded transitions:
\begin{equation}
\label{eq:velocity_target}
\hat v_t
=
\frac{h_\theta(x_{t+\Delta t'})-h_\theta(x_t)}{\Delta t'},
\end{equation}
where $\Delta t'=s\Delta t$ for an integer stride $s\ge 1$. Larger strides
reduce target variance,
$\Var(\hat v_{t,i})\approx \tilde\sigma_{e,i}^2/\Delta t'$, at the cost of
$O(\Delta t')$ bias from time-averaging the drift.

Let $\mathcal V_e$ denote the empirical distribution of encoded
state--velocity pairs $(\tilde z_t,\hat v_t)$ in regime $e$, and let
$\mathcal Z$ denote the equally weighted empirical distribution of encoded
states from the two regimes.

\paragraph{Sparse drift estimation.}
We reinitialize $\tilde f_\psi$ and fit it to the velocity targets with a
Jacobian sparsity penalty:
\begin{equation}
\label{eq:stage2_loss}
\mathcal{L}_{\mathrm{S2}}(\psi)
=
\underbrace{
\frac{1}{2}\sum_{e=1}^{2}
\E_{(\tilde z_t,\hat v_t)\sim\mathcal{V}_e}
\Big\|\tilde f_\psi(\tilde z_t)-\hat v_t\Big\|^2
}_{\text{velocity MSE}}
+
\lambda_{\mathrm{sparse}}
\underbrace{
\E_{\tilde z\sim\mathcal{Z}}
\sum_{i=1}^d\sum_{j=1}^d
\bigg|\frac{\partial \tilde f_{\psi,i}}{\partial \tilde z_j}(\tilde z)\bigg|
}_{\ell_1\text{ drift-Jacobian penalty}}.
\end{equation}
This stage is estimator-side only and is used when a stable sparse graph
estimate is needed.

\paragraph{Graph extraction.}
By Corollary~\ref{cor:jacobian_graph_identifiability}, the learned
drift-Jacobian zero pattern is meaningful only up to the latent permutation
induced by $\Pi$. In practice, the binary adjacency matrix is obtained by
thresholding the mean absolute Jacobian at a fixed threshold $\tau>0$:
\[
\overline{|D\tilde f_\psi|}
:=
\E_{\tilde z\sim\mathcal{Z}}
\big[|D\tilde f_\psi(\tilde z)|\big].
\]

\paragraph{Summary.}
Algorithm~\ref{alg:two_stage} in Appendix~\ref{app:algorithm} summarizes the
training and graph-extraction procedure.


\section{Experiments}
\label{sec:simulation}

\subsection{Synthetic experiments}
\label{subsec:synthetic}

We test whether the estimator matches the identifiability boundary predicted by
the theory.  We use three latent drift families: dense linear OU, sparse linear
OU, and nonlinear sparse drift.  For each family, we compare three
diffusion-regime conditions: \emph{Distinct
ratios}, where two environments have coordinate-wise distinct diffusion-variance
ratios and hence satisfy the separation requirement in Assumption~(S4);
\emph{One regime}, which removes environment variation; and \emph{Proportional
diffusions}, which keeps two environments but makes all variance ratios equal.
The theory guarantees coordinate recovery under Distinct ratios; the controls
remove the separating condition used by that guarantee.

Each drift family is tested at $d\in\{5,7\}$ under two nonlinear mixing
functions.  The main-text setting uses a three-layer MLP with leaky-$\tanh$
activation $\phi(z)=\tanh(z)+0.1z$, yielding a $C^\infty$ diffeomorphism that
satisfies the theory's smoothness assumptions.  We also test LeakyReLU mixing as
a robustness check beyond the $C^2$ setting.  To keep the main text focused, we
report the $d=5$ leaky-$\tanh$ results below; LeakyReLU and $d=7$ results are in
Appendix~\ref{app:simulations}, with the data-generating process, training
pipeline, hyperparameters, and metric definitions in
Appendix~\ref{app:sim_setup} (Table~\ref{tab:hyperparams}).  We report mean
correlation coefficient (MCC) for one-to-one absolute-correlation alignment
between learned and true coordinates; monomial score (Mon.) for how close the
mean absolute Jacobian of $h\circ g$ is to a permutation-scaling matrix; and,
for sparse linear and nonlinear drifts, graph match rate (GMR) for entrywise
agreement between the thresholded learned drift-Jacobian graph and the
ground-truth graph.  Higher is better for all metrics.

Table~\ref{tab:simulation_main} summarizes the quantitative results.  Across all
three drift families, Distinct ratios gives near-perfect MCC and high monomial
scores, whereas both controls show substantially lower coordinate recovery and
less monomial encoder Jacobians.  In the sparse linear and nonlinear families,
Distinct ratios also recovers the drift graph almost exactly.  GMR is omitted
for the dense linear family because sparse graph recovery is not the target
there.  The same qualitative pattern holds at $d=7$ and under LeakyReLU mixing
(Appendix~\ref{app:simulations}).
Appendix~\ref{app:d10_results} further reports a $d=10$
evaluation including a fully dense nonlinear drift,
Appendix~\ref{app:computational_cost} reports measured computational cost, and
Appendix~\ref{app:irregular_sampling} evaluates irregular sampling with a
timestamp-ablated control.

\begin{table*}[t]
  \caption{Synthetic $d=5$ results under three-layer leaky-$\tanh$ MLP mixing.
  Entries are mean $\pm$ s.d. over five training seeds.  \textbf{Bold}
  highlights the identifiable condition (Distinct ratios); the other two
  conditions are controls.  GMR is omitted for the dense linear family because
  sparse graph recovery is not the target there.}
  \label{tab:simulation_main}
  \centering
  \small
  \setlength{\tabcolsep}{6pt}
  \begin{tabular}{llccc}
    \toprule
    Drift family & Regime condition & MCC $\uparrow$ & Mon.\ $\uparrow$ & GMR $\uparrow$ \\
    \midrule
    Dense linear & Distinct ratios & $\mathbf{0.9976 \pm 0.0001}$ & $\mathbf{0.8697 \pm 0.0027}$ & -- \\
    Dense linear & One regime & $0.7819 \pm 0.0160$ & $0.4750 \pm 0.0253$ & -- \\
    Dense linear & Proportional diffusions & $0.8087 \pm 0.0428$ & $0.4940 \pm 0.0344$ & -- \\
    \midrule
    Sparse linear & Distinct ratios & $\mathbf{0.9976 \pm 0.0002}$ & $\mathbf{0.8920 \pm 0.0038}$ & $\mathbf{1.0000 \pm 0.0000}$ \\
    Sparse linear & One regime & $0.7206 \pm 0.0382$ & $0.4131 \pm 0.0349$ & $0.5760 \pm 0.0784$ \\
    Sparse linear & Proportional diffusions & $0.7018 \pm 0.0367$ & $0.4061 \pm 0.0269$ & $0.6400 \pm 0.0566$ \\
    \midrule
    Nonlinear & Distinct ratios & $\mathbf{0.9977 \pm 0.0001}$ & $\mathbf{0.8924 \pm 0.0035}$ & $\mathbf{1.0000 \pm 0.0000}$ \\
    Nonlinear & One regime & $0.7331 \pm 0.0232$ & $0.4234 \pm 0.0280$ & $0.6800 \pm 0.0506$ \\
    Nonlinear & Proportional diffusions & $0.7531 \pm 0.0141$ & $0.4386 \pm 0.0248$ & $0.7840 \pm 0.0742$ \\
    \bottomrule
  \end{tabular}
\end{table*}

Figures~\ref{fig:scatter} and~\ref{fig:diagnostics} give qualitative diagnostics
for the nonlinear $d=5$ setting, covering both latent disentanglement and graph
recovery.  Each column shows one regime condition, using the same randomly
selected training seed across conditions.  Figure~\ref{fig:scatter} shows that
Distinct ratios gives a near-permutation alignment between true and learned
coordinates, while the controls do not.  Figure~\ref{fig:diagnostics} shows that
the corresponding $\overline{|D\varphi|}$ matrix is near-monomial and the drift
graph is recovered exactly under Distinct ratios; the controls produce denser
encoder Jacobians and spurious or missed graph edges.  Analogous figures for
other drift families, dimensions, and mixing functions are in
Appendix~\ref{app:simulations}.

\begin{figure*}[t]
  \centering
  \includegraphics[width=0.99\textwidth]{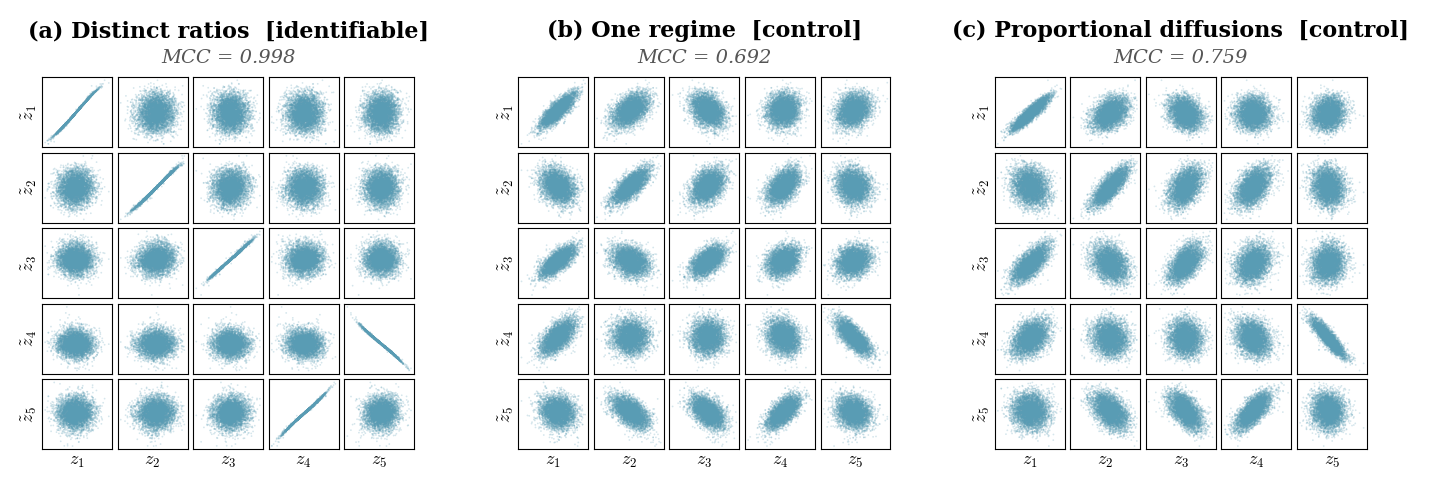}
  \caption{Latent disentanglement in the nonlinear $d=5$ setting (three-layer
  leaky-$\tanh$ MLP mixing). Each column shows the $5\times5$ scatter matrix of
  true coordinates $z_j$ (horizontal) versus learned coordinates $\tilde z_i$
  (vertical) for one regime condition. Distinct ratios yields the expected
  near-permutation alignment; the controls do not recover a clean one-to-one
  alignment.}
  \label{fig:scatter}
\end{figure*}

\begin{figure*}[t]
  \centering
  \includegraphics[width=0.99\textwidth]{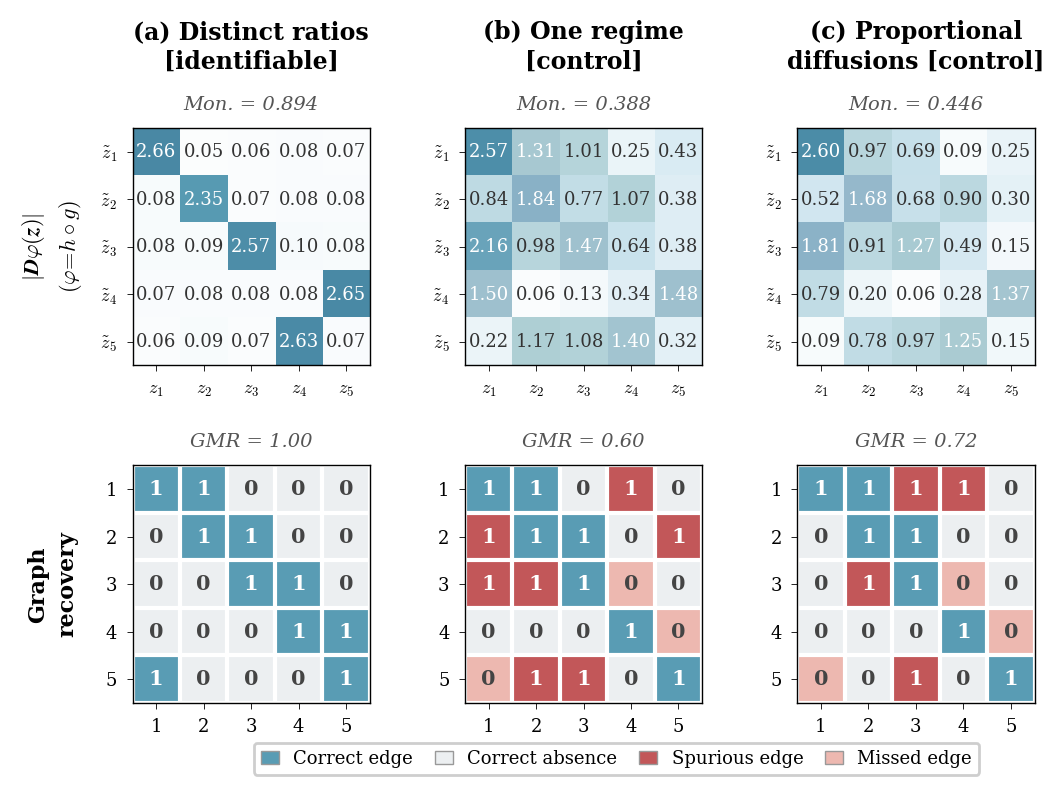}
  \caption{Structural diagnostics for the same setting and run as
  Figure~\ref{fig:scatter}. \textbf{Top row}: mean absolute encoder Jacobian
  $\overline{|D\varphi|}$ for $\varphi=h\circ g$; near-monomial structure
  indicates inversion up to permutation and scaling. \textbf{Bottom row}: learned
  versus ground-truth drift graph, with correct edges/absences in blue/gray and
  spurious/missed edges in red/pink. Distinct ratios achieves an exact graph
  match; both controls introduce graph errors.}
  \label{fig:diagnostics}
\end{figure*}

\subsection{Real-world case study}
\label{subsec:hardanger}

We further test the method on operational accelerometer recordings from the
Hardanger Bridge, a long-span suspension bridge instrumented for long-term wind
and vibration monitoring \cite{fenerci2020wind,fenerci2021data}.  The model is
trained only on bridge-deck acceleration channels; wind measurements define
regimes but are not used as model inputs.  This also probes robustness beyond
the square diffeomorphic setting of Assumption~(S1): the input has $16$ sensor
channels and the learned latent space has dimension $d=5$.  To obtain a
real-world analogue of the diffusion-separation condition in Assumption~(S4),
we split segments within a
matched wind-speed bin by turbulence intensity (TI).  This keeps mean wind
speed fixed while changing stochastic excitation, so low- and high-TI segments
are intended to share structural dynamics while carrying different diffusion
fingerprints.

We compare three regime constructions: \emph{TI-split regimes (ours)}, using
the low- and high-TI groups as two regimes; \emph{One regime}, merging the two
TI groups and removing regime-specific diffusion variation; and \emph{Shuffled
TI labels}, preserving the two-regime training format but destroying the
turbulence-based contrast.  Since ground-truth latent coordinates and graphs are
unavailable, we evaluate cross-seed reproducibility rather than accuracy, using
pairwise mean matched absolute correlation coefficient (MCC) and graph match
rate (GMR).  Table~\ref{tab:hardanger} reports results over $K=10$ seeds at
$d=5$.  The TI-split construction gives the highest MCC and GMR, consistent with
the synthetic gap between identifiable diffusion shifts and controls.
Appendix~\ref{app:hardanger} gives the dataset construction, training protocol,
aggregate diagnostics, Assumption~(S4)-style diffusion-fingerprint checks, a
representative drift-Jacobian interpretation, and limitations of the real-world
diagnostic.

\begin{table}[h]
  \caption{Hardanger Bridge accelerometer case study
  \cite{fenerci2020wind,fenerci2021data} ($K=10$ seeds, $d=5$).  Entries are
  mean pairwise cross-seed reproducibility scores (MCC, or GMR with
  $\tau=0.05$) over $\binom{K}{2}=45$ seed pairs.}
  \label{tab:hardanger}
  \centering
  \small
  \begin{tabular}{lcc}
    \toprule
    Condition                   & MCC $\uparrow$  & GMR $\uparrow$  \\
    \midrule
    TI-split regimes (ours)     & $\mathbf{0.806}$ & $\mathbf{0.877}$ \\
    One regime (control)        & $0.660$          & $0.813$          \\
    Shuffled TI labels (control) & $0.601$          & $0.830$          \\
    \bottomrule
  \end{tabular}
\end{table}

\section{Related work}
\label{sec:related}

Identifiable nonlinear ICA and temporal CRL show that auxiliary structure, such
as nonstationarity, temporal dependence, interventions, grouping, or environment
variation, can break nonlinear mixing ambiguities
\cite{hyvarinen2019nonlinear,khemakhem2020variational,yaolearning,yao2022temporally,song2024causal,lippe2022citris,chen2024caring}.

Their theorem-level temporal results, however, are largely discrete-time, based
on finite-lag transitions, autoregressive mechanisms, or temporal context on a
grid. Continuous-time SDE identification and causal discovery usually assume
that the stochastic state is directly observed
\cite{wang2023generator,wangneural,guan2024identifying,mantensignature,zweig2025towards};
the closest latent continuous-time work treats stochastic point processes or
system-level parameters rather than latent SDE state coordinates
\cite{ren2026causal,yao2024marrying,baumgartner2026disentangling}.
Recent work also studies identifiability of second-order ODE
parameters from video \cite{wang2026physicsvideo}, a deterministic
system-parameter target rather than latent SDE state coordinates. We address
this gap by identifying the latent coordinate system and drift-Jacobian graph of
continuous-time SDE states observed through an unknown nonlinear map, using two
environments with shared drift and different diffusion covariances.
Appendix~\ref{app:extended_related_work} gives a fuller comparison.

\section{Discussion and limitations}
\label{sec:discussion}

This paper gives a theorem-level account of identifiable CRL for
continuous-time latent SDEs. The main message is that \emph{multi-regime
diffusion structure} can anchor the latent coordinate system: with shared drift
and diagonal diffusion covariances whose coordinate-wise variance ratios are
distinct, nonlinear observations identify the coordinates up to a monomial affine
transformation, and the instantaneous drift-Jacobian graph at corresponding
states follows up to the same permutation, without drift sparsity. This is a step toward CRL for continuous
latent dynamics, where most temporal identifiability results are discrete-time
and most continuous-time SDE identification assumes observed states. The
assumptions are intentionally simple: an invertible observation map, additive
state-independent diffusion, a shared drift, and separated diffusion regimes.
The diffeomorphism assumption is the strongest idealization, excluding
non-invertible or nonsmooth observation maps. Our LeakyReLU experiments and
Hardanger study probe robustness beyond the smooth square diffeomorphic setting.
Extending the theory to weaker observation maps, state-dependent noise, and
imperfect regime labels remains important future work.

\begin{ack}
HL was supported by the Shanghai General AI Foundation Models Program (No. 2025SHZDZX026D35), the Beijing Natural Science Foundation (No. L257007), the Beijing Major Science and Technology Project (No. Z251100008425006), and the CCF-Kuaishou Large Model Explorer Fund (No. CCF-KuaiShou 2026018). MG was supported by ARC DP240102088 and WIS--MBZUAI 142571.
KZ would like to acknowledge support from NSF Award No.~2229881, the AI
Institute for Societal Decision Making (AI-SDM), the National Institutes of
Health (NIH) under Contract R01HL159805, and grants from Quris AI, Florin Court
Capital, the MBZUAI--WIS Joint Program, and the Al Deira Causal Education
project. The authors declare no competing interests.
\end{ack}

\clearpage
\bibliographystyle{abbrv}
\bibliography{references}

\newpage
\appendix

\section{Notation}
\label{app:notation}

\begin{table}[H]
  \caption{Summary of notation.}
  \label{tab:notation}
  \centering
  \small
  \setlength{\tabcolsep}{4pt}
  \renewcommand{\arraystretch}{0.94}
  \begin{tabular}{ll}
    \toprule
    \textbf{Symbol} & \textbf{Description} \\
    \midrule
    \multicolumn{2}{l}{\textit{Spaces and indices}} \\
    $d$                          & Latent dimension (also the observed dimension in the square theory and simulations) \\
    $d_x,d_z$                    & Observation and latent dimensions when they differ \\
    $e \in \{1,2\}$             & Environment (regime) index \\
    $t \in [0, T]$              & Continuous time; $T$ is the time horizon \\
    $\Delta t$                   & Observation interval (constant in the regular-grid experiments) \\
    $\Delta t'=s\Delta t$        & Stage~2 velocity-target interval at stride $s$ \\
    \midrule
    \multicolumn{2}{l}{\textit{Processes}} \\
    $z_t^{(e)} \in \R^d$        & Latent state in environment $e$ at time $t$ \\
    $x_t^{(e)} \in \R^d$        & Observed state in the square theory: $x_t^{(e)} = g(z_t^{(e)})$ \\
    $\tilde{z}_t^{(e)} \in \R^d$ & Alternative latent state under the tilded parameterization \\
    $W_t^{(e)}$                  & Standard $d$-dimensional Brownian motion \\
    \midrule
    \multicolumn{2}{l}{\textit{SDE components}} \\
    $f \colon \R^d \to \R^d$    & Latent drift function (shared across environments) \\
    $A \in \R^{d \times d}$     & Drift matrix in the linear (OU) case: $f(z) = Az$ \\
    $G_e \in \R^{d \times d}$   & Diffusion matrix in environment $e$ \\
    $\Sigma_e = G_e G_e^\top$   & Diffusion covariance in environment $e$ \\
    $\sigma_{e,i}^2$            & Diagonal entry of $\Sigma_e$ (variance of coordinate $i$ in environment $e$) \\
    \midrule
    \multicolumn{2}{l}{\textit{Mixing and coordinate change}} \\
    $g \colon \R^d \to \R^d$    & Unknown nonlinear observation map in the square theory (mixing diffeomorphism) \\
    $\varphi = \tilde{g}^{-1} \circ g$ & Latent coordinate change between two parameterizations \\
    $L = \Pi\Lambda$            & Monomial matrix (permutation and coordinate-wise scaling) \\
    $\Pi$                        & Permutation matrix \\
    $\Lambda$                    & Diagonal scaling matrix \\
    $c \in \R^d$                & Constant shift: $\varphi(z) = Lz + c$ \\
    \midrule
    \multicolumn{2}{l}{\textit{Jacobians and causal graph}} \\
    $Dg(z)$                      & Jacobian of $g$ at $z$ \\
    $D\varphi(z)$               & Jacobian of the coordinate change $\varphi$ at $z$ \\
    $Df(z)$                      & Jacobian of the latent drift $f$ at $z$ \\
    $E_f(z)$                     & Instantaneous drift-Jacobian graph: $\{(j \to i) : (Df(z))_{ij} \neq 0\}$ \\
    $\hat{E}_f$                  & Estimated drift-Jacobian graph (from Stage~2) \\
    $\overline{|D\tilde{f}_\psi|}$ & Mean absolute Jacobian of the learned drift \\
    \midrule
    \multicolumn{2}{l}{\textit{Estimation}} \\
    $h_\theta \colon \R^d \to \R^d$ & Encoder in the square theory and simulations \\
    $\tilde{f}_\psi \colon \R^d \to \R^d$ & Learned drift model \\
    $\tilde{\sigma}_{e,i}^2$    & Learned diffusion variance for coordinate $i$ in environment $e$ \\
    $\hat{v}_t$                  & Velocity target: $(h_\theta(x_{t+\Delta t'}) - h_\theta(x_t))/\Delta t'$ \\
    $\mathcal{L}_{\mathrm{S1}}$ & Stage~1 objective (transition negative log-likelihood) \\
    $\mathcal{L}_{\mathrm{S2}}$ & Stage~2 objective (velocity MSE $+$ Jacobian $L_1$ penalty) \\
    $\lambda_{\mathrm{sparse}}$ & Sparsity penalty weight (Stage~2) \\
    $\tau$                       & Threshold for graph extraction from $\overline{|D\tilde{f}_\psi|}$ \\
    $\mathcal P_e$               & Empirical distribution of observed transition pairs in regime $e$ \\
    $\mathcal V_e$               & Empirical distribution of encoded state--velocity pairs in regime $e$ \\
    $\mathcal Z$                 & Equally weighted empirical distribution of encoded states across regimes \\
    \midrule
    \multicolumn{2}{l}{\textit{Evaluation metrics}} \\
    MCC                          & Mean correlation coefficient for learned/true coordinate alignment \\
    Mon.                         & Monomial score for the mean absolute encoder Jacobian \\
    GMR                          & Graph match rate for thresholded drift-Jacobian graphs \\
    \bottomrule
  \end{tabular}
\end{table}

\newpage
\section{Extended related work}
\label{app:extended_related_work}

\paragraph{Identifiable nonlinear ICA.}
General nonlinear mixing is not identifiable without extra structure \cite{hyvarinen1999nonlinear,locatello2019challenging}. Identifiable nonlinear ICA uses nonstationarity, temporal dependence, or auxiliary variables \cite{hyvarinen2016unsupervised,hyvarinen2017nonlinear,hyvarinen2019nonlinear,khemakhem2020variational}, with extensions to hidden Markov and structured dependent sources \cite{halva2020hidden,halva2021disentangling}. These results show how suitable side information can break nonlinear mixing ambiguity.

Independent Mechanism Analysis (IMA) \cite{gresele2021independent} imposes orthogonality on the columns of the mixing Jacobian. In contrast, Assumption~(S3) constrains the latent diffusion covariance while leaving $g$ arbitrary among $C^2$ diffeomorphisms; orthogonality arises only after whitening the diffusion geometry.

\paragraph{Temporal CRL in discrete time.}
A closely related CRL line studies nonlinear mixtures of latent temporal processes in discrete time. LEAP identifies latent temporal causal processes under nonstationary noise \cite{yaolearning}; TDRL studies nonparametric time-delayed causal processes \cite{yao2022temporally}; NCTRL addresses unknown nonstationarity \cite{song2023temporally}; and CtrlNS uses sparse-transition structure to recover nonstationary latent dynamics without known domain variables \cite{song2024causal}. A parallel branch uses stronger side information: CITRIS and iCITRIS leverage interventions \cite{lippe2022citris,lippecausal}, G-CaRL uses grouping information \cite{morioka2024causal}, and CaRiNG allows non-invertible temporal generation processes by using temporal context to recover lost information \cite{chen2024caring}. More recent extensions study instantaneous dependence via sparse latent dynamics \cite{liidentification}, hierarchical temporal latents \cite{litowards}, and continuously evolving mechanisms rather than discrete switches \cite{fan2026trace}. These papers identify latent causal processes from sequential observations, but their arguments use discrete observations, transition factorizations, or auxiliary mechanism structure. Our proof instead identifies a continuous-time latent SDE path through diffusion covariance changes.

\paragraph{Multi-environment identification and partial observability in CRL.}
Environment heterogeneity and observation structure are also important sources of identifiability in CRL. Multi-environment results address unknown interventions \cite{von2023nonparametric}, cross-domain invariance \cite{ahuja2024multi}, general environment families \cite{zhang2024causal,varici2024general,ng2025causal}, and limits of observational identification \cite{welch2024identifiability}; multi-view and partial-observability methods exploit additional observation structure \cite{yaomulti,xusparsity}. Complementary latent-variable causal-discovery work uses relaxed pure-child or homologous-surrogate conditions \cite{li2024causalstructure,li2025homologous}; its information comes from linear cross-sectional measurement and graph structure rather than temporal diffusion changes. Other work connects invariance-based CRL with exchangeability and causal discovery \cite{yaounifying,reizingeridentifiable}. Here the drift is invariant across environments, while distinct diffusion-variance ratios identify the coordinates up to a monomial affine transformation.

\paragraph{Continuous-time identification and causal discovery for stochastic processes.}
A separate literature studies what can be identified when the stochastic system itself is observed. Wang et al. analyze generator identifiability for linear SDEs \cite{wang2023generator}. SCOTCH studies observed-space structure learning with neural SDEs under irregular sampling \cite{wangneural}. Guan et al. identify linear drift, diffusion, and causal structure from temporal marginals \cite{guan2024identifying}. Manten et al. derive path-space conditional independence constraints and corresponding tests for SDE models \cite{mantensignature}, and extend the graphical treatment to asymmetric independence on path spaces \cite{manten2025asymmetric}. Zweig et al. analyze identifiability of interventional SDEs from stationary distributions \cite{zweig2025towards}. Related deterministic-system results establish identifiability for homogeneous linear ODEs from discrete observations, linear ODEs with hidden confounders, and second-order ODE parameters from video \cite{wang2024homogeneousode,wang2024hiddenode,wang2026physicsvideo}; these targets differ from latent SDE state coordinates and their instantaneous drift graphs. Among continuous-time latent-process CRL papers, Ren et al. \cite{ren2026causal} study identifiable CRL for continuous-time latent stochastic point processes, not latent SDE state coordinates. Closely related dynamical-systems papers connect CRL with trajectory-specific or system-parameter identification \cite{yao2024marrying,baumgartner2026disentangling}. Our paper differs from all of these in the recovery target: we identify continuously evolving latent \emph{state coordinates} and their instantaneous Jacobian graph from unknown nonlinear mixtures, using two environments with shared drift and changed diffusion.

\newpage
\section{Deferred technical material}
\label{app:deferred_proofs}

\subsection{Auxiliary lemmas}
\label{app:aux_lemmas}

\begin{lemma}[Orthogonal conjugation of a diagonal matrix with simple spectrum]
\label{lem:orth_conj_simple_diag}
Let $K=\mathrm{diag}(k_1,\dots,k_d)\in\R^{d\times d}$ have real pairwise
distinct diagonal entries, $k_i\neq k_j$ for all $i\neq j$.
Let $Q\in \R^{d\times d}$ be an orthogonal matrix.
If $QKQ^\top$ is diagonal, then $Q$ is a signed permutation matrix: there
exists a permutation $\pi$ and signs $s_i\in\{\pm 1\}$ such that
\[
Q e_i = s_i e_{\pi(i)}\quad\text{for all }i,
\]
equivalently $Q = \Pi S$ where $\Pi$ is a permutation matrix and $S$ is diagonal
with $\pm 1$ entries. In particular, $QKQ^\top = \Pi K \Pi^\top$.
\end{lemma}

\begin{proof}
Because $K$ is diagonal with pairwise distinct diagonal entries, its eigenvalues
are exactly $\{k_1,\dots,k_d\}$ with algebraic (and geometric) multiplicity $1$
each. Hence the eigenspace associated with $k_i$ is one-dimensional and equals
$\mathrm{span}\{e_i\}$.

Consider $\widehat K := QKQ^\top$. Since $Q$ is orthogonal, $\widehat K$ is
similar to $K$, so it has the same multiset of eigenvalues and hence also has
$d$ distinct eigenvalues. Because $\widehat K$ is diagonal by assumption, each
eigenvalue appears as a diagonal entry of $\widehat K$, and its eigenspace is
spanned by a standard basis vector. Concretely, there exists a permutation
$\pi$ such that $\widehat K_{\pi(i)\pi(i)}=k_i$ for all $i$.

Now note that
\[
\widehat K(Qe_i) = QK e_i = Q(k_i e_i) = k_i (Qe_i),
\]
so $Qe_i$ is an eigenvector of $\widehat K$ with eigenvalue $k_i$. But the
eigenspace of $\widehat K$ associated with $k_i$ is
$\mathrm{span}\{e_{\pi(i)}\}$, therefore $Qe_i \in
\mathrm{span}\{e_{\pi(i)}\}$. Since $Q$ is orthogonal,
$\|Qe_i\|=\|e_i\|=1$, hence $Qe_i = s_i e_{\pi(i)}$ with
$s_i\in\{\pm 1\}$. This shows $Q$ is a signed permutation matrix, and the
final statement $QKQ^\top=\Pi K\Pi^\top$ follows because the sign diagonal
cancels in conjugation.
\end{proof}

\begin{lemma}[Innovation representation in the natural filtration]
\label{lem:innovation}
Let $(z_t)_{t\in[0,T]}$ be a solution (strong or weak) on some filtered
probability space $(\Omega,\mathcal{F},(\mathcal{F}_t),\mathbb{P})$ to
\[
dz_t = f(z_t)\,dt + G\,dW_t,
\]
where $f:\R^d\to\R^d$ is continuous, $G\in\R^{d\times d}$ is constant, and
$\Sigma:=GG^\top$ is positive definite. Let $(\mathcal{F}_t^z)$ denote the
usual augmentation (completed and right-continuous) of the natural filtration
$\sigma(z_s:0\le s\le t)$. Then the process
\[
\hat W_t := \Sigma^{-1/2}\!\Big(z_t-z_0-\int_0^t f(z_s)\,ds\Big),
\qquad t\in[0,T],
\]
is a standard $d$-dimensional Brownian motion with respect to
$(\mathcal{F}_t^z)$, and $z$ satisfies
\[
dz_t = f(z_t)\,dt + \Sigma^{1/2}\,d\hat W_t
\]
in the natural filtration $(\mathcal{F}_t^z)$.
\end{lemma}

\begin{proof}
Set $M_t := z_t-z_0-\int_0^t f(z_s)\,ds = \int_0^t G\,dW_s$. Then $M$ is a
continuous $\mathcal{F}_t$-martingale with $\langle M\rangle_t = \Sigma t$, and
$M_t$ is $\mathcal{F}_t^z$-adapted because it is a measurable function of the
path of $z$. Since $\mathcal{F}_t^z \subset \mathcal{F}_t$, the tower property
gives, for $s\le t$,
\[
\E[M_t\mid \mathcal{F}_s^z]
= \E\!\big[\E[M_t\mid \mathcal{F}_s]\mid \mathcal{F}_s^z\big]
= \E[M_s\mid \mathcal{F}_s^z]
= M_s,
\]
so $M$ is a continuous $\mathcal{F}_t^z$-martingale. Setting
$\hat W_t := \Sigma^{-1/2}M_t$, we have
$\langle \hat W\rangle_t = \Sigma^{-1/2}\Sigma\Sigma^{-1/2} t = I_d t$.
By L\'evy's characterization theorem, $\hat W$ is a standard
$\mathcal{F}_t^z$-Brownian motion.
\end{proof}

\begin{remark}[Common natural filtration under diffeomorphic reparameterization]
\label{rem:common_filtration}
If $\tilde z_t = \varphi(z_t)$ where $\varphi:\R^d\to\R^d$ is a homeomorphism
(in particular, a diffeomorphism), then
$\mathcal{F}_t^z = \mathcal{F}_t^{\tilde z}$: indeed
$\tilde z_s = \varphi(z_s)$ gives $\mathcal{F}_t^{\tilde z}\subset
\mathcal{F}_t^z$, and $z_s = \varphi^{-1}(\tilde z_s)$ gives
$\mathcal{F}_t^z\subset \mathcal{F}_t^{\tilde z}$. In particular, a process
that is a (local) martingale with respect to one filtration is automatically a
(local) martingale with respect to the other.
\end{remark}

\subsection{Proof of Theorem~\ref{thm:ou_two_regimes_monomial}}
\label{app:proof_ou}

\begin{proof}[Proof of Theorem~\ref{thm:ou_two_regimes_monomial}]
We prove that equality in law of observed OU trajectories across two diffusion
regimes forces the latent reparameterization
$\varphi=\tilde g^{-1}\circ g$ to be affine with a constant monomial Jacobian.

\medskip
\noindent\textbf{Step 0: Coupling and a pathwise state relation.}

Fix a regime $e\in\{1,2\}$. By (S2), the observed trajectories
$\{x_t^{(e)}\}_{t\in[0,T]}$ and $\{\tilde x_t^{(e)}\}_{t\in[0,T]}$ have the
same law on the Polish path space $C([0,T];\R^d)$. We may therefore realize
both as the canonical coordinate process on $C([0,T];\R^d)$ equipped with
their common law, yielding a single process $\{X_t^{(e)}\}_{t\in[0,T]}$
satisfying
\[
X^{(e)} = x^{(e)} = \tilde x^{(e)}
\quad \text{a.s. as elements of } C([0,T];\R^d).
\]
Since $g$ and $\tilde g$ are diffeomorphisms $\R^d\to\R^d$ by (S1), we
can invert $x_t^{(e)}=g(z_t^{(e)})$ and
$\tilde x_t^{(e)}=\tilde g(\tilde z_t^{(e)})$ pointwise to obtain
\[
  z_t^{(e)} = g^{-1}(X_t^{(e)}),
  \qquad
  \tilde z_t^{(e)} = \tilde g^{-1}(X_t^{(e)}),
  \qquad t\in[0,T].
\]
Therefore, with $\varphi:=\tilde g^{-1}\circ g$,
\begin{equation}
\label{eq:pathwise_phi_relation_sde}
  \tilde z_t^{(e)} = \varphi(z_t^{(e)})
  \qquad \text{a.s. for all } t\in[0,T].
\end{equation}

\medskip
\noindent\textbf{Step 1: Quadratic variation identifies an equation for $D\varphi$.}

Fix a regime $e$. Recall that under regime $e$ the latent OU satisfies
\[
dz_t^{(e)} = A z_t^{(e)}\,dt + G_e\,dW_t^{(e)},
\qquad
\Sigma_e := G_eG_e^\top.
\]
Since $\varphi\in C^2$ by (S1), It\^o's formula applies to
$\tilde z_t^{(e)}=\varphi(z_t^{(e)})$ and gives
\begin{align*}
d\tilde z_t^{(e)}
&=
D\varphi(z_t^{(e)})\,dz_t^{(e)}
+
\frac{1}{2}\sum_{k,\ell=1}^d (\Sigma_e)_{k\ell}
\,\partial_{k\ell}\varphi(z_t^{(e)})\,dt \\
&=
\Big(D\varphi(z_t^{(e)})A z_t^{(e)}
+\frac{1}{2}\sum_{k,\ell=1}^d (\Sigma_e)_{k\ell}
\,\partial_{k\ell}\varphi(z_t^{(e)})\Big)dt
+
D\varphi(z_t^{(e)})G_e\,dW_t^{(e)}.
\end{align*}

If a $d$-dimensional process $Y_t$ has the form
$dY_t = b_t\,dt + H_t\,dW_t$, then only the noise term contributes to
quadratic variation, and
\[
d\langle Y\rangle_t = H_tH_t^\top\,dt.
\]
Applying this with $Y_t=\tilde z_t^{(e)}$ and
$H_t=D\varphi(z_t^{(e)})G_e$, we obtain
\begin{equation}
\label{eq:QV_phi1}
  d\langle \tilde z^{(e)}\rangle_t
  =
  D\varphi(z_t^{(e)})\,\Sigma_e\,D\varphi(z_t^{(e)})^\top \, dt.
\end{equation}
On the other hand, under the alternative OU model \eqref{eq:ou_alt}, the
diffusion covariance is constant and equals
$\tilde\Sigma_e=\tilde G_e\tilde G_e^\top$, hence
\begin{equation}
\label{eq:QV_tilde}
  d\langle \tilde z^{(e)}\rangle_t = \tilde\Sigma_e\,dt.
\end{equation}
Equating \eqref{eq:QV_phi1} and \eqref{eq:QV_tilde} yields
\begin{equation}
\label{eq:Dphi_congruence_along_path}
  D\varphi(z_t^{(e)})\,\Sigma_e\,D\varphi(z_t^{(e)})^\top
  =
  \tilde\Sigma_e,
\end{equation}
holding for Lebesgue-a.e. $t\in[0,T]$ and almost surely.

Define the matrix-valued function along each sample path
\[
F(t,\omega):=D\varphi(z_t^{(e)}(\omega))\,\Sigma_e\,D\varphi(z_t^{(e)}(\omega))^\top.
\]
Because $t\mapsto z_t^{(e)}$ is continuous a.s. and $D\varphi$ is continuous,
$t\mapsto F(t,\omega)$ is continuous on $[0,T]$ for almost every $\omega$.
Moreover, \eqref{eq:Dphi_congruence_along_path} holds for a.e. $t$ and a.s. in
$\omega$; by a standard Fubini-type argument, there exists an event $\Omega_0$
with $\mathbb{P}(\Omega_0)=1$ such that for each $\omega\in\Omega_0$,
\eqref{eq:Dphi_congruence_along_path} holds for Lebesgue-a.e. $t\in[0,T]$.
Fix such an $\omega$. Since $F(\cdot,\omega)$ is continuous and equals the
constant matrix $\tilde\Sigma_e$ for a.e. $t$, it must equal $\tilde\Sigma_e$
for all $t\in[0,T]$. Therefore
\eqref{eq:Dphi_congruence_along_path} holds for all $t\in[0,T]$ a.s.

\medskip
\noindent\textbf{Step 2: Strictly positive density upgrades the pathwise identity to a pointwise identity on $\R^d$.}

Fix $e\in\{1,2\}$ and any $t_\star\in(0,T]$. Define the continuous
matrix-valued function on $\R^d$ by
\[
B_e(z):=D\varphi(z)\,\Sigma_e\,D\varphi(z)^\top-\tilde\Sigma_e.
\]
Evaluating \eqref{eq:Dphi_congruence_along_path} at $t=t_\star$ gives
$B_e(z_{t_\star}^{(e)})=0$ a.s.

We claim this implies $B_e(z)\equiv 0$ on $\R^d$. Suppose not. Then there
exists $z_0\in\R^d$ with $B_e(z_0)\neq 0$. By continuity of $B_e$, there exist
$\varepsilon>0$ and an open neighborhood $V$ of $z_0$ such that
$\|B_e(z)\|_F>\varepsilon$ for all $z\in V$. Since $V$ is open, it has positive
Lebesgue measure. Now, the OU solution satisfies
$z_{t_\star}^{(e)}=e^{At_\star}z_0^{(e)}+\xi_{t_\star}^{(e)}$, where
$\xi_{t_\star}^{(e)}:=\int_0^{t_\star} e^{A(t_\star-s)}G_e\,dW_s^{(e)}$ is a
zero-mean Gaussian with covariance
\[
P_{t_\star}
=
\int_0^{t_\star} e^{A(t_\star-s)}\Sigma_e e^{A^\top(t_\star-s)}\,ds \succ 0.
\]
Indeed, for any $v\neq 0$,
\[
v^\top P_{t_\star}v
=
\int_0^{t_\star}\|\Sigma_e^{1/2}e^{A^\top(t_\star-s)}v\|^2\,ds > 0,
\]
since $\Sigma_e^{1/2}$ is invertible and
$s\mapsto e^{A^\top(t_\star-s)}v$ is continuous and nonvanishing. Hence the
conditional density $p(z_{t_\star}^{(e)}\mid z_0^{(e)})$ is strictly positive on
all of $\R^d$ for every value of $z_0^{(e)}$, and therefore the marginal
density of $z_{t_\star}^{(e)}$ is also strictly positive on $\R^d$. Thus
\[
\mathbb{P}(z_{t_\star}^{(e)}\in V)>0,
\]
contradicting $B_e(z_{t_\star}^{(e)})=0$ a.s. Hence $B_e\equiv 0$ on $\R^d$,
i.e.
\begin{equation}
\label{eq:Dphi_congruence_on_U}
  D\varphi(z)\,\Sigma_e\,D\varphi(z)^\top = \tilde\Sigma_e,
  \qquad \forall z\in \R^d,\ \ e\in\{1,2\}.
\end{equation}

\medskip
\noindent\textbf{Step 3: Reduce to an orthogonal conjugation constraint.}

Use \eqref{eq:Dphi_congruence_on_U} with $e=1$. Since $\Sigma_1$ and
$\tilde\Sigma_1$ are diagonal positive definite by (S3), their square
roots $\Sigma_1^{1/2}$ and $\tilde\Sigma_1^{1/2}$ are well-defined diagonal
positive definite. Define
\[
  M(z) := \tilde\Sigma_1^{-1/2}\,D\varphi(z)\,\Sigma_1^{1/2}.
\]
Then $M(z)M(z)^\top=I$, so $M(z)\in \R^{d\times d}$ is an orthogonal matrix for
all $z\in\R^d$. Using \eqref{eq:Dphi_congruence_on_U} with $e=2$ yields
\begin{equation}
\label{eq:orth_conj_K_to_Ktilde}
  M(z)\,K\,M(z)^\top = \tilde K,
\end{equation}
where
\[
  K:=\Sigma_1^{-1/2}\Sigma_2\Sigma_1^{-1/2}
  = \mathrm{diag}\!\Big(\frac{\sigma_{2,1}^2}{\sigma_{1,1}^2},\dots,\frac{\sigma_{2,d}^2}{\sigma_{1,d}^2}\Big),
\qquad
  \tilde K:=\tilde\Sigma_1^{-1/2}\tilde\Sigma_2\tilde\Sigma_1^{-1/2},
\]
and $\tilde K$ is diagonal positive definite. By (S4), $K$ has simple
spectrum.

\medskip
\noindent\textbf{Step 4: Conclude $M(z)$ is a signed permutation, hence $D\varphi(z)$ is monomial.}

Because $M(z)\in \R^{d\times d}$ is an orthogonal matrix and
$M(z)KM(z)^\top$ is diagonal, Lemma~\ref{lem:orth_conj_simple_diag} implies that
$M(z)$ is a signed permutation matrix for each $z\in\R^d$. Hence
\[
  D\varphi(z)=\tilde\Sigma_1^{1/2} M(z)\Sigma_1^{-1/2}
\]
is monomial for each $z\in\R^d$.

\medskip
\noindent\textbf{Step 5: Connectedness forces $D\varphi$ to be constant on $\R^d$.}

The map $z\mapsto M(z)$ is continuous on $\R^d$ and takes values in a finite
(discrete) set (signed permutations), so it must be constant on connected
$\R^d$. Thus $M(z)\equiv M_0$ for a fixed signed permutation $M_0$, and
therefore
\[
  D\varphi(z)\equiv L := \tilde\Sigma_1^{1/2}M_0\Sigma_1^{-1/2}
  \qquad \forall z\in\R^d.
\]
Since $\Sigma_1^{1/2}$ and $\tilde\Sigma_1^{1/2}$ are diagonal and $M_0$ is a
signed permutation, $L$ is a constant monomial matrix.

\medskip
\noindent\textbf{Step 6: Constant Jacobian implies $\varphi$ is affine; identify $\tilde A$.}

Since $D\varphi\equiv L$ on $\R^d$, $\varphi$ is affine:
\[
\varphi(z)=Lz+c,
\qquad z\in\R^d,
\]
for some constant $c\in\R^d$.

Fix a regime $e$. Since $\varphi$ is a global $C^2$ diffeomorphism
by (S1), It\^o's formula applied to
$\tilde z_t^{(e)} = \varphi(z_t^{(e)})$ gives
\begin{align*}
d\tilde z_t^{(e)}
&= D\varphi(z_t^{(e)})\,dz_t^{(e)}
 + \frac{1}{2} \sum_{k,\ell=1}^d (\Sigma_e)_{k\ell}
   \partial_{k\ell}\varphi(z_t^{(e)})\,dt \\
&= \Big(D\varphi(z_t^{(e)})A z_t^{(e)}
 + \tfrac{1}{2}\,\Sigma_e : D^2\varphi(z_t^{(e)})\Big)\,dt
 + D\varphi(z_t^{(e)})G_e\,dW_t^{(e)},
\end{align*}
where $\Sigma_e : D^2\varphi$ denotes
$\sum_{k,\ell}(\Sigma_e)_{k\ell}\partial_{k\ell}\varphi$.

On the other hand, the alternative OU model gives
\[
d\tilde z_t^{(e)}=\tilde A \tilde z_t^{(e)}\,dt + \tilde G_e\,d\tilde W_t^{(e)}.
\]

By Lemma~\ref{lem:innovation}, we may replace $G_e\,dW_t^{(e)}$ in the It\^o
expansion with $\Sigma_e^{1/2}\,d\hat W_t^{(e)}$, where $\hat W^{(e)}$ is a
standard $\mathcal{F}_t^{z^{(e)}}$-Brownian motion. The noise term in the
expansion thus becomes
$D\varphi(z_t^{(e)})\Sigma_e^{1/2}\,d\hat W_t^{(e)}$, a continuous
$\mathcal{F}_t^{z^{(e)}}$-local martingale. Likewise,
Lemma~\ref{lem:innovation} applied to the alternative model gives
$d\tilde z_t^{(e)} = \tilde A\tilde z_t^{(e)}\,dt + \tilde\Sigma_e^{1/2}
\,d\hat{\tilde W}_t^{(e)}$ with $\hat{\tilde W}^{(e)}$ an
$\mathcal{F}_t^{\tilde z^{(e)}}$-Brownian motion. By
Remark~\ref{rem:common_filtration},
$\mathcal{F}_t^{z^{(e)}} = \mathcal{F}_t^{\tilde z^{(e)}}$ since
$\tilde z^{(e)} = \varphi(z^{(e)})$ with $\varphi$ a diffeomorphism, so both
noise integrals are continuous local martingales in the common natural
filtration. Subtracting the two representations of $d\tilde z^{(e)}$, the
difference has the form $\int_0^t H(s)\,ds + M_t$, where
$\int_0^t H(s)\,ds$ is of finite variation and $M_t$ is a continuous local
martingale in the common filtration. A continuous process that is both of
finite variation and a local martingale must be constant; since it starts at
$0$, it is identically $0$. Hence $H(t)=0$ for Lebesgue-a.e. $t\in[0,T]$,
almost surely. Comparing drift terms yields
\[
\tilde A\,\tilde z_t^{(e)}
= D\varphi(z_t^{(e)})A z_t^{(e)}
 + \tfrac{1}{2}\,\Sigma_e : D^2\varphi(z_t^{(e)})
\quad\text{for a.e. } t\ \text{and a.s.}
\]
Define
\[
H(t):=\tilde A\,\tilde z_t^{(e)}-
\Big(D\varphi(z_t^{(e)})A z_t^{(e)}+\tfrac12\,\Sigma_e:D^2\varphi(z_t^{(e)})\Big).
\]
We have shown that $H(t)=0$ for a.e. $t\in[0,T]$, a.s. Since
$t\mapsto z_t^{(e)}$ and $t\mapsto \tilde z_t^{(e)}$ are a.s. continuous and
$\varphi\in C^2$, the map $t\mapsto H(t)$ is a.s. continuous. Hence
$H(t)=0$ for all $t\in[0,T]$, a.s. In particular, fix any $t_\star\in(0,T]$;
the identity holds at $t=t_\star$.

Since $D\varphi \equiv L$ and $D^2\varphi \equiv 0$ on $\R^d$ (from Steps~4--5),
evaluating at $t = t_\star$ gives
\[
\tilde A\,\tilde z_{t_\star}^{(e)} = L A z_{t_\star}^{(e)} \quad \text{a.s.}
\]
Since $\tilde z_{t_\star}^{(e)} = \varphi(z_{t_\star}^{(e)}) = Lz_{t_\star}^{(e)}+c$,
this becomes
\[
\tilde A(Lz_{t_\star}^{(e)} + c) = L A z_{t_\star}^{(e)} \quad \text{a.s.}
\]

We now upgrade this to a pointwise identity on $\R^d$. Define the continuous
function $F(z):=\tilde A(Lz+c)-LAz$ on $\R^d$. We have shown
$F(z_{t_\star}^{(e)})=0$ a.s. If $F(z_0)\neq 0$ for some $z_0\in\R^d$, then by
continuity there exists an open set $V\ni z_0$ with $\|F(z)\|>0$ for all
$z\in V$, contradicting $F(z_{t_\star}^{(e)})=0$ a.s. since
$z_{t_\star}^{(e)}$ has a strictly positive density on $\R^d$ (as shown in
Step~2). Therefore,
\[
\tilde A(Lz+c)=LAz
\qquad \forall z\in\R^d.
\]
Expanding the left-hand side gives $\tilde A L z + \tilde A c =LAz$ for all
$z\in\R^d$. Matching coefficients of $z$ and constant terms yields
\[
\tilde A L=LA,
\qquad
\tilde A c=0.
\]
Finally, $\tilde\Sigma_e=L\Sigma_eL^\top$ follows from
\eqref{eq:Dphi_congruence_on_U} with $D\varphi\equiv L$.
Since $L=\Pi\Lambda$ with $\Lambda$ invertible diagonal,
\[
  \tilde A = LAL^{-1}
  = \Pi\big(\Lambda A\Lambda^{-1}\big)\Pi^{-1}.
\]
Left and right multiplication by invertible diagonal matrices preserves the
zero pattern of $A$, while conjugation by $\Pi$ relabels coordinates by the
permutation $\pi$. Therefore
\[
(j\to i)\in E_A
\quad \Longleftrightarrow \quad
(\pi(j)\to \pi(i))\in E_{\tilde A},
\]
which proves the graph-support statement.
\end{proof}

\subsection{Open-set positivity for additive-noise SDEs}
\label{app:open_set_positivity}

\begin{lemma}[Open-set positivity under nondegenerate additive noise]
\label{lem:open_set_positivity}
Let $(z_t)_{t\in[0,T]}$ be a weak solution with continuous sample paths to the
SDE
\[
dz_t = f(z_t)\,dt + \Sigma^{1/2}\,dW_t,
\]
where $f:\R^d\to\R^d$ is continuous, $\Sigma\in\R^{d\times d}$ is positive
definite, and the initial condition $z_0$ has an arbitrary law $\mu$ on $\R^d$.
Then for every $t\in(0,T]$ and every nonempty open set $V\subset\R^d$,
\[
\mathbb{P}(z_t\in V)>0.
\]
Equivalently, the law of $z_t$ charges every nonempty open subset of $\R^d$.
\end{lemma}

\begin{proof}
Fix $t\in(0,T]$ and a nonempty open set $V\subset\R^d$. We first prove the
claim for a deterministic initial condition, and then pass to an arbitrary
initial law by conditioning on $z_0$.

\medskip
\noindent\textbf{Step 1: deterministic initial condition.}
Fix $x\in\R^d$ and suppose first that $z_0=x$ almost surely. Pick $y\in V$ and
choose $r>0$ such that $B(y,r)\subset V$.

Let $(\mathcal F_s)_{0\le s\le T}$ denote the filtration on the underlying
weak-solution probability space. Choose $R>0$ so large that
$x\in B(0,R)$ and $B(y,r)\subset B(0,R)$. Define the exit time
\[
\tau_R := \inf\bigl\{s\in[0,T]: \|z_s\|\ge R\bigr\},
\]
with the convention $\inf\varnothing=\infty$. Since $z$ has continuous adapted
paths and starts from $x\in B(0,R)$, we have $\tau_R>0$ almost surely.

Because $f$ is continuous, it is bounded on the compact set $\overline{B(0,R)}$,
hence
\[
M_R := \sup_{\|u\|\le R}\bigl\|\Sigma^{-1/2}f(u)\bigr\| < \infty.
\]
Define
\[
\theta_s := \Sigma^{-1/2}f(z_s)\,\mathbf{1}_{\{s\le \tau_R\}},
\qquad 0\le s\le t.
\]
Then $\theta$ is progressively measurable and
\[
\int_0^t \|\theta_s\|^2\,ds \le M_R^2 t
\qquad\text{a.s.}
\]
Therefore Novikov's condition holds:
\[
\mathbb{E}\Big[\exp\Big(\frac12\int_0^t \|\theta_s\|^2\,ds\Big)\Big]
\le e^{M_R^2 t/2}<\infty.
\]

Define
\[
\Lambda_t
:=
\exp\biggl(
-\int_0^t \theta_s^\top\,dW_s
-\frac12\int_0^t \|\theta_s\|^2\,ds
\biggr),
\qquad
\frac{d\mathbb Q}{d\mathbb P}\bigg|_{\mathcal F_t} := \Lambda_t.
\]
Then $\Lambda_t>0$ almost surely and $\mathbb E_{\mathbb P}[\Lambda_t]=1$, so
$\mathbb P$ and $\mathbb Q$ are equivalent on $\mathcal F_t$.

By Girsanov's theorem, the process
\[
\widetilde W_s := W_s + \int_0^s \theta_u\,du,
\qquad 0\le s\le t,
\]
is a standard $d$-dimensional Brownian motion under $\mathbb Q$.
Moreover, for every $s\in[0,t]$,
\[
\begin{aligned}
z_{s\wedge \tau_R}
&=
x + \int_0^{s\wedge \tau_R} f(z_u)\,du + \Sigma^{1/2}W_{s\wedge \tau_R}\\
&=
x + \int_0^s f(z_u)\,\mathbf{1}_{\{u\le \tau_R\}}\,du
+ \Sigma^{1/2}\biggl(
\widetilde W_{s\wedge \tau_R}
-
\int_0^{s\wedge \tau_R}\theta_u\,du
\biggr)\\
&=
x + \Sigma^{1/2}\widetilde W_{s\wedge \tau_R},
\end{aligned}
\]
since $\Sigma^{1/2}\theta_u = f(z_u)\,\mathbf{1}_{\{u\le \tau_R\}}$ for
Lebesgue-a.e.\ $u$.

Now define
\[
X_s := x + \Sigma^{1/2}\widetilde W_s,
\qquad 0\le s\le t.
\]
Under $\mathbb Q$, the law of $\widetilde W$ on
$C_0([0,t];\R^d):=\{\eta\in C([0,t];\R^d):\eta(0)=0\}$ is Wiener measure,
which has full support on $C_0([0,t];\R^d)$. Since
\[
X=\Psi(\widetilde W),
\qquad
\Psi(\eta)(s):=x+\Sigma^{1/2}\eta(s),
\]
and $\Psi$ is a homeomorphism from $C_0([0,t];\R^d)$ onto
\[
C_x([0,t];\R^d)
:=
\{\eta\in C([0,t];\R^d):\eta(0)=x\},
\]
the law of $X$ under $\mathbb Q$ has full support on
$C_x([0,t];\R^d)$.

Consider the straight-line path
\[
\gamma(s):= x + \frac{s}{t}(y-x),
\qquad s\in[0,t].
\]
Because $B(0,R)$ is convex and $x,y\in B(0,R)$, we have
$\gamma([0,t])\subset B(0,R)$. Since $\gamma([0,t])$ is compact and $B(0,R)$
is open, there exists $\delta>0$ such that
\[
\delta<r
\qquad\text{and}\qquad
\sup_{0\le s\le t}\|\gamma(s)\|+\delta<R.
\]
Define the event
\[
E := \Bigl\{\sup_{0\le s\le t}\|X_s-\gamma(s)\|<\delta\Bigr\}.
\]
Then $E\in\mathcal F_t$, and by the support statement above,
\[
\mathbb Q(E)>0.
\]

We claim that
\[
E \subset \{z_t\in B(y,r)\}\subset \{z_t\in V\}.
\]
Indeed, fix $\omega\in E$. Then for every $s\in[0,t]$,
\[
\|X_s(\omega)\| \le \|\gamma(s)\|+\delta < R.
\]
If $\tau_R(\omega)\le t$, then by continuity of $z$ and the definition of
$\tau_R$,
\[
\|z_{\tau_R(\omega)}(\omega)\|=R.
\]
But from the stopped identity above,
\[
z_{\tau_R(\omega)}(\omega)
=
z_{\tau_R(\omega)\wedge \tau_R(\omega)}(\omega)
=
X_{\tau_R(\omega)}(\omega),
\]
which contradicts $\|X_{\tau_R(\omega)}(\omega)\|<R$. Hence
$\tau_R(\omega)>t$. Therefore
\[
z_t(\omega)=z_{t\wedge \tau_R}(\omega)=X_t(\omega),
\]
and so
\[
\|z_t(\omega)-y\|
=
\|X_t(\omega)-\gamma(t)\|
<\delta<r.
\]
Thus $z_t(\omega)\in B(y,r)\subset V$, proving the claim.

Since $\mathbb P$ and $\mathbb Q$ are equivalent on $\mathcal F_t$,
\[
\mathbb P(z_t\in V)>0
\]
whenever $z_0=x$ almost surely. This proves the deterministic-start claim.

\medskip
\noindent\textbf{Step 2: arbitrary initial law.}
Now return to the original weak solution with arbitrary initial law
$\mu=\mathcal{L}(z_0)$. Realize the weak solution on the canonical product path space
\[
\Omega := C([0,T];\R^d)\times C([0,T];\R^d),
\]
with canonical coordinates $(Z,W)$ and raw canonical filtration
\[
\mathcal G_t := \sigma(Z_s,W_s: 0\le s\le t),
\qquad 0\le t\le T.
\]
Let $\mathbf P$ denote the joint law of $(z,W)$ on $\Omega$. Since $\Omega$ and
$\R^d$ are Polish, there exists a regular conditional law
$(\mathbf P_x)_{x\in\R^d}$ of $(Z,W)$ given $Z_0=x$.

We first identify a $\mu$-full set of initial states on which the canonical pair
solves the same SDE with deterministic start. By definition of the conditional
law, there exists a Borel set $E_0\subset\R^d$ with $\mu(E_0)=1$ such that for
every $x\in E_0$,
\[
Z_0=x
\qquad \mathbf P_x\text{-a.s.}
\]
Since $W_0=0$ $\mathbf P$-a.s., after shrinking $E_0$ if needed we may also
assume
\[
W_0=0
\qquad \mathbf P_x\text{-a.s.}
\]
for every $x\in E_0$.

Let $A_{\mathrm{SDE}}$ be the measurable set of path pairs on which
\[
Z_t = Z_0 + \int_0^t f(Z_s)\,ds + \Sigma^{1/2}W_t
\]
holds for all rational $t\in[0,T]$. Since this identity holds $\mathbf P$-a.s.,
there exists a Borel set $E_1\subset\R^d$ with $\mu(E_1)=1$ such that
$\mathbf P_x(A_{\mathrm{SDE}})=1$ for all $x\in E_1$. By continuity of $Z$ and
$W$, the integral identity then holds for all $t\in[0,T]$, $\mathbf P_x$-a.s.,
for every $x\in E_1$.

It remains to verify the Brownian property. For each rational $0\le s<t\le T$,
let $\mathcal C_s$ be a countable $\pi$-system generating $\mathcal G_s$
(for example, finite cylinder sets at rational times with rational rectangles),
and let $\mathcal I$ be the countable $\pi$-system of rectangles with rational
endpoints generating $\mathcal B(\R^d)$. Under $\mathbf P$, the canonical
second coordinate $W$ is a Brownian motion with respect to $(\mathcal G_t)$:
this follows because on the original filtered space the Brownian increment is
independent of the underlying filtration, hence in particular independent of the
smaller sigma-field generated by the past of $(z,W)$. Therefore, for every rational $s<t$, every $A\in\mathcal C_s$, every $B\in\mathcal I$, and every
bounded Borel function $\phi:\R^d\to\R$,
\begin{equation*}
\begin{split}
\int_{\R^d} \phi(x)\,\mathbf P_x\!\big(A\cap\{W_t-W_s\in B\}\big)\,\mu(dx)
&= \mathbb E_{\mathbf P}\big[\phi(Z_0)\,\mathbf 1_A\,\mathbf 1_{\{W_t-W_s\in B\}}\big]\\
&= \mathbb E_{\mathbf P}\big[\phi(Z_0)\,\mathbf 1_A\big]\,\gamma_{t-s}(B)\\
&= \int_{\R^d} \phi(x)\,\mathbf P_x(A)\,\gamma_{t-s}(B)\,\mu(dx),    
\end{split}    
\end{equation*}

where $\gamma_{t-s}=N(0,(t-s)I_d)$. Hence, for each fixed rational $s<t$,
$A\in\mathcal C_s$, and $B\in\mathcal I$, there exists a Borel set
$E_{s,t,A,B}\subset\R^d$ with $\mu(E_{s,t,A,B})=1$ such that
\[
\mathbf P_x\big(A\cap\{W_t-W_s\in B\}\big)
=
\mathbf P_x(A)\,\gamma_{t-s}(B)
\qquad \forall x\in E_{s,t,A,B}.
\]
Intersecting over the countable family of such tuples gives a Borel set
$E_2\subset\R^d$ with $\mu(E_2)=1$ such that for every $x\in E_2$ and every
rational $0\le s<t\le T$, the increment $W_t-W_s$ under $\mathbf P_x$ is
independent of $\mathcal G_s$ and has law $\gamma_{t-s}$. By a monotone-class
argument, this extends from $\mathcal C_s$ and $\mathcal I$ to all
$A\in\mathcal G_s$ and all Borel $B\subset\R^d$. Since $W$ has continuous
paths, the Brownian property extends from rational times to all real times.
Hence, for every $x\in E_2$, $W$ is a Brownian motion with respect to the raw
canonical filtration under $\mathbf P_x$, and therefore also with respect to
its $\mathbf P_x$-usual augmentation.

Set $E:=E_0\cap E_1\cap E_2$. Then $\mu(E)=1$, and for every $x\in E$ the
pair $(Z,W)$ is a weak solution of the same SDE with deterministic initial
condition $x$. Applying Step~1 under $\mathbf P_x$ gives
\[
\mathbf P_x(Z_t\in V)>0
\qquad \text{for every }x\in E.
\]
Integrating over the initial law,
\[
\mathbb P(z_t\in V)
=
\mathbf P(Z_t\in V)
=
\int_{\R^d} \mathbf P_x(Z_t\in V)\,\mu(dx)
\ge \int_E \mathbf P_x(Z_t\in V)\,\mu(dx)
>0.
\]
The last inequality holds because the integrand is measurable and strictly
positive on the full-measure set $E$.
This proves the lemma.
\end{proof}

\subsection{Proof of Theorem~\ref{thm:nonlinear_sde_two_regimes} and Corollary~\ref{cor:jacobian_graph_identifiability}}
\label{app:proof_nonlinear}

\begin{proof}[Proof of Theorem~\ref{thm:nonlinear_sde_two_regimes}]
The proof is arranged to reuse the OU argument as much as possible. More
precisely, Steps~0, 1, 3, 4, and 5 of the proof of
Theorem~\ref{thm:ou_two_regimes_monomial}, together with the implication
``constant Jacobian $\Rightarrow$ affine map'' at the start of Step~6 there,
depend only on \textup{(S1)--(S4)} and on having a pointwise version of the
quadratic-variation constraint. They do not use the linear form of the drift.
We therefore record only the places where the nonlinear setting requires a
modified argument.

\medskip
\noindent\textbf{Step 0: Coupling and the along-path covariance constraint.}

By exactly the same coupling construction as in Step~0 of the proof of
Theorem~\ref{thm:ou_two_regimes_monomial}, now applied to
\eqref{eq:true_sde_model} and \eqref{eq:alt_sde_model}, we may realize the two
observed path laws on a common canonical path space and obtain
\[
\tilde z_t^{(e)} = \varphi(z_t^{(e)})
\qquad \text{a.s. for all } t\in[0,T],\ e\in\{1,2\}.
\]
By (S5), the latent paths are weak solutions with a.s. continuous sample paths;
as SDE solutions, they are continuous semimartingales. Hence It\^o's formula
and quadratic variation apply exactly as in the OU proof.
Repeating Step~1 of Theorem~\ref{thm:ou_two_regimes_monomial} word for
word, we obtain for each regime $e$:
\begin{equation}
\label{eq:QV_phi}
d\langle \tilde z^{(e)}\rangle_t
= D\varphi(z_t^{(e)})\Sigma_e D\varphi(z_t^{(e)})^\top\,dt,
\qquad \Sigma_e := G_eG_e^\top,
\end{equation}
while under the alternative model,
\begin{equation}
\label{eq:QV_alt}
d\langle \tilde z^{(e)}\rangle_t = \tilde\Sigma_e\,dt,
\qquad \tilde\Sigma_e := \tilde G_e\tilde G_e^\top.
\end{equation}
Therefore, with the same continuity-in-time upgrade used in the OU proof,
\begin{equation}
\label{eq:matrix_constraint_along_paths}
D\varphi(z_t^{(e)})\Sigma_e D\varphi(z_t^{(e)})^\top = \tilde\Sigma_e
\qquad \text{a.s. for all } t\in[0,T].
\end{equation}

\medskip
\noindent\textbf{Step 1: Full support upgrades the along-path identity to a pointwise identity on $\R^d$.}

Fix $e\in\{1,2\}$ and any $t_\star\in(0,T]$. Define the continuous
matrix-valued function on $\R^d$ by
\[
B_e(z) := D\varphi(z)\Sigma_e D\varphi(z)^\top - \tilde\Sigma_e.
\]
Evaluating \eqref{eq:matrix_constraint_along_paths} at $t=t_\star$ gives
$B_e(z_{t_\star}^{(e)}) = 0$ a.s.

We claim $B_e\equiv 0$ on $\R^d$. Suppose not. Then there exists
$z_0\in\R^d$ with $B_e(z_0)\neq 0$. By continuity, there exist
$\varepsilon>0$ and an open neighborhood $V$ of $z_0$ such that
$\|B_e(z)\|_F>\varepsilon$ for all $z\in V$. By
Lemma~\ref{lem:open_set_positivity}---proved by first establishing positivity
conditional on a deterministic initial state and then integrating over the law
of $z_0^{(e)}$---we have $\mathbb{P}(z_{t_\star}^{(e)}\in V)>0$, contradicting
$B_e(z_{t_\star}^{(e)}) = 0$ a.s. Hence
\begin{equation}
\label{eq:matrix_constraint_pointwise}
D\varphi(z)\,\Sigma_e\, D\varphi(z)^\top = \tilde \Sigma_e
\qquad \forall z\in \R^d,\ \forall e\in\{1,2\}.
\end{equation}

\medskip
\noindent\textbf{Step 2: Reuse the OU monomial-identifiability core.}

Once \eqref{eq:matrix_constraint_pointwise} is available, Steps~3, 4, and 5 of
the proof of Theorem~\ref{thm:ou_two_regimes_monomial} apply verbatim. Define
\begin{equation}
\label{eq:M_def}
M(z):=\tilde\Sigma_1^{-1/2}\, D\varphi(z)\, \Sigma_1^{1/2}.
\end{equation}
Then $M(z)$ is orthogonal for every $z\in\R^d$, and using regime $e=2$ gives
\begin{equation}
\label{eq:diag_conjugacy_eq}
M(z)\,K\, M(z)^\top = \tilde K,
\qquad
K:=\Sigma_1^{-1/2}\Sigma_2\Sigma_1^{-1/2},
\quad
\tilde K:=\tilde\Sigma_1^{-1/2}\tilde\Sigma_2\tilde\Sigma_1^{-1/2}.
\end{equation}
By (S3), $\tilde K$ is diagonal, and by (S4), $K$ has simple spectrum. The same
whitening/simple-spectrum argument as in the OU proof therefore shows that
$M(z)$ is a constant signed permutation matrix. Consequently,
\begin{equation}
\label{eq:Dphi_constant}
D\varphi(z)\equiv L := \tilde\Sigma_1^{1/2} M_0 \Sigma_1^{-1/2}
\qquad \forall z\in\R^d,
\end{equation}
for a constant monomial matrix $L=\Pi\Lambda$. Exactly as at the start of
Step~6 in the OU proof, $D\varphi\equiv L$ implies that
$\varphi(z)=Lz+c$ for some constant $c\in\R^d$. Substituting
$D\varphi\equiv L$ into \eqref{eq:matrix_constraint_pointwise} gives
\[
\tilde\Sigma_e = L\Sigma_e L^\top
\qquad \text{for each } e\in\{1,2\}.
\]

\medskip
\noindent\textbf{Step 3: Identify the drift transformation.}

This step parallels the latter half of Step~6 in the OU proof, with $Az$ and
$\tilde A\tilde z$ replaced by the general drifts $f(z)$ and
$\tilde f(\tilde z)$. Fix a regime $e$. Since $\varphi$ is a global $C^2$
diffeomorphism by (S1), applying It\^o's formula to
$\tilde z_t^{(e)} = \varphi(z_t^{(e)})$ yields
\begin{align*}
d\tilde z_t^{(e)}
&= D\varphi(z_t^{(e)})\,dz_t^{(e)}
 + \frac{1}{2} \sum_{k,\ell=1}^d (\Sigma_e)_{k\ell}
   \partial_{k\ell}\varphi(z_t^{(e)})\,dt \\
&= \Big(D\varphi(z_t^{(e)})f(z_t^{(e)})
 + \tfrac{1}{2}\,\Sigma_e : D^2\varphi(z_t^{(e)})\Big)\,dt
 + D\varphi(z_t^{(e)})G_e\,dW_t^{(e)}.
\end{align*}
Under the alternative model \eqref{eq:alt_sde_model},
\[
d\tilde z_t^{(e)} = \tilde f(\tilde z_t^{(e)})\,dt + \tilde G_e\,d\tilde W_t^{(e)}.
\]

To justify comparing drift terms, we use Lemma~\ref{lem:innovation}:
replacing $G_e\,dW_t^{(e)}$ with $\Sigma_e^{1/2}\,d\hat W_t^{(e)}$ (where
$\hat W^{(e)}$ is a standard $\mathcal{F}_t^{z^{(e)}}$-Brownian motion), the
noise term in the It\^o expansion becomes
$D\varphi(z_t^{(e)})\Sigma_e^{1/2}\,d\hat W_t^{(e)}$, a continuous
$\mathcal{F}_t^{z^{(e)}}$-local martingale. Likewise, the alternative model's
innovation representation yields a noise integral that is an
$\mathcal{F}_t^{\tilde z^{(e)}}$-local martingale. By
Remark~\ref{rem:common_filtration}, these filtrations coincide since
$\tilde z^{(e)}=\varphi(z^{(e)})$ with $\varphi$ a diffeomorphism.
Subtracting the two representations of $d\tilde z^{(e)}$, the difference has
the form $\int_0^t H(s)\,ds + M_t$, where $\int_0^t H(s)\,ds$ is of finite
variation and $M_t$ is a continuous local martingale in the common filtration.
A continuous process that is both of finite variation and a local martingale
must be constant; since it starts at $0$, it is identically $0$. Hence
$H(t)=0$ for Lebesgue-a.e. $t\in[0,T]$, almost surely. Therefore,
\[
\tilde f(\tilde z_t^{(e)})
= D\varphi(z_t^{(e)})f(z_t^{(e)})
 + \tfrac{1}{2}\,\Sigma_e : D^2\varphi(z_t^{(e)})
\quad\text{for a.e. } t\ \text{and a.s.}
\]

Define
\[
H(t):=\tilde f(\tilde z_t^{(e)})-
\Big(D\varphi(z_t^{(e)})f(z_t^{(e)})+\tfrac12\,\Sigma_e: D^2\varphi(z_t^{(e)})\Big).
\]
As above, $H(t)=0$ for a.e. $t$ a.s., and $H$ is a.s. continuous in $t$.
Therefore $H(t)=0$ for all $t\in[0,T]$ a.s. In particular, fix any
$t_\star\in(0,T]$; the identity holds at $t=t_\star$.

Since $D\varphi \equiv L$ and $D^2\varphi \equiv 0$ on $\R^d$ (from
Step~2), evaluating at $t = t_\star$ gives
\[
\tilde f(\varphi(z_{t_\star}^{(e)})) = L f(z_{t_\star}^{(e)}) \quad \text{a.s.}
\]
Since $\varphi(z) = Lz+c$ for all $z\in\R^d$, this becomes
\[
\tilde f(Lz_{t_\star}^{(e)} + c) = L f(z_{t_\star}^{(e)}) \quad \text{a.s.}
\]

We now upgrade this almost-sure equality to a pointwise identity on $\R^d$.
Define the continuous function $F(z) := \tilde f(Lz+c) - Lf(z)$ on $\R^d$
(continuous by (S5)). We have shown $F(z_{t_\star}^{(e)})=0$ a.s. Suppose
$F(z_0)\neq 0$ for some $z_0\in\R^d$. By continuity of $F$, there exist
$\varepsilon>0$ and an open neighborhood $V$ of $z_0$ such that
$\|F(z)\|>\varepsilon$ for all $z\in V$. Applying
Lemma~\ref{lem:open_set_positivity} once more (again via conditioning on the
initial state $z_0^{(e)}$ and integrating over its law) gives
$\mathbb{P}(z_{t_\star}^{(e)}\in V)>0$, contradicting
$F(z_{t_\star}^{(e)})=0$ a.s. Therefore $F\equiv 0$ on $\R^d$, i.e.
\[
\tilde f(Lz+c) = Lf(z)
\qquad \forall z\in\R^d.
\]
\end{proof}

\begin{proof}[Proof of Corollary~\ref{cor:jacobian_graph_identifiability}]
Assume $f$ and $\tilde f$ are $C^1$ on $\R^d$. From
Theorem~\ref{thm:nonlinear_sde_two_regimes} we have, for all $z\in\R^d$,
\[
\tilde f(\tilde z)=L f(z),
\qquad \text{where } \tilde z:=\varphi(z)=Lz+c.
\]
Equivalently, using $z=L^{-1}(\tilde z-c)$, we may write
\[
\tilde f(\tilde z)=L\,f\!\big(L^{-1}(\tilde z-c)\big),
\qquad \forall\,\tilde z\in\R^d.
\]
Differentiating with respect to $\tilde z$ and applying the chain rule yields,
for all $\tilde z\in\R^d$,
\[
D\tilde f(\tilde z)
= L\,Df\!\big(L^{-1}(\tilde z-c)\big)\,L^{-1}.
\]
Equivalently, writing $\tilde z=Lz+c$, we have
\[
D\tilde f(Lz+c)=L\,Df(z)\,L^{-1},
\qquad \forall z\in\R^d.
\]
Write $L=\Pi\Lambda$ with $\Pi$ a permutation matrix and $\Lambda$ an
invertible diagonal matrix. Then
\[
L\,Df(z)\,L^{-1}
= \Pi\big(\Lambda\,Df(z)\,\Lambda^{-1}\big)\Pi^{-1}.
\]
Left/right multiplication by an invertible diagonal matrix preserves the zero
pattern, hence $\Lambda Df(z)\Lambda^{-1}$ has the same support as $Df(z)$.
Therefore the Jacobian edge set is permuted by $\Pi$:
\[
(\partial_j f_i(z)\neq 0)
\quad\Longleftrightarrow\quad
\big(\partial_{\pi(j)}\tilde f_{\pi(i)}(\tilde z)\neq 0\big),
\qquad \tilde z=Lz+c.
\]
Equivalently,
\[
(j\to i)\in E_f(z)
\quad \Longleftrightarrow \quad
(\pi(j)\to \pi(i))\in E_{\tilde f}(Lz+c).
\]
This proves the claimed Jacobian-graph identifiability.
\end{proof}

\subsection{Why weak well-posedness suffices}
\label{app:weak_vs_strong}

\begin{remark}[Why weak well-posedness suffices]
\label{rem:weak_vs_strong}
(S5) requires only \emph{weak} well-posedness of the SDE
(existence of a weak solution that is unique in law), rather than the stronger
notion of a \emph{strong} (pathwise) solution. This distinction is worth
clarifying.

A \textbf{strong solution} is constructed on a given probability space with a
given Brownian motion $W_t$: one seeks a process $z_t$ that is adapted to the
filtration of $W$ and satisfies the integral equation
\[
z_t = z_0 + \int_0^t f(z_s)\,ds + \int_0^t G_e\,dW_s.
\]
Uniqueness is \emph{pathwise}: for each sample path of $W$, the solution
$z_t(\omega)$ is uniquely determined. A standard sufficient condition for
strong existence and pathwise uniqueness is global Lipschitz continuity,
together with a linear-growth condition.

A \textbf{weak solution} allows freedom to choose the probability space and the
driving Brownian motion: one seeks some probability space supporting both a
Brownian motion $W_t$ and a process $z_t$ that together satisfy the SDE.
Uniqueness is \emph{in law}: any two weak solutions induce the same law on path
space $C([0,T];\R^d)$. Weak existence can hold under much milder conditions than
strong existence; see, for example, \cite{stroock2007multidimensional}. In this
paper we directly assume weak well-posedness in (S5) and do not rely on a
specific existence theorem.

The identifiability argument in
Theorem~\ref{thm:nonlinear_sde_two_regimes} depends only on the
\emph{distributional} properties of the observed trajectories:
\begin{itemize}[leftmargin=1.5em,itemsep=2pt]
    \item (S2) compares the laws of $\{x_t^{(e)}\}$ and
    $\{\tilde x_t^{(e)}\}$;
    \item the coupling in Step~0 is a distributional construction;
    \item It\^o's formula and quadratic variation apply to any continuous
    semimartingale, including weak solutions;
    \item the full-support argument in Steps~1 and~3 uses the law of
    $z_{t_\star}^{(e)}$, not its pathwise construction.
\end{itemize}
Requiring strong solutions would unnecessarily restrict the class of admissible
drifts without strengthening the identifiability conclusion.
\end{remark}

\newpage
\section{Estimation algorithm}
\label{app:algorithm}

Algorithm~\ref{alg:two_stage} summarizes the two-stage estimator used in the
experiments. Stage~1 fits the shared-drift, regime-specific diagonal-diffusion
transition model in the learned coordinate system. Stage~2 is optional and is
used only when a sparse drift-Jacobian graph estimate is needed.

\begin{algorithm}[h]
\caption{Two-stage latent SDE estimation with optional sparse drift-Jacobian graph recovery}
\label{alg:two_stage}
\begin{algorithmic}[1]
\REQUIRE Multi-regime trajectories $\{x_{t_0}^{(e)},\ldots,x_{t_N}^{(e)}\}_{e=1}^{2}$;
         time step $\Delta t$; stride $s$; hyperparameter $\lambda_{\mathrm{sparse}}$;
         threshold $\tau$
\ENSURE  Encoder $h_\theta$; learned drift model $\tilde f_\psi$; optional drift-Jacobian graph estimate $\hat E_f$

\medskip
\STATE \textbf{Stage~1: short-time transition fitting in learned coordinates}
\STATE Initialize encoder $h_\theta$, drift model $\tilde f_\psi$, and diffusion
       parameters $\{\log \tilde\sigma_{e,i}^2\}_{e,i}$
\FOR{optimization step $k=1,\ldots,T_1$}
    \FOR{$e=1,2$}
        \STATE Sample mini-batch $(x_t,x_{t+\Delta t})\sim\mathcal{P}_e$
        \STATE Encode $\tilde z_t\leftarrow h_\theta(x_t)$ and
               $\tilde z_{t+\Delta t}\leftarrow h_\theta(x_{t+\Delta t})$
        \STATE Compute the per-regime transition score from \eqref{eq:nll_obs}
    \ENDFOR
    \STATE Update $(\theta,\psi,\{\log \tilde\sigma_{e,i}^2\})$
           using $\nabla \mathcal{L}_{\mathrm{S1}}$
\ENDFOR

\medskip
\STATE \textbf{Optional Stage~2: sparse drift-Jacobian graph recovery}
\STATE Freeze $h_\theta$ and $\{\tilde\sigma_{e,i}^2\}$; reinitialize $\tilde f_\psi$
\STATE Compute velocity targets
       $\hat v_t \leftarrow \big(h_\theta(x_{t+s\Delta t})-h_\theta(x_t)\big)/(s\Delta t)$
\FOR{optimization step $k=1,\ldots,T_2$}
    \STATE Sample mini-batches $(\tilde z_t,\hat v_t)$ from both regimes
    \STATE Update $\psi$ using $\nabla \mathcal{L}_{\mathrm{S2}}$
\ENDFOR

\medskip
\STATE Estimate
       $\overline{|D\tilde f_\psi|}_{ij}
       \leftarrow
       \E_{\tilde z\sim\mathcal{Z}}
       \big[|\partial \tilde f_{\psi,i}/\partial \tilde z_j(\tilde z)|\big]$
\STATE Return the thresholded drift-Jacobian edge set
       $\hat E_f = \{(j\to i): \overline{|D\tilde f_\psi|}_{ij}>\tau\}$ if graph
       recovery is desired
\end{algorithmic}
\end{algorithm}

\newpage

\section{Simulation setup}
\label{app:sim_setup}

This section specifies the data-generating process, observation model, training procedure, evaluation metrics, and hyperparameters for all simulation experiments.

\subsection{Latent process}
\label{app:latent_process}

Each experiment generates $N_{\mathrm{traj}}$ independent trajectories from a $d$-dimensional latent stochastic differential equation
\[
  \mathrm{d}z_t^{(e)} = f(z_t^{(e)})\,\mathrm{d}t + \Sigma_e^{1/2}\,\mathrm{d}W_t^{(e)},
\]
where $f\colon \mathbb{R}^d \to \mathbb{R}^d$ is the latent drift, $\Sigma_e = \operatorname{diag}(\sigma_{e,1}^2,\dots,\sigma_{e,d}^2)$ is the environment-specific diagonal diffusion covariance, and $W_t^{(e)}$ is a standard $d$-dimensional Brownian motion. Trajectories are discretized via the Euler--Maruyama scheme with step size $\Delta t$. In the implementation, $T$ is a nominal horizon parameter and $n_{\mathrm{steps}}=\lfloor T/\Delta t\rfloor=20$ points are stored including the initial condition. Thus the last stored time is $(n_{\mathrm{steps}}-1)\Delta t$ and each trajectory yields $19$ adjacent one-step transition pairs. Initial conditions are drawn i.i.d.\ from $\mathcal{N}(0,\, 0.25\,I_d)$.

We consider three drift families: dense linear OU, sparse linear OU, and nonlinear sparse drift, whose constructions are detailed in Sections~\ref{app:dense_linear} to~\ref{app:nonlinear}. Within each drift family and dimension, the drift $f$ is fixed and shared across all regime conditions and all training seeds.

\subsection{Regime conditions and diffusion covariances}
\label{app:regimes}

Each drift family is evaluated under three regime conditions. These conditions vary only the number and choice of diagonal diffusion covariances; all other components are held fixed. In all conditions, the first environment uses
\[
  \sigma_{1,j}^2 = 0.5\,j, \qquad j = 1, \dots, d,
\]
i.e.\ $\Sigma_1 = \operatorname{diag}(0.5,\,1.0,\,1.5,\dots,0.5d)$.

\paragraph{Distinct ratios (identifiable).}
Two environments ($e \in \{1,2\}$) with $\Sigma_2$ chosen so that the variance ratios $\sigma_{2,j}^2 / \sigma_{1,j}^2$ are pairwise distinct across all coordinates $j = 1, \dots, d$. The diagonal entries of $\Sigma_2$ are
\begin{align*}
  d=5\colon \quad \sigma_{2,j}^2 &= (2.0,\; 1.5,\; 0.8,\; 0.5,\; 0.4), \\
  d=7\colon \quad \sigma_{2,j}^2 &= (2.5,\; 0.3,\; 2.1,\; 0.35,\; 1.8,\; 0.43,\; 1.6).
\end{align*}

\paragraph{One regime (control).}
A single environment with diffusion covariance $\Sigma_1$ only. Since there is only one environment, the distinct-ratios condition is not applicable.

\paragraph{Proportional diffusions (control).}
Two environments with $\sigma_{2,j}^2 = 2\,\sigma_{1,j}^2$ for all $j$, so that all variance ratios $\sigma_{2,j}^2/\sigma_{1,j}^2$ equal $2$. This violates the two-regime separation requirement in Assumption~(S4), since all ratios are identical.

\noindent For $d\in\{5,7\}$, the same diffusion-covariance choices are used across all three drift families. The separate $d=10$ setup is specified in Section~\ref{app:d10_results}.

\subsection{Observation model}
\label{app:obs_model}

The learner observes $x_t^{(e)} = g(z_t^{(e)})$. We implement $g\colon \mathbb{R}^d \to \mathbb{R}^d$ as an invertible MLP mixing map with orthogonal linear layers and componentwise activations. We test two choices of activation within the same $m$-layer orthogonal MLP architecture:
\[
  y_1=\phi(R_1z+b_1), \qquad
  y_k=\phi(R_k y_{k-1}+b_k)\;\;(2\le k\le m-1), \qquad
  g(z)=R_m y_{m-1},
\]
where each $R_k \in \mathrm{O}(d)$ is a random orthogonal matrix (drawn via the Haar measure), $b_k \sim \mathcal{N}(0, I_d)$ for $k=1,\ldots,m-1$ are bias vectors, and $\phi$ is a componentwise nonlinearity. The final layer is a pure rotation with no bias or nonlinearity. In all experiments $m = 3$.

\paragraph{Leaky-$\tanh$ (main text).}
$\phi(z) = \tanh(z) + 0.1\,z$. Since $\phi'(z) = \operatorname{sech}^2(z) + 0.1 > 0$ everywhere and $\phi \in C^\infty$, the composition $g$ is a $C^\infty$ diffeomorphism, consistent with the smoothness assumptions of our identifiability results.

\paragraph{LeakyReLU (appendix).}
$\phi(z) = \max(z,\, 0.2\,z)$. This map is invertible and piecewise linear, but it is not $C^1$ at zero and hence not $C^2$, providing a robustness check beyond the smoothness assumptions.

\noindent The orthogonal matrices and biases are drawn with a fixed seed, so the same $g$ is used across all regime conditions and training seeds within each fixed drift family, dimension, and activation choice.

\subsection{Training procedure}
\label{app:training}

We follow the two-stage procedure described in Section~\ref{sec:estimation} and summarized in Algorithm~\ref{alg:two_stage}.

\paragraph{Stage~1: Short-time transition fitting.}
The encoder $h_\theta\colon \mathbb{R}^d \to \mathbb{R}^d$ is a fully connected MLP with hidden dimension~256, four hidden layers, and ELU activations. We minimize the Stage~1 objective $\mathcal{L}_{\mathrm{S1}}$ (Eq.~\ref{eq:stage1_loss}), including the change-of-variables log-determinant term from Eq.~\ref{eq:nll_obs}:
\[
  -\log\big|\det Dh_\theta(x_{t+\Delta t})\big|.
\]
The log-determinant term is the local volume correction in the observation-space pseudo-likelihood and penalizes collapse of the encoder Jacobian. All parameters $(\theta, \psi, \{\tilde\sigma_{e,i}^2\})$ are optimized jointly with Adam (learning rate $10^{-3}$, batch size $4{,}096$) for $20{,}000$ stochastic optimization steps, each using a fresh minibatch.
\paragraph{Stage~2: Sparse drift-Jacobian graph recovery (sparse linear and nonlinear only).}
After Stage~1, the encoder is frozen and we minimize $\mathcal{L}_{\mathrm{S2}}$ (Eq.~\ref{eq:stage2_loss}) to fit a drift model to the encoded velocity targets. For the sparse linear family, the drift model is a linear map $\tilde{f}_\psi(\tilde{z}) = \hat{A}\tilde{z}$, with sparsity penalty $\sum_{i,j}|\hat{A}_{ij}|$. For the nonlinear family, $\tilde{f}_\psi$ is a three-layer MLP with ELU activations, with the Jacobian $L_1$ penalty in Eq.~\ref{eq:stage2_loss} evaluated on encoded samples. We optimize with Adam (learning rate $5 \times 10^{-4}$). Dense linear experiments skip Stage~2 because the drift is dense by design, making sparse drift-Jacobian graph recovery inapplicable.
\paragraph{Computational resources.}
The reported experiments were run as standard Slurm jobs requesting one NVIDIA L40S GPU and 32 GB memory per job.  For a fixed setting, the five training seeds used in the synthetic tables typically completed in a few minutes.  The implementation does not rely on specialized hardware and can also run on CPU or on other CUDA GPUs, with CPU runs expected to be slower.

\paragraph{Graph thresholding.}
Following Algorithm~\ref{alg:two_stage}, the learned drift-Jacobian graph is obtained by thresholding the mean absolute Jacobian $\overline{|D\tilde{f}_\psi|}$ at a fixed threshold $\tau = 0.05$ in the encoded space, before any scaling correction. The resulting binary adjacency is then permuted back to the original coordinate ordering for comparison with the ground truth.

\paragraph{Evaluation.}
Each experiment is repeated over five independent training seeds. We report mean $\pm$ standard deviation of three metrics defined below. For diagnostic plots (scatter matrices, Jacobian heatmaps, graph comparisons), we use one representative seed, held fixed across the regime conditions shown in each figure.

\subparagraph{Mean correlation coefficient (MCC).}
Given $n$ samples of the true latent coordinates $z \in \R^{n \times d}$ and the learned coordinates $\tilde{z} \in \R^{n \times d}$, we compute the $d \times d$ absolute Pearson correlation matrix $C$ with entries $C_{ij} = \bigl|\operatorname{corr}(z_{\cdot,i},\, \tilde{z}_{\cdot,j})\bigr|$. The optimal one-to-one assignment $\pi^*$ is found by the Hungarian algorithm:
\[
  \pi^* = \arg\max_{\pi \in S_d} \sum_{i=1}^{d} C_{i,\pi(i)},
  \qquad
  \mathrm{MCC} = \frac{1}{d}\sum_{i=1}^{d} C_{i,\pi^*(i)}.
\]
A value near~1 indicates that each learned coordinate aligns with exactly one true latent.

\subparagraph{Monomial score (Mon.).}
Let $\bar{J} = \overline{|D(h_\theta \circ g)(z)|}$ denote the mean absolute Jacobian of the composition of the encoder and observation map, a $d \times d$ non-negative matrix. We normalize $\bar{J}$ along rows and columns separately:
\[
  R_{ij} = \frac{\bar{J}_{ij}}{\sum_{k} \bar{J}_{ik}},
  \qquad
  Q_{ij} = \frac{\bar{J}_{ij}}{\sum_{k} \bar{J}_{kj}},
\]
and define the row and column concentration as the mean of the row-wise and column-wise maxima:
\[
  \rho_{\mathrm{row}} = \frac{1}{d}\sum_{i=1}^{d} \max_{j}\, R_{ij},
  \qquad
  \rho_{\mathrm{col}} = \frac{1}{d}\sum_{j=1}^{d} \max_{i}\, Q_{ij}.
\]
The monomial score is $\mathrm{Mon.} = \tfrac{1}{2}(\rho_{\mathrm{row}} + \rho_{\mathrm{col}})$. For a perfect monomial matrix every row and column has exactly one nonzero entry, giving $\mathrm{Mon.} = 1$; a uniform matrix yields $\mathrm{Mon.} = 1/d$.

\subparagraph{Graph match rate (GMR).}
The mean absolute Jacobian of the learned drift, $\overline{|D\tilde{f}_\psi|}$, is binarized at threshold $\tau$:
$\hat{G}^{\mathrm{enc}}_{ij} = \mathbf{1}\!\bigl[\,\overline{|D\tilde{f}_\psi|}_{ij} > \tau\,\bigr]$.
Let $p$ denote the encoder permutation extracted from the encoder Jacobian, where $p(k)$ is the true coordinate matched to encoded coordinate $k$. This learned drift-Jacobian adjacency is then permuted back to the original coordinate ordering:
$\hat{G}_{ij} = \hat{G}^{\mathrm{enc}}_{p^{-1}(i),\,p^{-1}(j)}$.
In practice, $p$ is obtained by applying the Hungarian algorithm to the mean absolute encoder Jacobian $\bar{J}$.
The ground-truth adjacency $G^*$ includes self-loops ($G^*_{ii} = 1$ for all $i$) because every coordinate has a self-regulating drift term. The graph match rate is the fraction of all $d^2$ entries that agree:
\[
  \mathrm{GMR} = \frac{1}{d^2}\sum_{i,j=1}^{d} \mathbf{1}\!\bigl[\hat{G}_{ij} = G^*_{ij}\bigr].
\]

Table~\ref{tab:hyperparams} summarizes the key simulation and training hyperparameters. Values that differ across drift families are noted explicitly.

\begin{table}[ht!]
  \caption{Simulation and training hyperparameters. All values are shared across the three drift families unless noted otherwise.}
  \label{tab:hyperparams}
  \centering
  \small
  \setlength{\tabcolsep}{5pt}
  \begin{tabular}{lcc}
    \toprule
    & $d=5$ & $d=7$ \\
    \midrule
    \multicolumn{3}{l}{\textit{Data generation}} \\
    Step size $\Delta t$              & 0.005       & 0.0025      \\
    Nominal horizon parameter $T$     & 0.1         & 0.05        \\
    Number of trajectories per regime $N_{\mathrm{traj}}$ & 10{,}000 & 20{,}000    \\
    \midrule
    \multicolumn{3}{l}{\textit{Stage 1 (encoder)}} \\
    Hidden dim / layers               & 256 / 4     & 256 / 4     \\
    Learning rate                     & $10^{-3}$   & $10^{-3}$   \\
    Optimization steps                & 20{,}000    & 20{,}000    \\
    Batch size                        & 4{,}096     & 4{,}096     \\
    \midrule
    \multicolumn{3}{l}{\textit{Stage 2 (drift, sparse linear \& nonlinear only)}} \\
    Drift MLP dim / layers$^\dagger$  & 128 / 3     & 128 / 3     \\
    Learning rate                     & $5\!\times\!10^{-4}$ & $5\!\times\!10^{-4}$ \\
    Optimization steps (sparse linear) & 10{,}000    & 10{,}000    \\
    Optimization steps (nonlinear)    & 20{,}000    & 20{,}000    \\
    Velocity stride $s$                & 1           & 1           \\
    Jacobian batch size (nonlinear)    & 256         & 512         \\
    $\lambda_{\mathrm{sparse}}$ (sparse linear) & 0.3 & 0.25    \\
    $\lambda_{\mathrm{sparse}}$ (nonlinear)$^\ddagger$     & 0.5 & 0.4      \\
    \bottomrule
  \end{tabular}
  \\[2pt]
  {\footnotesize $^\dagger$Nonlinear experiments only; sparse linear uses a linear drift model $\tilde{f}_\psi(\tilde{z}) = \hat{A}\tilde{z}$.}\\
  {\footnotesize $^\ddagger$Leaky-$\tanh$ mixing; the LeakyReLU experiments use $\lambda_{\mathrm{sparse}} = 0.9$ for both $d=5$ and $d=7$.}
\end{table}

\clearpage

\section{Additional simulation results}
\label{app:simulations}

Table~\ref{tab:simulation_lrelu} reports results under the same three drift families and regime conditions as in the main text (Table~\ref{tab:simulation_main}), but with a three-layer MLP using LeakyReLU activations as the mixing function. The same identifiability gap persists: Distinct ratios yields near-perfect disentanglement across all drift families, whereas both controls substantially reduce MCC and monomial score.

\begin{table*}[ht!]
  \caption{Simulation results on $d=5$ under a three-layer MLP with LeakyReLU activations. All entries are mean $\pm$ standard deviation over five seeds. \textbf{Bold} highlights the Distinct ratios condition; the other two conditions are controls. GMR is omitted for the dense linear family because graph sparsity is not the target there.}
  \label{tab:simulation_lrelu}
  \centering
  \small
  \setlength{\tabcolsep}{6pt}
  \begin{tabular}{llccc}
    \toprule
    Drift family & Regime condition & MCC $\uparrow$ & Mon.\ $\uparrow$ & GMR $\uparrow$ \\
    \midrule
    Dense linear & Distinct ratios & $\mathbf{0.9977 \pm 0.0002}$ & $\mathbf{0.8429 \pm 0.0068}$ & -- \\
    Dense linear & One regime & $0.7363 \pm 0.0516$ & $0.4178 \pm 0.0304$ & -- \\
    Dense linear & Proportional diffusions & $0.6978 \pm 0.0234$ & $0.3877 \pm 0.0113$ & -- \\
    \midrule
    Sparse linear & Distinct ratios & $\mathbf{0.9979 \pm 0.0003}$ & $\mathbf{0.8667 \pm 0.0088}$ & $\mathbf{1.0000 \pm 0.0000}$ \\
    Sparse linear & One regime & $0.7364 \pm 0.0491$ & $0.4129 \pm 0.0273$ & $0.5840 \pm 0.0543$ \\
    Sparse linear & Proportional diffusions & $0.7153 \pm 0.0600$ & $0.4081 \pm 0.0375$ & $0.5680 \pm 0.0588$ \\
    \midrule
    Nonlinear & Distinct ratios & $\mathbf{0.9979 \pm 0.0002}$ & $\mathbf{0.8660 \pm 0.0068}$ & $\mathbf{1.0000 \pm 0.0000}$ \\
    Nonlinear & One regime & $0.7387 \pm 0.0647$ & $0.4269 \pm 0.0332$ & $0.6160 \pm 0.0650$ \\
    Nonlinear & Proportional diffusions & $0.7336 \pm 0.0495$ & $0.4142 \pm 0.0328$ & $0.7200 \pm 0.1043$ \\
    \bottomrule
  \end{tabular}
\end{table*}

Tables~\ref{tab:simulation_ltanh_d7} and~\ref{tab:simulation_lrelu_d7} extend the analysis to $d=7$ under leaky-$\tanh$ and LeakyReLU mixing, respectively. The same qualitative pattern holds: Distinct ratios yields near-perfect disentanglement across all three drift families, whereas both controls substantially reduce coordinate recovery.

\begin{table*}[ht!]
  \caption{Simulation results on $d=7$ under a three-layer MLP with leaky-$\tanh$ activations. All entries are mean $\pm$ standard deviation over five seeds. \textbf{Bold} highlights the Distinct ratios condition; the other two conditions are controls. GMR is omitted for the dense linear family because graph sparsity is not the target there.}
  \label{tab:simulation_ltanh_d7}
  \centering
  \small
  \setlength{\tabcolsep}{6pt}
  \begin{tabular}{llccc}
    \toprule
    Drift family & Regime condition & MCC $\uparrow$ & Mon.\ $\uparrow$ & GMR $\uparrow$ \\
    \midrule
    Dense linear & Distinct ratios & $\mathbf{0.9949 \pm 0.0041}$ & $\mathbf{0.8283 \pm 0.0196}$ & -- \\
    Dense linear & One regime & $0.6745 \pm 0.0555$ & $0.3520 \pm 0.0406$ & -- \\
    Dense linear & Proportional diffusions & $0.6319 \pm 0.0276$ & $0.3252 \pm 0.0225$ & -- \\
    \midrule
    Sparse linear & Distinct ratios & $\mathbf{0.9974 \pm 0.0005}$ & $\mathbf{0.8616 \pm 0.0090}$ & $\mathbf{0.9959 \pm 0.0082}$ \\
    Sparse linear & One regime & $0.6441 \pm 0.0535$ & $0.3424 \pm 0.0279$ & $0.6327 \pm 0.0428$ \\
    Sparse linear & Proportional diffusions & $0.6350 \pm 0.0575$ & $0.3270 \pm 0.0351$ & $0.6776 \pm 0.0523$ \\
    \midrule
    Nonlinear & Distinct ratios & $\mathbf{0.9972 \pm 0.0002}$ & $\mathbf{0.8564 \pm 0.0053}$ & $\mathbf{0.9755 \pm 0.0153}$ \\
    Nonlinear & One regime & $0.6256 \pm 0.0468$ & $0.3264 \pm 0.0239$ & $0.7061 \pm 0.0277$ \\
    Nonlinear & Proportional diffusions & $0.6387 \pm 0.0484$ & $0.3230 \pm 0.0311$ & $0.8204 \pm 0.0327$ \\
    \bottomrule
  \end{tabular}
\end{table*}

\begin{table*}[ht!]
  \caption{Simulation results on $d=7$ under a three-layer MLP with LeakyReLU activations. All entries are mean $\pm$ standard deviation over five seeds. \textbf{Bold} highlights the Distinct ratios condition; the other two conditions are controls. GMR is omitted for the dense linear family because graph sparsity is not the target there.}
  \label{tab:simulation_lrelu_d7}
  \centering
  \small
  \setlength{\tabcolsep}{6pt}
  \begin{tabular}{llccc}
    \toprule
    Drift family & Regime condition & MCC $\uparrow$ & Mon.\ $\uparrow$ & GMR $\uparrow$ \\
    \midrule
    Dense linear & Distinct ratios & $\mathbf{0.9948 \pm 0.0006}$ & $\mathbf{0.7815 \pm 0.0039}$ & -- \\
    Dense linear & One regime & $0.6750 \pm 0.0519$ & $0.3364 \pm 0.0259$ & -- \\
    Dense linear & Proportional diffusions & $0.6452 \pm 0.0550$ & $0.3123 \pm 0.0324$ & -- \\
    \midrule
    Sparse linear & Distinct ratios & $\mathbf{0.9960 \pm 0.0005}$ & $\mathbf{0.8023 \pm 0.0114}$ & $\mathbf{1.0000 \pm 0.0000}$ \\
    Sparse linear & One regime & $0.6542 \pm 0.0410$ & $0.3327 \pm 0.0238$ & $0.5714 \pm 0.0671$ \\
    Sparse linear & Proportional diffusions & $0.6319 \pm 0.0251$ & $0.3213 \pm 0.0180$ & $0.5837 \pm 0.0909$ \\
    \midrule
    Nonlinear & Distinct ratios & $\mathbf{0.9961 \pm 0.0007}$ & $\mathbf{0.8095 \pm 0.0078}$ & $\mathbf{0.9388 \pm 0.0316}$ \\
    Nonlinear & One regime & $0.6343 \pm 0.0367$ & $0.3276 \pm 0.0168$ & $0.5429 \pm 0.0988$ \\
    Nonlinear & Proportional diffusions & $0.5969 \pm 0.0278$ & $0.3044 \pm 0.0136$ & $0.6612 \pm 0.1625$ \\
    \bottomrule
  \end{tabular}
\end{table*}

\subsection{Dense linear drift}
\label{app:dense_linear}

The dense linear experiments use an Ornstein--Uhlenbeck process $\mathrm{d}z = Az\,\mathrm{d}t + G\,\mathrm{d}W$ with a symmetric negative-definite drift matrix $A$. This family evaluates coordinate recovery without using graph sparsity as an evaluation target. We construct $A$ as
\[
  A = -(MM^\top + \lambda I), \qquad M_{ij} \stackrel{\mathrm{iid}}{\sim} \mathcal{N}(0,1),
\]
with $\lambda = 0.5$ to ensure a stability margin. The random matrix $M$ is drawn with a fixed seed, so the same $A$ is shared across all training seeds and all three regime conditions. For $d=5$ the resulting drift matrix, rounded to two decimals, is
\[
  A =
  \begin{pmatrix}
    -4.16  &   0.00 & -2.79 &  1.30 & -4.78 \\
     0.00  &  -1.74 & -1.56 &  1.26 &  1.57 \\
    -2.79  &  -1.56 & -5.71 &  2.56 & -0.36 \\
     1.30  &   1.26 &  2.56 & -3.51 &  0.15 \\
    -4.78  &   1.57 & -0.36 &  0.15 & -10.27
  \end{pmatrix}.
\]

Figure~\ref{fig:dense_scatter} shows the pairwise scatter plots of true latent coordinates $z_j$ versus learned representations $\tilde{z}_i$ for each regime condition. Under Distinct ratios, each learned coordinate aligns tightly with exactly one true coordinate, producing a near-permutation pattern. Both controls lose this structure.

Figure~\ref{fig:dense_diagnostics} displays the mean absolute Jacobian $\overline{|D\varphi|}$ of the composition $\varphi = h \circ g$. Under Distinct ratios the matrix is near-monomial, confirming that the encoder inverts the mixing up to coordinate permutation and scaling. Both controls yield less monomial, more mixed encoder Jacobians. Graph recovery is not shown because the drift is dense, so sparse graph recovery is not the evaluation target.

\begin{figure*}[ht!]
  \centering
  \includegraphics[width=\textwidth]{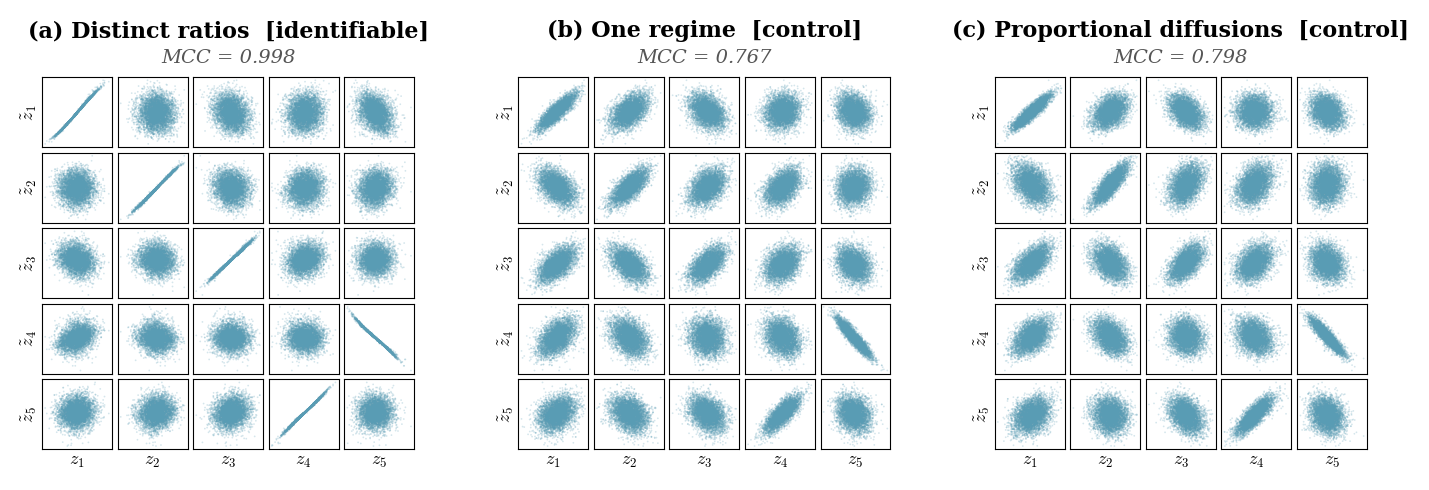}
  \caption{Latent recovery in the dense linear $d=5$ setting (three-layer MLP, leaky-$\tanh$ mixing). Each column shows the $5 \times 5$ scatter matrix of true latent coordinates $z_j$ (horizontal) versus learned representations $\tilde{z}_i$ (vertical) for one regime condition. Under Distinct ratios, each $\tilde{z}_i$ aligns with exactly one $z_j$. Both controls fail to disentangle.}
  \label{fig:dense_scatter}
\end{figure*}

\begin{figure*}[ht!]
  \centering
  \includegraphics[width=\textwidth]{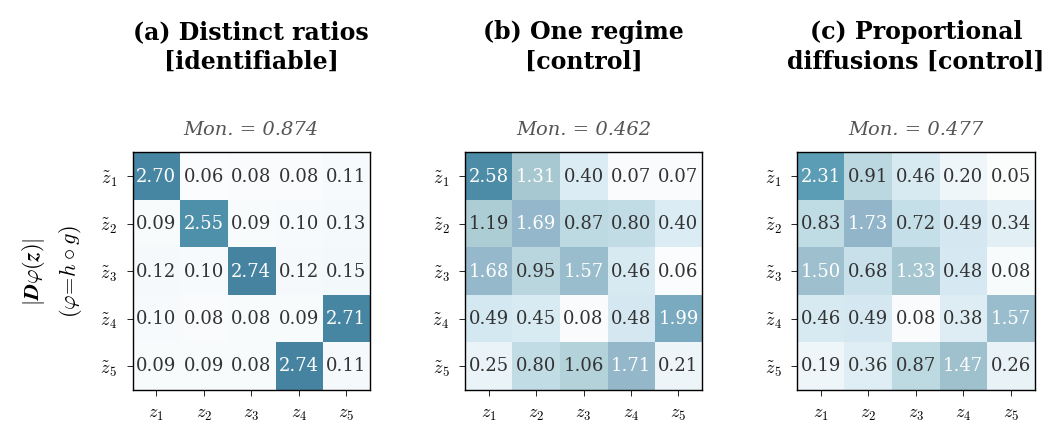}
  \caption{Mean absolute encoder Jacobian $\overline{|D\varphi|}$ ($\varphi = h \circ g$) for the dense linear $d=5$ setting. All three regime conditions use the same representative seed. A near-monomial matrix confirms successful inversion of the mixing under Distinct ratios. Both controls yield less monomial, more mixed encoder Jacobians.}
  \label{fig:dense_diagnostics}
\end{figure*}

The $d=7$ case uses the same construction with identical $\lambda = 0.5$. The resulting drift matrix, rounded to two decimals, is
\[
  A =
  \begin{pmatrix}
    -4.89 &  1.11 & -2.12 & -0.13 &  0.98 & -2.94 &  4.09 \\
     1.11 & -5.21 & -2.30 & -0.70 &  4.00 & -1.50 &  0.93 \\
    -2.12 & -2.30 & -9.39 & -0.37 & -0.35 & -5.45 &  4.37 \\
    -0.13 & -0.70 & -0.37 & -8.22 &  4.36 & -3.82 & -1.67 \\
     0.98 &  4.00 & -0.35 &  4.36 & -8.66 &  3.26 & -1.52 \\
    -2.94 & -1.50 & -5.45 & -3.82 &  3.26 & -7.56 &  4.88 \\
     4.09 &  0.93 &  4.37 & -1.67 & -1.52 &  4.88 & -8.83
  \end{pmatrix}.
\]

Figures~\ref{fig:dense_scatter_d7} and~\ref{fig:dense_diagnostics_d7} show the $d=7$ results. The same qualitative pattern holds: Distinct ratios achieves near-perfect disentanglement, while both controls show substantially less coordinate recovery.

\begin{figure*}[ht!]
  \centering
  \includegraphics[width=\textwidth]{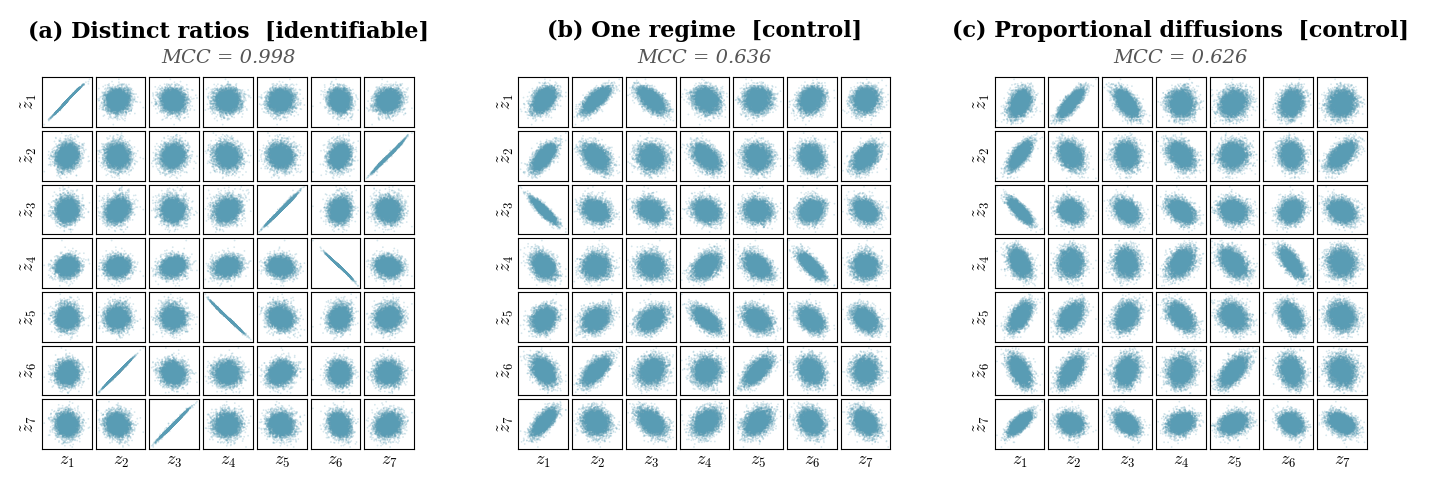}
  \caption{Latent recovery in the dense linear $d=7$ setting (three-layer MLP, leaky-$\tanh$ mixing). Layout follows Figure~\ref{fig:dense_scatter}. Under Distinct ratios, each $\tilde{z}_i$ aligns with exactly one $z_j$. Both controls fail to disentangle.}
  \label{fig:dense_scatter_d7}
\end{figure*}

\begin{figure*}[ht!]
  \centering
  \includegraphics[width=\textwidth]{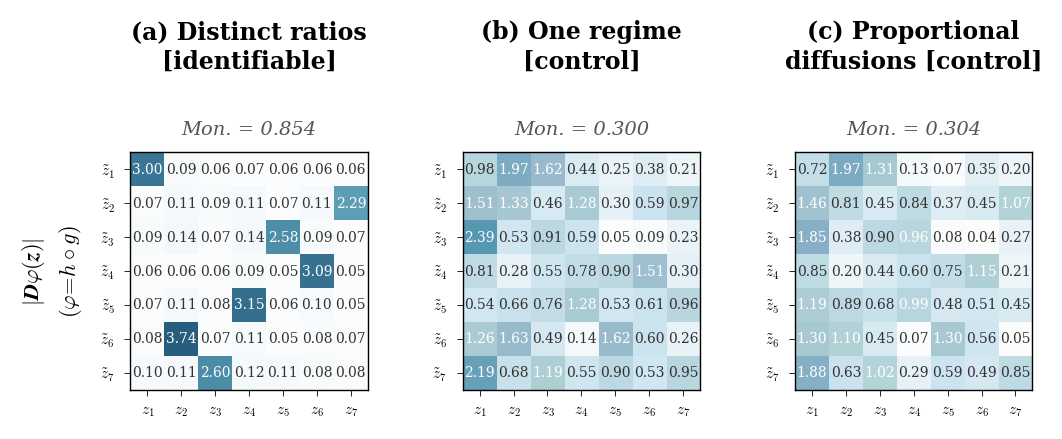}
  \caption{Mean absolute encoder Jacobian $\overline{|D\varphi|}$ ($\varphi = h \circ g$) for the dense linear $d=7$ setting. All three regime conditions use the same representative seed. A near-monomial matrix confirms successful inversion under Distinct ratios.}
  \label{fig:dense_diagnostics_d7}
\end{figure*}

\clearpage

\subsection{Sparse linear drift}
\label{app:sparse_linear}

The sparse linear experiments use an Ornstein--Uhlenbeck process with a sparse drift matrix $A$ built from a directed cycle support. The nonzero off-diagonal support entries are placed at $(i,i+1)$ for $i=1,\ldots,d-1$ and at $(d,1)$, and the drift matrix is constructed as
\[
  A_{ii} = -\alpha, \qquad
  A_{ij} = \beta \, c_{ij} \;\text{if } \mathrm{adj}(i,j)=1, \qquad
  A_{ij} = 0 \;\text{otherwise},
\]
where $\alpha = 1.0$, $\beta = 1.5$, and $c_{ij} \stackrel{\mathrm{iid}}{\sim} \mathrm{Uniform}(0.8, 1.2)$ with a fixed seed. The same $A$ is shared across all training seeds and all three regime conditions.

For $d=5$ the drift matrix, rounded to two decimals, is
\[
  A =
  \begin{pmatrix}
    -1.00 &  1.77 &  0    &  0    &  0    \\
     0    & -1.00 &  1.72 &  0    &  0    \\
     0    &  0    & -1.00 &  1.33 &  0    \\
     0    &  0    &  0    & -1.00 &  1.37 \\
     1.57 &  0    &  0    &  0    & -1.00
  \end{pmatrix}.
\]

Figure~\ref{fig:sparse_scatter} shows the pairwise scatter plots for the sparse linear $d=5$ setting. Figure~\ref{fig:sparse_diagnostics} displays the encoder Jacobian and graph recovery. Under Distinct ratios, the encoder inverts the mixing up to permutation and scaling, and the thresholded drift-Jacobian graph is recovered exactly (GMR $=1$ for the representative seed).

\begin{figure*}[ht!]
  \centering
  \includegraphics[width=\textwidth]{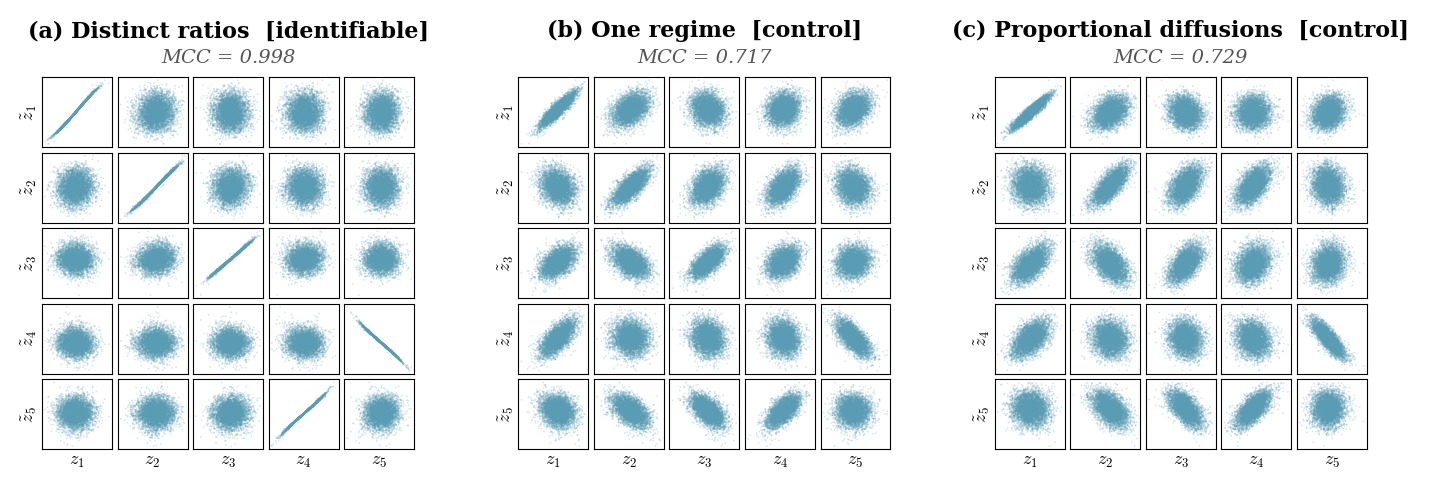}
  \caption{Latent recovery in the sparse linear $d=5$ setting (three-layer MLP, leaky-$\tanh$ mixing). Each column shows the $5 \times 5$ scatter matrix of true latent coordinates $z_j$ (horizontal) versus learned representations $\tilde{z}_i$ (vertical) for one regime condition.}
  \label{fig:sparse_scatter}
\end{figure*}

\begin{figure*}[ht!]
  \centering
  \includegraphics[width=\textwidth]{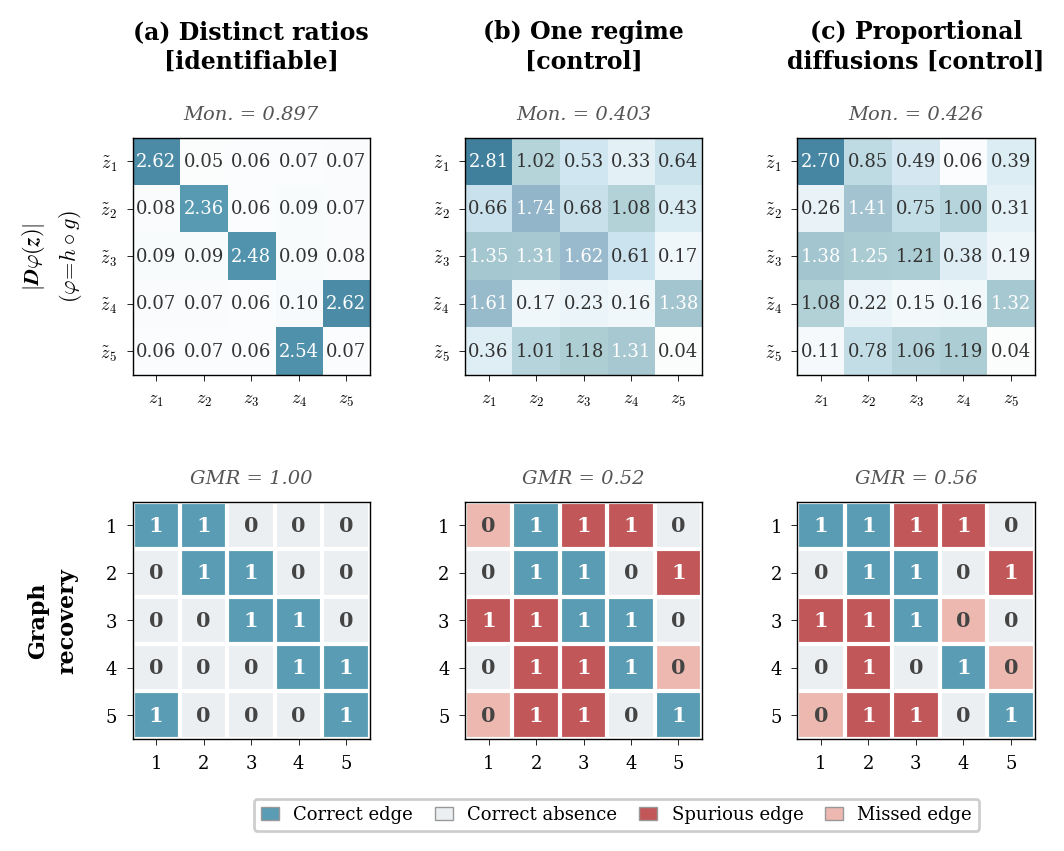}
  \caption{Structural diagnostics for the sparse linear $d=5$ setting. All three regime conditions use the same representative seed. \textbf{Top row}: mean absolute Jacobian $\overline{|D\varphi|}$. \textbf{Bottom row}: learned causal graph versus the ground-truth graph.}
  \label{fig:sparse_diagnostics}
\end{figure*}

The $d=7$ case uses the same construction. The drift matrix, rounded to two decimals, is
\[
  A =
  \begin{pmatrix}
    -1.00 &  1.77 &  0    &  0    &  0    &  0    &  0    \\
     0    & -1.00 &  1.62 &  0    &  0    &  0    &  0    \\
     0    &  0    & -1.00 &  1.51 &  0    &  0    &  0    \\
     0    &  0    &  0    & -1.00 &  1.67 &  0    &  0    \\
     0    &  0    &  0    &  0    & -1.00 &  1.77 &  0    \\
     0    &  0    &  0    &  0    &  0    & -1.00 &  1.50 \\
     1.22 &  0    &  0    &  0    &  0    &  0    & -1.00
  \end{pmatrix}.
\]

Figures~\ref{fig:sparse_scatter_d7} and~\ref{fig:sparse_diagnostics_d7} show the corresponding results for $d=7$. The same qualitative pattern holds: Distinct ratios gives near-monomial encoder Jacobians and near-perfect graph recovery, whereas the controls show mixed encoder Jacobians and lower GMR.

\begin{figure*}[ht!]
  \centering
  \includegraphics[width=\textwidth]{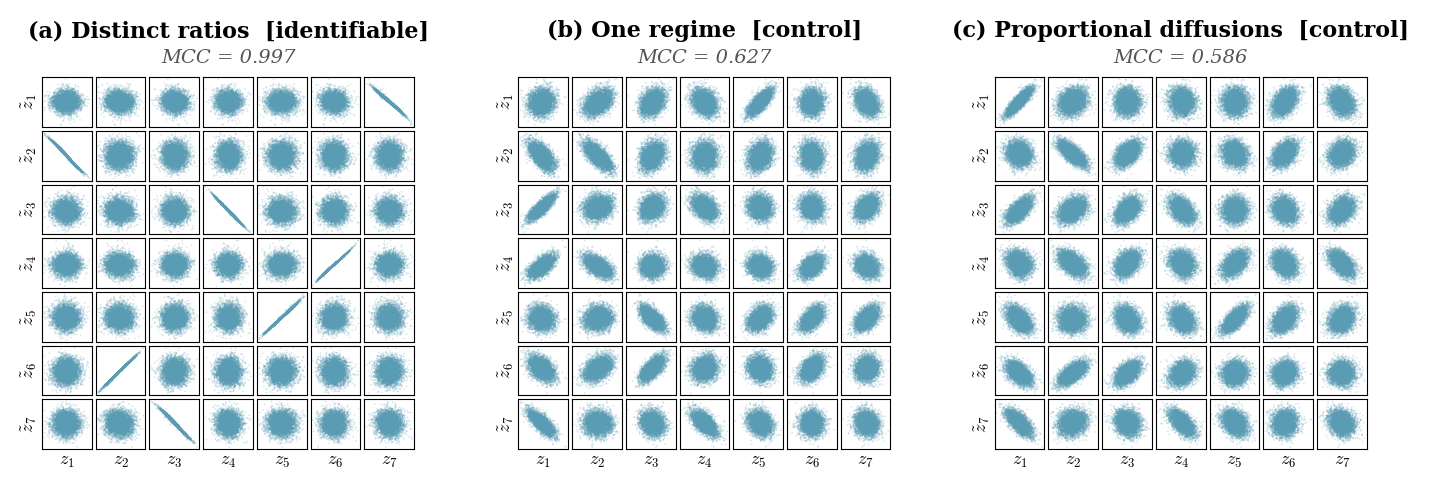}
  \caption{Latent recovery in the sparse linear $d=7$ setting (three-layer MLP, leaky-$\tanh$ mixing). Layout follows Figure~\ref{fig:sparse_scatter}.}
  \label{fig:sparse_scatter_d7}
\end{figure*}

\begin{figure*}[ht!]
  \centering
  \includegraphics[width=\textwidth]{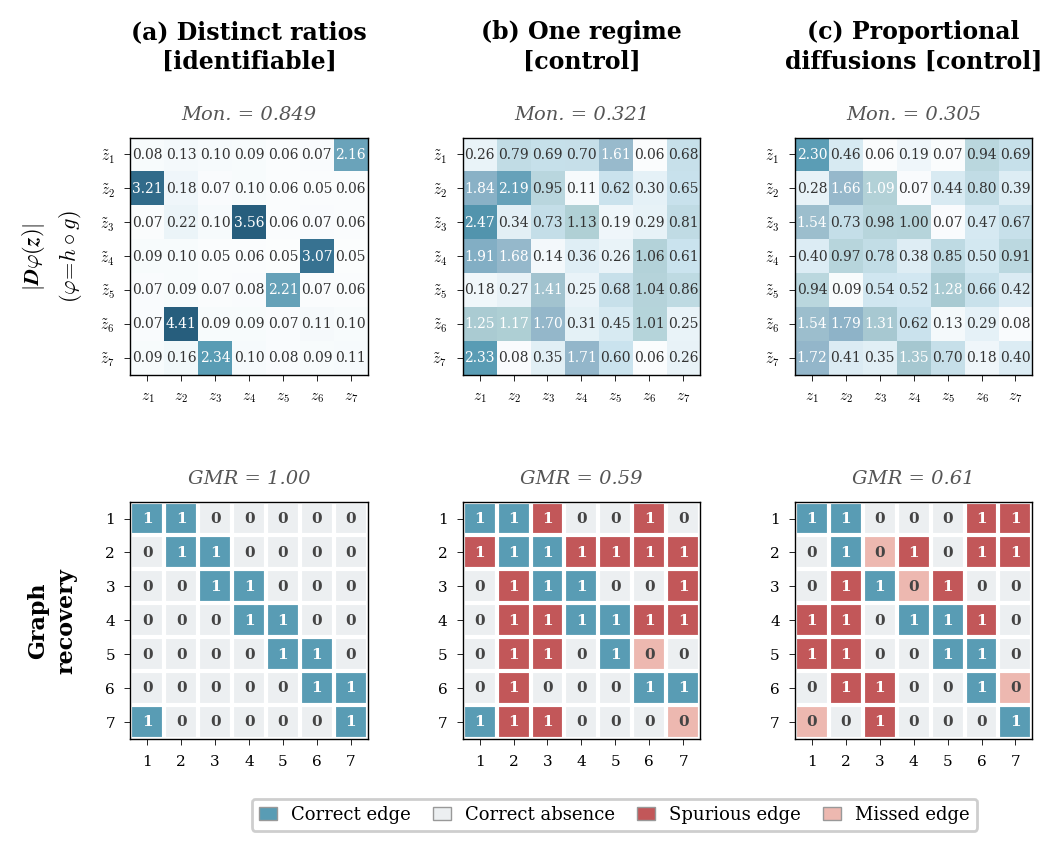}
  \caption{Structural diagnostics for the sparse linear $d=7$ setting. All three regime conditions use the same representative seed. \textbf{Top row}: mean absolute encoder Jacobian $\overline{|D\varphi|}$. \textbf{Bottom row}: learned causal graph versus the ground-truth graph.}
  \label{fig:sparse_diagnostics_d7}
\end{figure*}

\clearpage

\subsection{Nonlinear sparse drift}
\label{app:nonlinear}

The nonlinear experiments use the same directed cycle support as the sparse linear case (Section~\ref{app:sparse_linear}), but replace the linear coupling with a $\tanh$ nonlinearity. Concretely, the drift function is
\[
  f_i(z) \;=\; -\alpha \, z_i \;+\; \beta \sum_{j:\,\mathrm{adj}(i,j)=1} c_{ij}\,\tanh(z_j),
\]
where $\alpha = 1.0$ enforces mean-reversion along the diagonal and $\beta = 1.5$ controls the coupling strength. The coefficients $c_{ij}$ are drawn independently from $\mathrm{Uniform}(0.8, 1.2)$ with a fixed seed and are nonzero only where the adjacency is~$1$. For $d=5$ the nonzero coupling coefficients, rounded to two decimals, are
\[
  c_{12} = 1.18,\quad
  c_{23} = 1.15,\quad
  c_{34} = 0.88,\quad
  c_{45} = 0.92,\quad
  c_{51} = 1.04.
\]
Here $\mathrm{adj}(i,j)=1$ means that coordinate $z_j$ enters the $i$th drift component $f_i$.

For $d=7$ the nonzero off-diagonal support entries are again placed at $(i,i+1)$ for $i=1,\ldots,6$ and at $(7,1)$, and the nonzero coupling coefficients, rounded to two decimals, are
\[
  c_{12} = 1.09,\quad
  c_{23} = 0.95,\quad
  c_{34} = 1.13,\quad
  c_{45} = 1.06,\quad
  c_{56} = 1.03,\quad
  c_{67} = 0.97,\quad
  c_{71} = 1.08.
\]
The Jacobian of $f$ evaluated at the origin recovers the same sparse support pattern as in Section~\ref{app:sparse_linear}, but away from the origin the $\tanh$ saturates, making the dynamics genuinely nonlinear.

For each fixed dimension, the same drift function is shared across the leaky-$\tanh$ and LeakyReLU mixing experiments. The main text (Figures~\ref{fig:scatter} and~\ref{fig:diagnostics}) reports the $d=5$ results under leaky-$\tanh$ mixing. Here we show the corresponding results under LeakyReLU mixing for both $d=5$ and $d=7$. All three regime conditions in each figure use the same representative training seed for a fair comparison. The identifiability gap is consistent with the leaky-$\tanh$ results in the main text. Under Distinct ratios, the displayed run recovers clear coordinate alignment and a substantially more accurate thresholded graph; the controls show mixed coordinates and graph errors.

\begin{figure*}[ht!]
  \centering
  \includegraphics[width=\textwidth]{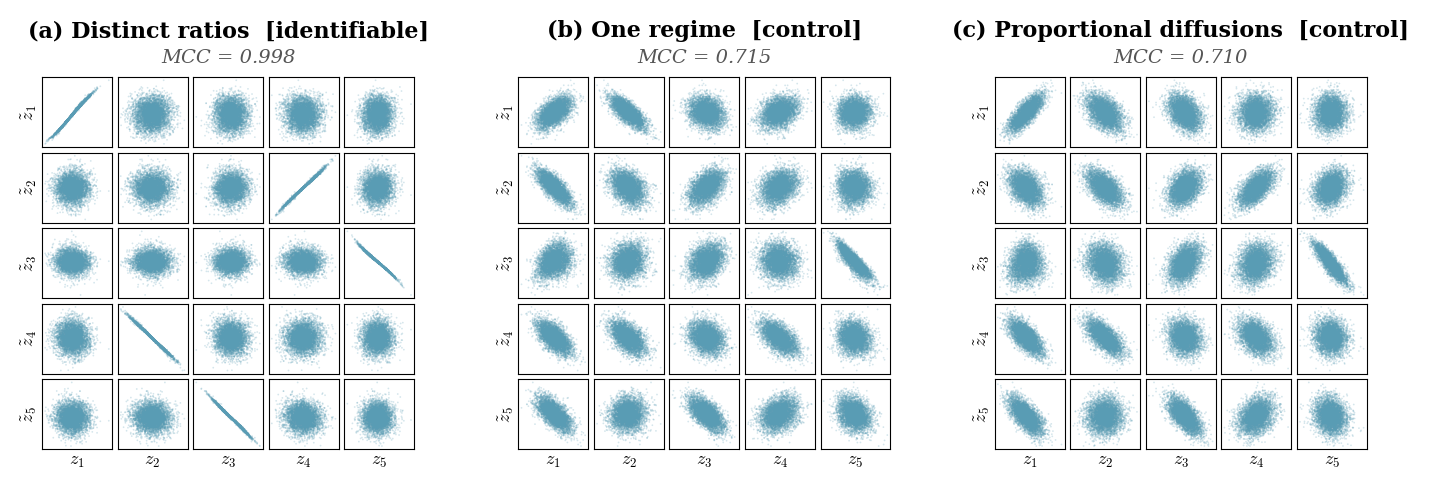}
  \caption{Latent recovery in the nonlinear $d=5$ setting (three-layer MLP, LeakyReLU mixing). Each column shows the $5 \times 5$ scatter matrix of true latent coordinates $z_j$ (horizontal) versus learned representations $\tilde{z}_i$ (vertical) for one regime condition.}
  \label{fig:nonlinear_lrelu_scatter}
\end{figure*}

\begin{figure*}[ht!]
  \centering
  \includegraphics[width=\textwidth]{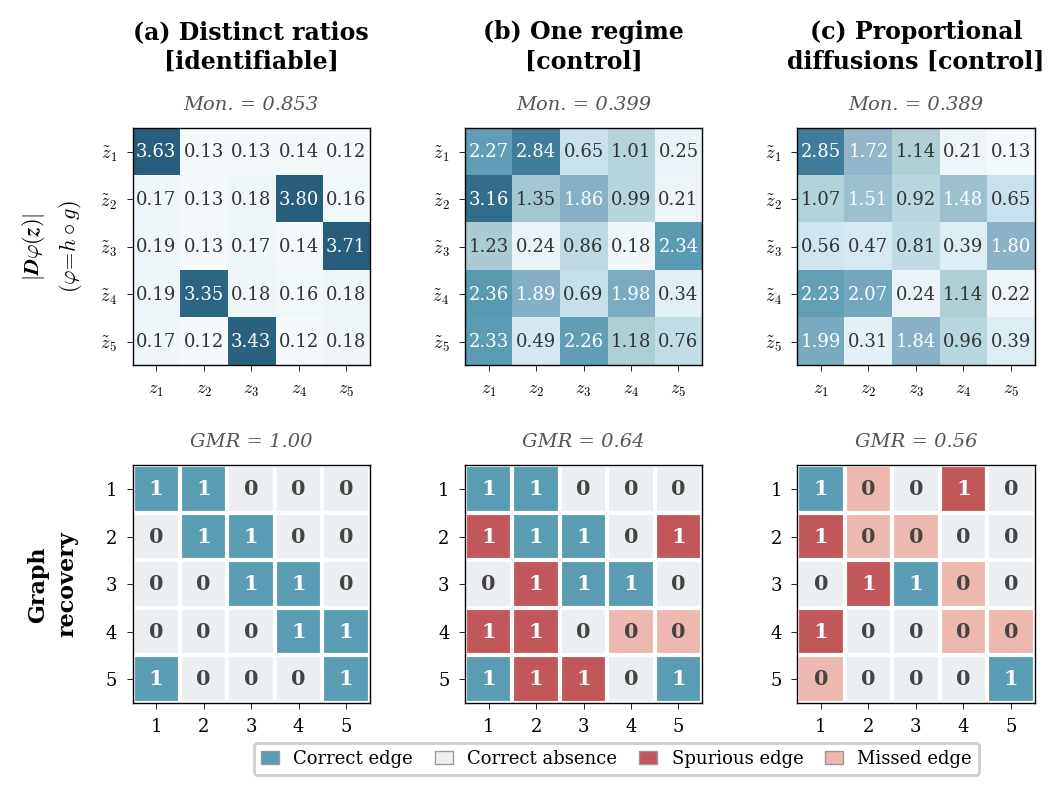}
  \caption{Structural diagnostics for the nonlinear $d=5$ setting with LeakyReLU mixing. All three regime conditions use the same representative seed. \textbf{Top row}: mean absolute encoder Jacobian $\overline{|D\varphi|}$, where $\varphi=h\circ g$. \textbf{Bottom row}: learned causal graph versus the ground-truth graph.}
  \label{fig:nonlinear_lrelu_diagnostics}
\end{figure*}

Figures~\ref{fig:nonlinear_lrelu_scatter_d7} and~\ref{fig:nonlinear_lrelu_diagnostics_d7} show the $d=7$ results under LeakyReLU mixing. The identifiability gap remains clear: Distinct ratios achieves near-perfect disentanglement and high graph match rate, whereas both controls show weaker coordinate recovery and less reliable graph recovery.

\begin{figure*}[ht!]
  \centering
  \includegraphics[width=\textwidth]{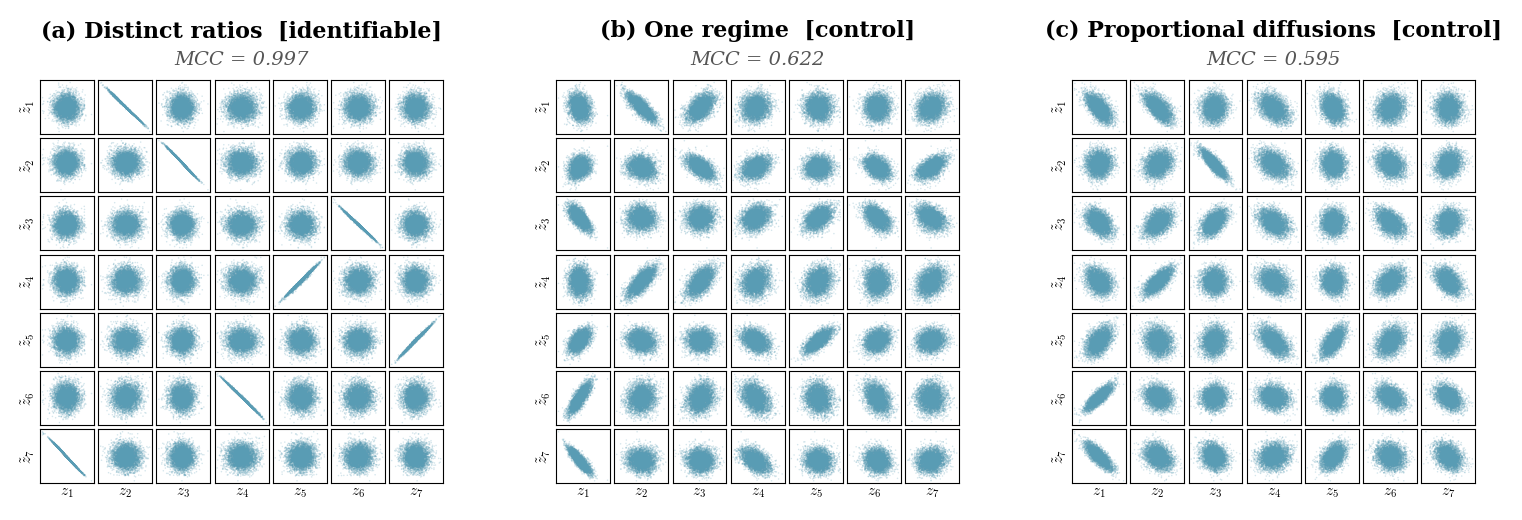}
  \caption{Latent recovery in the nonlinear $d=7$ setting (three-layer MLP, LeakyReLU mixing). Each column shows the $7 \times 7$ scatter matrix of true latent coordinates $z_j$ (horizontal) versus learned representations $\tilde{z}_i$ (vertical) for one regime condition. Under Distinct ratios, each $\tilde{z}_i$ aligns with exactly one $z_j$. Both controls fail to disentangle.}
  \label{fig:nonlinear_lrelu_scatter_d7}
\end{figure*}

\begin{figure*}[ht!]
  \centering
  \includegraphics[width=\textwidth]{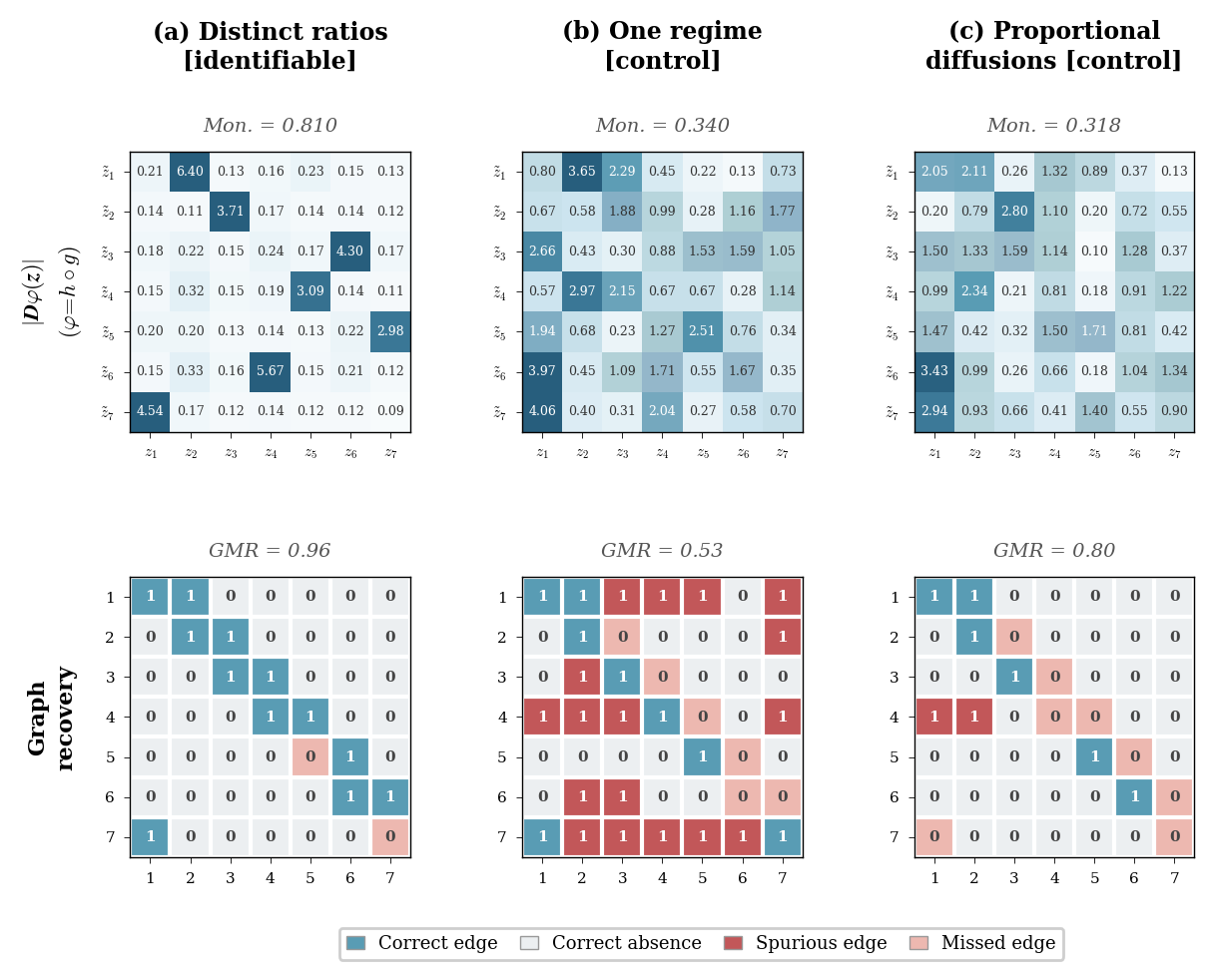}
  \caption{Structural diagnostics for the nonlinear $d=7$ setting with LeakyReLU mixing. All three regime conditions use the same representative seed. \textbf{Top row}: mean absolute encoder Jacobian $\overline{|D\varphi|}$, where $\varphi=h\circ g$. \textbf{Bottom row}: learned causal graph versus the ground-truth graph.}
  \label{fig:nonlinear_lrelu_diagnostics_d7}
\end{figure*}

\clearpage
\begingroup

\subsection{\texorpdfstring{$d=10$}{d=10} evaluation and dense nonlinear drift}
\label{app:d10_results}

We additionally evaluate $d=10$ under the same smooth three-layer
leaky-$\tanh$ MLP mixing used in the main experiments.  The evaluation covers
dense linear, sparse linear, sparse nonlinear, and fully dense nonlinear
drifts under all three regime conditions.  Table~\ref{tab:d10_results} reports
means $\pm$ s.d. over five training seeds.

For reproducibility, these runs use $\Delta t=0.001$, nominal horizon
$T=0.02$, $N_{\mathrm{traj}}=50{,}000$ per regime, a width-$512$ encoder with
four hidden layers, and $30{,}000$ Stage~1 optimization steps.  In the
Distinct-ratios condition, all families use
$\operatorname{diag}(\Sigma_2)=(2.5,0.3,2.1,0.35,1.8,0.43,1.6,0.5,1.1,0.38)$.
For dense linear, $\operatorname{diag}(\Sigma_1)=(0.5,1.0,\ldots,5.0)$;
the other three families use
$\operatorname{diag}(\Sigma_1)=(0.5,0.8,1.1,1.4,1.7,2.0,2.3,2.5,2.7,3.0)$.
The One-regime control retains only the corresponding $\Sigma_1$, and the
Proportional-diffusions control uses $\Sigma_2=2\Sigma_1$.
Sparse-linear and sparse-nonlinear Stage~2 use
$\lambda_{\mathrm{sparse}}=0.2$ and $0.8$, respectively; the dense families
omit Stage~2.

\begin{table*}[ht!]
  \caption{$d=10$ results under three-layer leaky-$\tanh$ MLP mixing.  Bold
  entries indicate the identifiable Distinct-ratios condition.  GMR is omitted
  for dense drift families, for which sparse graph recovery is not the target.}
  \label{tab:d10_results}
  \centering
  \small
  \setlength{\tabcolsep}{4pt}
  \begin{tabular}{llccc}
    \toprule
    Drift family & Regime condition & MCC $\uparrow$ & Mon. $\uparrow$ & GMR $\uparrow$ \\
    \midrule
    Dense linear & Distinct ratios & $\mathbf{0.9810 \pm 0.0067}$ & $\mathbf{0.7709 \pm 0.0086}$ & -- \\
    Dense linear & One regime & $0.5548 \pm 0.0178$ & $0.2566 \pm 0.0087$ & -- \\
    Dense linear & Proportional diffusions & $0.5701 \pm 0.0221$ & $0.2622 \pm 0.0097$ & -- \\
    \midrule
    Sparse linear & Distinct ratios & $\mathbf{0.9962 \pm 0.0016}$ & $\mathbf{0.8207 \pm 0.0153}$ & $\mathbf{0.9880 \pm 0.0147}$ \\
    Sparse linear & One regime & $0.5696 \pm 0.0263$ & $0.2483 \pm 0.0176$ & $0.5820 \pm 0.0449$ \\
    Sparse linear & Proportional diffusions & $0.5510 \pm 0.0260$ & $0.2401 \pm 0.0180$ & $0.6540 \pm 0.0206$ \\
    \midrule
    Sparse nonlinear & Distinct ratios & $\mathbf{0.9845 \pm 0.0227}$ & $\mathbf{0.8085 \pm 0.0188}$ & $\mathbf{0.8940 \pm 0.0258}$ \\
    Sparse nonlinear & One regime & $0.5966 \pm 0.0127$ & $0.2604 \pm 0.0155$ & $0.7880 \pm 0.0791$ \\
    Sparse nonlinear & Proportional diffusions & $0.5654 \pm 0.0108$ & $0.2522 \pm 0.0100$ & $0.8200 \pm 0.0126$ \\
    \midrule
    Dense nonlinear & Distinct ratios & $\mathbf{0.9947 \pm 0.0037}$ & $\mathbf{0.8113 \pm 0.0130}$ & -- \\
    Dense nonlinear & One regime & $0.5777 \pm 0.0237$ & $0.2525 \pm 0.0158$ & -- \\
    Dense nonlinear & Proportional diffusions & $0.5763 \pm 0.0067$ & $0.2534 \pm 0.0021$ & -- \\
    \bottomrule
  \end{tabular}
\end{table*}

For the dense nonlinear family, we use
\[
  f_i(z)=-1.5z_i+\sum_{j\ne i}C_{ij}\tanh(z_j),
  \qquad \lVert C\rVert_2=1,
\]
where all off-diagonal coefficients $C_{ij}$ are nonzero.  Its drift Jacobian
therefore has all $90$ directed off-diagonal couplings and all $10$ diagonal
self-effects nonzero, for $100/100$ nonzero entries.  Under Distinct ratios it
achieves MCC $0.9947\pm0.0037$ and Monomial score $0.8113\pm0.0130$, compared
with approximately $0.577/0.253$ for both controls.  Across all four drift
families, the Distinct-ratios condition retains high coordinate recovery,
whereas both controls remain substantially lower.

For the sparse nonlinear family, the $d=10$ graph contains only $20$ edges
among $100$ entries, so full-matrix GMR is influenced by correctly matched
absences.  Pooling the off-diagonal decisions across the five seeds gives F1
$0.716$ under Distinct ratios, compared with $0.200$ and $0.039$ for One regime
and Proportional diffusions, respectively.  These results are a $d=10$ stress
test rather than a claim of asymptotic scalability.

\subsection{Computational cost and scaling}
\label{app:computational_cost}

In the synthetic experiments, the observation and latent
dimensions are both $d$, so the encoder Jacobian is a $d\times d$ matrix.  Our
finite-difference implementation evaluates the encoder once at the sampled
state and once after perturbing each input coordinate, giving $d+1$ forward
evaluations.  It explicitly stores the resulting Jacobian, requiring $O(d^2)$
storage, and computes its log-determinant in $O(d^3)$ arithmetic per sampled
state, excluding the cost of the encoder forward evaluations themselves.
Sparse-nonlinear Stage~2 differentiates the drift Jacobian with higher-order
automatic differentiation and is the main measured bottleneck; sparse-linear
Stage~2 avoids this cost.

Table~\ref{tab:computational_cost} reports training-only wall time and peak GPU
memory on one NVIDIA L40S.  Entries are means over five seeds; timing s.d. is
below $2\,\mathrm{s}$ throughout, and memory is the peak allocation across
stages.

\begin{table}[ht!]
  \caption{Measured computational cost for the identifiable-condition
  configurations.}
  \label{tab:computational_cost}
  \centering
  \small
  \setlength{\tabcolsep}{5pt}
  \begin{tabular}{lrrrr}
    \toprule
    Drift family & $d$ & Stage~1 (s) & Stage~2 (s) & Peak MiB \\
    \midrule
    Sparse linear & 5  & 113.7 & 17.8  & 162 \\
    Sparse linear & 7  & 131.4 & 18.9  & 166 \\
    Sparse linear & 10 & 267.1 & 43.7  & 334 \\
    Sparse nonlinear & 5  & 122.9 & 111.6 & 187 \\
    Sparse nonlinear & 7  & 141.1 & 151.4 & 191 \\
    Sparse nonlinear & 10 & 289.5 & 350.2 & 385 \\
    \bottomrule
  \end{tabular}
\end{table}

Dense families have comparable Stage~1 cost at the same $d$ and omit Stage~2.
Because the $d=10$ configurations use wider networks and more optimization
steps, the measured increase combines dimension, model capacity, and
optimization effort rather than estimating an asymptotic scaling exponent.
Together with the $d=10$ recovery results in
Table~\ref{tab:d10_results}, these measurements establish practical feasibility
at this scale and identify nonlinear Stage~2 as the dominant measured cost.

\subsection{Irregular sampling}
\label{app:irregular_sampling}

This experiment asks whether one learned continuous-time drift,
together with the regime-specific learned diffusions, can describe transitions
at observation gaps not seen during training.  We use the dense-linear
system at $d=5$ over five seeds.  Data are
generated with exact Ornstein--Uhlenbeck transitions, while the learner retains
the local Euler Stage~1 objective and the same smooth three-layer MLP mixing as
in the main experiments.  Training uses balanced gaps
$\{0.002,0.008\}$.  At test time, $0.005$ is an unseen interpolation gap,
whereas $0.001$ and $0.010$ lie outside the training range.

We compare a gap-aware estimator, which uses the observed
elapsed time in each transition mean and covariance, with an otherwise
identical mean-gap ablation that substitutes $0.005$ for every elapsed time.
The latter is a controlled analogue of a fixed-lag discrete-time model that
ignores timestamps.  Before scoring, we freeze the encoder and drift and refit
each variant's diagonal diffusion on independent validation data.  Lower
held-out observation-space transition pseudo-NLL is better.  We report the
reduction
$\mathrm{NLL}_{\mathrm{mean\text{-}gap}}-
\mathrm{NLL}_{\mathrm{gap\text{-}aware}}$ in nats (natural-log units) per
transition, so a positive value favors the gap-aware model.  Calibration is the
mean squared normalized residual: one is ideal, a value below one means that
the predictive variance is too large, and a value above one means that it is
too small.

\begin{table}[ht!]
  \caption{Held-out transition quality at gaps not used for training.
  Pseudo-NLL reduction is mean-gap minus gap-aware, in natural-log units per
  transition; positive values favor the gap-aware model.  Calibration s.d. is
  at most $0.031$ across the five seeds.}
  \label{tab:irregular_sampling}
  \centering
  \small
  \setlength{\tabcolsep}{4pt}
  \begin{tabular}{rrrr}
    \toprule
    Test gap & Gap-aware calibration & Mean-gap calibration & Pseudo-NLL reduction \\
    \midrule
    $0.001$ & $1.058$ & $0.263$ & $2.074\pm0.028$ \\
    $0.005$ & $1.032$ & $1.022$ & $0.013\pm0.011$ \\
    $0.010$ & $1.001$ & $2.068$ & $0.969\pm0.051$ \\
    \bottomrule
  \end{tabular}
\end{table}

The gap-aware model remains close to ideal calibration at the
interpolation point and at both extrapolation points.  The mean-gap ablation
predicts too much variance at the shorter gap and too little at the longer gap.
At $0.005$, where the ablation happens to use the correct elapsed time, the two
models have nearly identical calibration and pseudo-NLL.  This matched point is
a sanity check: it shows that the gains at $0.001$ and $0.010$ come from using
the correct elapsed time rather than from a systematic difference in training
quality.  Coordinate recovery remains high for both variants (Pearson MCC
$0.9972\pm0.0003$ versus $0.9967\pm0.0003$), as expected because this experiment
targets transition calibration rather than representation accuracy.  The
result therefore illustrates the practical role of the continuous-time
parameterization: the same shared drift and regime-specific diffusions can be
evaluated across observation gaps, including gaps outside the training range.

\endgroup

\clearpage
\section{Real-world case study: Hardanger Bridge}
\label{app:hardanger}

This appendix gives the setup and diagnostics behind the Hardanger Bridge
experiment reported in Section~\ref{subsec:hardanger}.  The goal of this
experiment is not to evaluate recovery against ground-truth latent coordinates,
which are unavailable in the real bridge data.  Instead, we test whether the
TI-split two-regime construction produces a more reproducible latent
representation and drift structure than controls that remove the
turbulence-intensity (TI) contrast.  The appendix first describes the data and
training protocol, then reports cross-seed reproducibility, diffusion-fingerprint
diagnostics, and a representative learned drift structure.

\subsection{Dataset and regime construction}
\label{app:hardanger_setup}

We use operational acceleration recordings from the Hardanger Bridge, a
long-span suspension bridge monitored with wind and acceleration sensors
\cite{fenerci2020wind,fenerci2021data}.  The dataset DOI/repository record
lists a Creative Commons Attribution 4.0 International (CC BY 4.0) license
\cite{fenerci2020wind}.  The model receives only bridge-deck accelerometer
channels as observations.  Wind measurements are used only to construct
regimes, not as model inputs.

To keep the shared-drift assumption physically meaningful, we restrict the data
to a single mean-wind-speed bin, $4$--$6\,\mathrm{m/s}$.  Within this bin, each
raw recording is divided into non-overlapping $120\,\mathrm{s}$ sub-windows.
For each sub-window we compute the turbulence intensity
\[
  \mathrm{TI} = \frac{\mathrm{std}(\text{wind speed})}
                     {\mathrm{mean}(\text{wind speed})}.
\]
The two regimes are then defined by a quantile split of TI: the bottom $20\%$
of sub-windows form the low-TI regime and the top $20\%$ form the high-TI
regime; the middle $60\%$ is discarded.  This construction is the real-world
analogue of the distinct-diffusion-regime condition in the synthetic
experiments: the bridge dynamics are approximately shared within a matched wind
speed bin, while turbulence intensity changes the stochastic excitation.

Table~\ref{tab:hardanger_data_setup} summarizes the resulting dataset.  The
accelerometer signals are detrended, low-pass filtered at $2.5\,\mathrm{Hz}$,
resampled to $10\,\mathrm{Hz}$, and standardized with one scaler fitted jointly
across both regimes.  The selected $120\,\mathrm{s}$ sub-windows are then
converted into overlapping $12.8\,\mathrm{s}$ trajectory windows with
$6.4\,\mathrm{s}$ stride for training.

\begin{table}[ht!]
  \caption{Hardanger data construction for the reported $4$--$6\,\mathrm{m/s}$,
  16-channel experiment.  The same training tensors are used for the
  TI-split-regime experiment and for both controls, except that the controls
  alter or remove the regime labels as described in
  Section~\ref{app:hardanger_protocol}.}
  \label{tab:hardanger_data_setup}
  \centering
  \small
  \setlength{\tabcolsep}{6pt}
  \begin{tabular}{ll}
    \toprule
    Quantity & Value \\
    \midrule
    Mean-wind-speed bin & $4$--$6\,\mathrm{m/s}$ \\
    Accelerometer channels & $16$ bridge-deck channels \\
    Regime statistic window & $120\,\mathrm{s}$, non-overlapping \\
    Low-TI threshold & $0.0336$ \\
    High-TI threshold & $0.0794$ \\
    Selected sub-windows & $987$ low-TI and $987$ high-TI \\
    Trajectory window & $12.8\,\mathrm{s}$ ($128$ samples at $10\,\mathrm{Hz}$) \\
    Trajectory stride & $6.4\,\mathrm{s}$ \\
    Training trajectories & $16{,}779$ per regime \\
    Input dimension / latent dimension & $d_x=16$, $d_z=5$ \\
    \bottomrule
  \end{tabular}
\end{table}

\subsection{Training protocol and controls}
\label{app:hardanger_protocol}

We use the same two-stage estimator as in Section~\ref{sec:estimation}.  The
encoder is a three-layer MLP with hidden dimension $64$, and the drift model is
linear in the learned $d_z=5$ coordinates.  Both stages are trained for up to
$10{,}000$ optimization steps with early stopping patience $100$.  Stage~2 uses stride
$s=6$, sparse-Jacobian penalty $\lambda_{\mathrm{sparse}}=1.0$, and graph
threshold $\tau=0.05$.  All three regime constructions use the same list of ten
random training seeds.  The Hardanger runs use the same single-GPU compute setup
as the simulation experiments in Appendix~\ref{app:training}.

Because the Hardanger encoder maps $\R^{16}$ to $\R^5$, its
Jacobian is rectangular.  In Stage~1 we replace the square-map volume term by
the row-volume regularizer
$-\tfrac12\log\det\!\big(Dh_\theta(x)Dh_\theta(x)^\top\big)$.  This reduces to
$-\log|\det Dh_\theta(x)|$ in the square case, but for the rectangular encoder
it is a regularizer within the transition pseudo-likelihood rather than an
exact change-of-variables likelihood.

We compare three regime constructions:
\begin{itemize}[leftmargin=1.5em]
  \item \emph{TI-split regimes (ours)}: the bottom and top TI quantiles are
  used as two regimes.
  \item \emph{One regime (control)}: the two TI groups are merged into a single
  environment.  This preserves the total amount of data but removes
  regime-specific diffusion parameters.
  \item \emph{Shuffled TI labels (control)}: the two TI groups are pooled and
  randomly re-split into two equally sized groups.  This preserves the
  two-regime training format but destroys the turbulence-based contrast.
\end{itemize}

\subsection{Evaluation metrics}
\label{app:hardanger_metrics}

Because ground-truth latent states are not available, we evaluate
reproducibility across random seeds.  For each seed, we encode the same set of
$5{,}000$ evaluation points.  For every pair of seeds, we align learned
coordinates by Hungarian matching on absolute correlation and compute the mean
matched absolute correlation coefficient (MCC).  We then use the same alignment
to compare the thresholded drift-Jacobian graphs and report the graph match
rate (GMR).  With ten seeds, this gives $\binom{10}{2}=45$ seed pairs.  These
are reproducibility diagnostics rather than accuracy metrics: high values mean
that independent trainings recover compatible coordinate systems and drift
graphs.  As an auxiliary within-seed diagnostic, we also report the mean absolute
off-diagonal correlation among learned latent coordinates; lower values indicate
less redundant latent coordinates.

We also report diffusion-fingerprint diagnostics for the two-regime conditions.
For parametric fingerprints, we use the learned diagonal variances
$\tilde\sigma_{e,k}^2$ from Stage~1 and compute
$r_k^{\mathrm{par}}=\tilde\sigma_{2,k}^2/\tilde\sigma_{1,k}^2$.  For empirical
fingerprints, we compute residual covariance matrices $C_e$ after Stage~2 and
use $r_k^{\mathrm{emp}}=C_{2,kk}/C_{1,kk}$.  Since learned coordinates may
differ across runs by permutation, scaling, and a possible shift, and since the sign of the
log-ratio follows the internal regime ordering, we summarize contrast magnitude
by $|\log r_k|$ pooled over seeds and latent dimensions.

\subsection{Cross-seed consistency}
\label{app:hardanger_consistency}

Table~\ref{tab:hardanger_cross_seed_full} expands the main-text result with
descriptive standard deviations across seed pairs.  TI-split regimes yield the
highest latent reproducibility and the highest drift-graph reproducibility.  The
MCC gain is substantial relative to both controls: $+0.146$ over One regime and
$+0.205$ over Shuffled TI labels.  The same construction also gives the lowest
mean latent off-diagonal correlation, suggesting that the learned coordinates are
less redundant within each learned representation.  Because the thresholded
drift graphs are mostly diagonal with sparse off-diagonal edges, absolute GMR
values can be helped by matching many absent edges; we therefore interpret the
relative TI-split gain together with MCC and the diffusion fingerprints below.

\begin{table}[ht!]
  \caption{Cross-seed consistency for the Hardanger experiment.  Entries are
  descriptive means $\pm$ s.d.  For MCC and GMR, the s.d. is across the $45$
  pairwise seed comparisons, not a standard error and not a standard deviation
  over ten independent seeds.  For mean latent off-diagonal correlation, the
  s.d. is across the ten seeds.  GMR thresholds the mean absolute drift
  Jacobian at $\tau=0.05$.}
  \label{tab:hardanger_cross_seed_full}
  \centering
  \small
  \setlength{\tabcolsep}{4pt}
  \begin{tabular}{lccc}
    \toprule
    Regime construction & MCC $\uparrow$ & GMR $\uparrow$ & \makecell{Mean latent\\offdiag.\ corr.\ $\downarrow$} \\
    \midrule
    TI-split regimes (ours)       & $\mathbf{0.806 \pm 0.061}$ & $\mathbf{0.877 \pm 0.081}$ & $\mathbf{0.182 \pm 0.022}$ \\
    One regime (control)          & $0.660 \pm 0.051$          & $0.813 \pm 0.077$          & $0.250 \pm 0.040$ \\
    Shuffled TI labels (control)  & $0.601 \pm 0.067$          & $0.830 \pm 0.070$          & $0.239 \pm 0.052$ \\
    \bottomrule
  \end{tabular}
\end{table}

\subsection{Diffusion-fingerprint diagnostics}
\label{app:hardanger_fingerprint}

Table~\ref{tab:hardanger_fingerprint} shows that the TI-split regimes produce a
strong learned diffusion contrast, whereas Shuffled TI labels collapse the
fingerprint ratios to nearly one.  This is the intended behavior of the control:
it preserves the existence of two batches in the objective but removes the
physical TI-based regime information.  Thus the reproducibility gap in
Table~\ref{tab:hardanger_cross_seed_full} is accompanied by the expected
difference in learned diffusion fingerprints.  The large TI-split contrasts
show that the construction supplies nontrivial learned diffusion variation,
which is necessary but not sufficient for Assumption~(S4); the shuffled-label
control does not.  In particular, this pooled contrast does not establish the
pairwise distinctness of all coordinate-wise ratios.  It is therefore a
diagnostic rather than a formal verification of Assumption~(S4) in the absence of
ground-truth latent coordinates.

The One-regime control is omitted from the fingerprint table because it has only
one diffusion covariance and therefore no variance ratio.

\begin{table}[ht!]
  \caption{Diffusion-fingerprint diagnostics.  Larger contrast magnitudes are
  more favorable for the separation required by Assumption~(S4).  The parametric column uses
  $|\log(\tilde\sigma^2_{2,k}/\tilde\sigma^2_{1,k})|$, while the empirical
  column uses $|\log(C_{2,kk}/C_{1,kk})|$.  Entries are mean $\pm$ s.d. over
  all seeds and latent dimensions ($10\times5$ values); these are descriptive
  contrast magnitudes, not confidence intervals.}
  \label{tab:hardanger_fingerprint}
  \centering
  \small
  \setlength{\tabcolsep}{4pt}
  \begin{tabular}{lcc}
    \toprule
    Regime construction
    & \makecell{Parametric\\contrast $\uparrow$}
    & \makecell{Empirical\\contrast $\uparrow$} \\
    \midrule
    TI-split regimes (ours)
      & $\mathbf{0.786 \pm 0.144}$
      & $\mathbf{0.797 \pm 0.150}$ \\
    Shuffled TI labels (control)
      & $0.004 \pm 0.002$
      & $0.003 \pm 0.002$ \\
    \bottomrule
  \end{tabular}
\end{table}

\subsection{Representative learned structure}
\label{app:hardanger_representative}

Figure~\ref{fig:hardanger_drift_jacobian} shows the learned drift Jacobians for
a randomly selected seed, shared across the three regime constructions and
ordered as in the main text.
In each heatmap, row $i$ corresponds to drift component
$\tilde f_{\psi,i}$ and column $j$ corresponds to source coordinate
$\tilde z_j$; the plotted value is the evaluation mean of
$|\partial \tilde f_{\psi,i}/\partial \tilde z_j|$.  Each panel is
diagonal-dominant, but the off-diagonal structure differs across regime
constructions.  This is the qualitative pattern one would expect if the learned
coordinates approximate a modal-like coordinate system for bridge vibrations:
self-dynamics should dominate, while cross-coordinate effects should appear as a
small number of structured couplings rather than as dense or scattered
dependence.  In this representative TI-split-regime run, the largest
off-diagonal entries form a reciprocal pair between two learned coordinates: one
drift component depends on the other coordinate with magnitude about $0.48$, and
the reverse dependence has magnitude about $0.40$.  A reciprocal pair of this
kind is consistent with a learned two-dimensional dynamical subspace, as can
arise from nearby vibration modes, non-proportional damping, aeroelastic or
environmental coupling, or residual mixing in modal-like coordinates.  By
contrast, the controls under the same representative seed also contain
off-diagonal entries, but these appear as more scattered one-way couplings across
different coordinate pairs rather than as a single coherent coupled block.  The
TI-split structure is therefore more naturally aligned with this modal-like
physical interpretation, while still serving only as an illustrative structural
diagnostic that complements the aggregate reproducibility and fingerprint results
above.

\begin{figure*}[ht!]
  \centering
  \includegraphics[width=\textwidth]{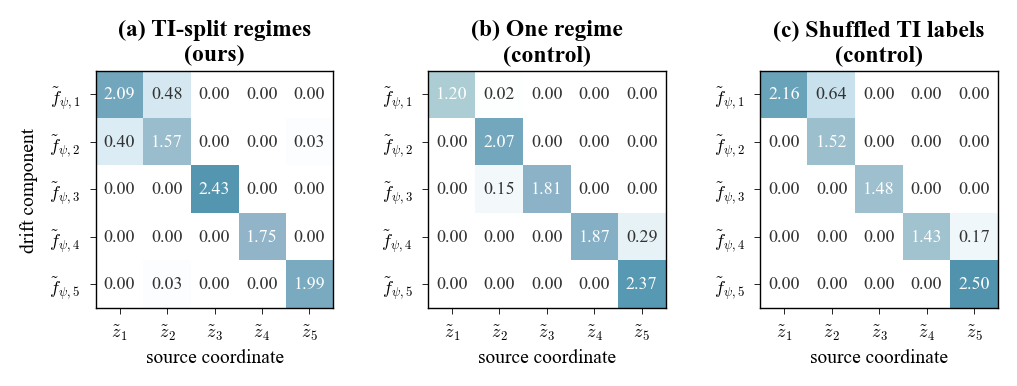}
  \caption{Drift-Jacobian diagnostics for a randomly selected seed, shared across
  the Hardanger regime constructions.  Panels follow the main-text order:
  TI-split regimes (ours), One regime (control), and Shuffled TI labels
  (control).  Rows index drift components $\tilde f_{\psi,i}$ and columns index
  source coordinates $\tilde z_j$; cell $(i,j)$ is the mean absolute derivative
  $\overline{|\partial \tilde f_{\psi,i}/\partial \tilde z_j|}$.}
  \label{fig:hardanger_drift_jacobian}
\end{figure*}

\subsection{Limitations of the real-world diagnostic}
\label{app:hardanger_limitations}

This experiment should be read as a real-world diagnostic rather than a
ground-truth recovery benchmark.  The Hardanger data do not provide true latent
coordinates or a true causal graph, so MCC and GMR measure cross-seed
reproducibility after alignment, not absolute accuracy; absolute GMR values can
also be influenced by matching many absent edges in sparse thresholded graphs.
The diffusion-fingerprint ratios support the presence of TI-dependent diffusion
variation, but they do not formally verify Assumption~(S4): the true latent
coordinates and diffusion covariances are unobserved, and the pooled contrast
does not itself test pairwise distinctness of all ratios.  Likewise, the
modal-like interpretation of the representative drift Jacobian is not a
substitute for external validation against finite-element models or modal
analysis labels.  Finally, the TI quantile split is a proxy for diffusion-regime
variation.  Its usefulness depends on maintaining enough stochastic-forcing
contrast while keeping the deterministic structural dynamics approximately
shared, exactly the tradeoff made explicit by the theory.

\clearpage
\section*{NeurIPS Paper Checklist}

\begin{enumerate}

\item {\bf Claims}
    \item[] Question: Do the main claims made in the abstract and introduction accurately reflect the paper's contributions and scope?
    \item[] Answer: \answerYes{}
    \item[] Justification: The abstract and Introduction state the main scope, assumptions, theoretical claims, estimator, and experiments. The assumptions and limitations are further qualified in Sections~\ref{sec:theory}, \ref{sec:estimation}, \ref{sec:simulation}, and \ref{sec:discussion}.

\item {\bf Limitations}
    \item[] Question: Does the paper discuss the limitations of the work performed by the authors?
    \item[] Answer: \answerYes{}
    \item[] Justification: Section~\ref{sec:discussion} discusses the main limitations of the work. Appendix~\ref{app:hardanger_limitations} separately discusses the limitations of the real-world Hardanger diagnostic.

\item {\bf Theory assumptions and proofs}
    \item[] Question: For each theoretical result, does the paper provide the full set of assumptions and a complete (and correct) proof?
    \item[] Answer: \answerYes{}
    \item[] Justification: The assumptions for the OU and nonlinear SDE results are stated in Sections~\ref{sec:theory_ou} and~\ref{sec:theory_nonlinear}. The corresponding proofs are provided in Appendix~\ref{app:deferred_proofs}, with proof sketches in the main text.

    \item {\bf Experimental result reproducibility}
    \item[] Question: Does the paper fully disclose all the information needed to reproduce the main experimental results of the paper to the extent that it affects the main claims and/or conclusions of the paper (regardless of whether the code and data are provided or not)?
    \item[] Answer: \answerYes{}
    \item[] Justification: The estimator is specified in Section~\ref{sec:estimation} and Algorithm~\ref{alg:two_stage}. Appendix~\ref{app:sim_setup} gives the synthetic data-generating process, observation model, training details, hyperparameters, metrics, and training-seed protocol, while Appendix~\ref{app:hardanger} gives the Hardanger data construction and evaluation protocol.

\item {\bf Open access to data and code}
    \item[] Question: Does the paper provide open access to the data and code, with sufficient instructions to faithfully reproduce the main experimental results, as described in supplemental material?
    \item[] Answer: \answerYes{}
    \item[] Justification: Code for reproducing the experiments is included in the supplementary material. The synthetic data are procedurally generated from Appendix~\ref{app:sim_setup}, and the real-world Hardanger recordings are available from the cited dataset DOI/repository described in Appendix~\ref{app:hardanger}.

\item {\bf Experimental setting/details}
    \item[] Question: Does the paper specify all the training and test details (e.g., data splits, hyperparameters, how they were chosen, type of optimizer) necessary to understand the results?
    \item[] Answer: \answerYes{}
    \item[] Justification: Appendix~\ref{app:sim_setup} specifies the synthetic data generation, regimes, mixing functions, optimizer, training schedule, hyperparameters, and metrics. Appendix~\ref{app:hardanger} specifies the Hardanger preprocessing, regime construction, controls, model settings, and evaluation diagnostics.

\item {\bf Experiment statistical significance}
    \item[] Question: Does the paper report error bars suitably and correctly defined or other appropriate information about the statistical significance of the experiments?
    \item[] Answer: \answerYes{}
    \item[] Justification: The synthetic tables report mean $\pm$ standard deviation over five training seeds, as stated in Section~\ref{sec:simulation} and Appendix~\ref{app:sim_setup}. The main Hardanger table reports mean cross-seed scores, and Appendix~\ref{app:hardanger} reports descriptive standard deviations across seed pairs or seeds, with the relevant variability source specified there.

\item {\bf Experiments compute resources}
    \item[] Question: For each experiment, does the paper provide sufficient information on the computer resources (type of compute workers, memory, time of execution) needed to reproduce the experiments?
    \item[] Answer: \answerYes{}
    \item[] Justification: Appendix~\ref{app:training} reports the compute setup used for the experiments: standard Slurm jobs requesting one NVIDIA L40S GPU and 32 GB memory per job, with five-seed synthetic batches typically completing in a few minutes. Appendix~\ref{app:hardanger_protocol} states that the Hardanger runs use the same single-GPU setup. The appendix also notes that the implementation can run on CPU or other CUDA GPUs, with CPU execution expected to be slower.
    
\item {\bf Code of ethics}
    \item[] Question: Does the research conducted in the paper conform, in every respect, with the NeurIPS Code of Ethics \url{https://neurips.cc/public/EthicsGuidelines}?
    \item[] Answer: \answerYes{}
    \item[] Justification: The work is theoretical and experimental, using synthetic simulations and non-human bridge sensor recordings. It does not involve human subjects or private personal data, and the Hardanger study is a diagnostic case study rather than a deployed decision system.

\item {\bf Broader impacts}
    \item[] Question: Does the paper discuss both potential positive societal impacts and negative societal impacts of the work performed?
    \item[] Answer: \answerNA{}
    \item[] Justification: This is foundational work on identifiability for latent stochastic dynamics and is not tied to a deployed decision system or a high-risk application. We do not identify a direct societal harm beyond generic concerns that may apply to causal modeling tools.
    
\item {\bf Safeguards}
    \item[] Question: Does the paper describe safeguards that have been put in place for responsible release of data or models that have a high risk for misuse (e.g., pre-trained language models, image generators, or scraped datasets)?
    \item[] Answer: \answerNA{}
    \item[] Justification: The paper does not release high-risk models, generative systems, scraped datasets, or other assets that require controlled access safeguards.

\item {\bf Licenses for existing assets}
    \item[] Question: Are the creators or original owners of assets (e.g., code, data, models), used in the paper, properly credited and are the license and terms of use explicitly mentioned and properly respected?
    \item[] Answer: \answerYes{}
    \item[] Justification: The Hardanger Bridge data source is credited through the original references in Section~\ref{subsec:hardanger} and Appendix~\ref{app:hardanger}. Appendix~\ref{app:hardanger} states that the dataset DOI/repository record lists a Creative Commons Attribution 4.0 International (CC BY 4.0) license.

\item {\bf New assets}
    \item[] Question: Are new assets introduced in the paper well documented and is the documentation provided alongside the assets?
    \item[] Answer: \answerNA{}
    \item[] Justification: The paper does not introduce a new dataset, benchmark, or pretrained model asset. The synthetic data used in the experiments are generated procedurally from the specifications in Appendix~\ref{app:sim_setup}.

\item {\bf Crowdsourcing and research with human subjects}
    \item[] Question: For crowdsourcing experiments and research with human subjects, does the paper include the full text of instructions given to participants and screenshots, if applicable, as well as details about compensation (if any)? 
    \item[] Answer: \answerNA{}
    \item[] Justification: The paper does not involve crowdsourcing or research with human subjects.

\item {\bf Institutional review board (IRB) approvals or equivalent for research with human subjects}
    \item[] Question: Does the paper describe potential risks incurred by study participants, whether such risks were disclosed to the subjects, and whether Institutional Review Board (IRB) approvals (or an equivalent approval/review based on the requirements of your country or institution) were obtained?
    \item[] Answer: \answerNA{}
    \item[] Justification: The paper does not involve crowdsourcing or research with human subjects, so IRB approval is not applicable.

\item {\bf Declaration of LLM usage}
    \item[] Question: Does the paper describe the usage of LLMs if it is an important, original, or non-standard component of the core methods in this research? Note that if the LLM is used only for writing, editing, or formatting purposes and does \emph{not} impact the core methodology, scientific rigor, or originality of the research, declaration is not required.
    \item[] Answer: \answerNA{}
    \item[] Justification: The core methodology, theory, and experiments do not use LLMs as an important, original, or non-standard component.

\end{enumerate}

\end{document}